\documentclass[hidelinks]{article}
\usepackage{iclr2023_conference,times}

\usepackage{hyperref}
\usepackage{url}

\usepackage[utf8]{inputenc} 
\usepackage[T1]{fontenc}    
\usepackage{graphicx}
\usepackage[normalem]{ulem}
\usepackage{float}
\usepackage{microtype}
\usepackage{subcaption}
\usepackage{stackengine}

\usepackage{verbatim} 

\usepackage{amsmath, amssymb, bm, bbm, mathtools}
\usepackage{booktabs} 
\usepackage{color}
\usepackage[english]{babel}
\usepackage{times}

\usepackage{authblk} 
\usepackage{xurl} 

\usepackage[header,page,toc]{appendix}
\usepackage{titletoc}

\usepackage{amsthm}
\theoremstyle{plain}
\newtheorem{theorem}{Theorem}[section]

\newtheorem{proposition}[theorem]{Proposition}

\newtheorem*{thm}{Theorem}
\newtheorem*{prop}{Proposition}

\theoremstyle{definition}
\newtheorem{definition}[theorem]{Definition}
\newtheorem{assumption}[theorem]{Assumption}
\newtheorem{fact}[theorem]{Fact}

\theoremstyle{remark}

\usepackage{enumitem}

\usepackage[globalcitecopy]{bibunits}
\usepackage{chngcntr}

\usepackage[disable,textwidth=1.3in,textsize=footnotesize]{todonotes}

\usepackage[capitalize,noabbrev]{cleveref}

\usepackage{xcolor}
\definecolor{FillColorCBBlue}{HTML}{EDF9FF}
\definecolor{FillColorCBGreen}{HTML}{EDFFFA}
\definecolor{FillColorCBOrange}{HTML}{FFF5ED}
\definecolor{FillColorCBPurple}{HTML}{FFF8FD}
\definecolor{FillColorCBYellow}{HTML}{FFFEF1}

\definecolor{CustomIvory}{HTML}{FCFCED}

\definecolor{FontColorBlue}{HTML}{0433FF}
\definecolor{FontColorGreen}{HTML}{009051}
\definecolor{FontColorOrange}{HTML}{941100}
\definecolor{FontColorPurple}{HTML}{FF2F92}
\definecolor{FontColorYellow}{HTML}{929000}

\definecolor{StrokeColorCBBlue}{HTML}{0173B2}
\definecolor{StrokeColorCBGreen}{HTML}{029E73}
\definecolor{StrokeColorCBOrange}{HTML}{D55E00}
\definecolor{StrokeColorCBPurple}{HTML}{CC78BC}
\definecolor{StrokeColorCBYellow}{HTML}{C6BD2B}

\definecolor{RevisionBorder}{HTML}{D1E7FF}
\definecolor{RevisionHighlight}{HTML}{FEECB4}
\definecolor{RevisionText}{HTML}{000000}

\newcommand{\reReviewerBase}[9]{\todo[author=\textcolor{#3}{\small \textit{Reviewer #7:} \textbf{#8}}, backgroundcolor=#1, bordercolor=#2, linecolor=#4, textcolor=#5, tickmarkheight=#6, caption={Re: #8 by \textcolor{#3}{Reviewer #7}}]{#9}}

\newcommand{\reReviewerTickedVoffset}[7]{\vspace{#7}\reReviewerBase{#1}{#2}{#3}{#2}{black}{0.1cm}{#4}{#5}{#6}\vspace{-#7}}

\newcommand{\reBlue}[3]{\reReviewerTickedVoffset{FillColorCBBlue}{StrokeColorCBBlue}{FontColorBlue}{pPYU}{#1}{#2}{#3}}
\newcommand{\reGreen}[3]{\reReviewerTickedVoffset{FillColorCBGreen}{StrokeColorCBGreen}{FontColorGreen}{WEUm}{#1}{#2}{#3}}
\newcommand{\reOrange}[3]{\reReviewerTickedVoffset{FillColorCBOrange}{StrokeColorCBOrange}{FontColorOrange}{qVYP}{#1}{#2}{#3}}
\newcommand{\rePurple}[3]{\reReviewerTickedVoffset{FillColorCBPurple}{StrokeColorCBPurple}{FontColorPurple}{N3M4}{#1}{#2}{#3}}

\newcommand{\leftSide}{\reversemarginpar\setlength{\marginparwidth}{1.3 in}}
\newcommand{\rightSide}{\normalmarginpar\setlength{\marginparwidth}{0.07 in}}  

\usepackage{multirow}

\usepackage{array}
\newcolumntype{L}[1]{>{\raggedright\let\newline\\\arraybackslash\hspace{0pt}}m{#1}}
\newcolumntype{C}[1]{>{\centering\let\newline\\\arraybackslash\hspace{0pt}}m{#1}}
\newcolumntype{R}[1]{>{\raggedleft\let\newline\\\arraybackslash\hspace{0pt}}m{#1}}

\makeatletter
\newcommand*{\indep}{
    \mathbin{
        \mathpalette{\@indep}{}
    }
}
\newcommand*{\nindep}{
    \mathbin{
        \mathpalette{\@indep}{\not}
    }
}
\newcommand*{\@indep}[2]{%
    \sbox0{$#1\perp\m@th$}
    \sbox2{$#1=$}
    \sbox4{$#1\vcenter{}$}
    \rlap{\copy0}
    \dimen@=\dimexpr\ht2-\ht4-.2pt\relax
    \kern\dimen@
    {#2}%
    \kern\dimen@
    \copy0 
}
\makeatother

\newcommand{\Ebb}{\mathbb{E}}
\newcommand{\Nbb}{\mathbb{N}}

\newcommand{\Fbf}{\mathbf{F}}

\newcommand{\Acal}{\mathcal{A}}
\newcommand{\Dcal}{\mathcal{D}}
\newcommand{\Epsilon}{\mathcal{E}}

\newcommand{\Hcal}{\mathcal{H}}

\newcommand{\Xcal}{\mathcal{X}}
\newcommand{\Ycal}{\mathcal{Y}}

\usepackage{romannum}
\usepackage{pdflscape}  
\usepackage{longtable}  
\usepackage{afterpage}  

\makeatletter
\newcommand{\settitle}{\@maketitle}
\makeatother

\newcommand{\Yori}{Y^{\text{(ori)}}}
\newcommand{\Yobs}{Y^{\text{(obs)}}}
\newcommand{\STIR}{\textcolor{RevisionText}{\Delta_{\mathrm{STIR}} \rvert_{t}^{t + 1}}}

\title{Tier Balancing: Towards Dynamic Fairness over Underlying Causal Factors}

\iclrfinalcopy 

\author[1]{\textbf{Zeyu Tang}}
\author[2]{\textbf{Yatong Chen}}
\author[2]{\textbf{Yang Liu}}
\author[1,3]{\textbf{Kun Zhang}}

\affil[1]{Department of Philosophy, Carnegie Mellon University}
\affil[2]{Computer Science and Engineering Department, University of California, Santa Cruz}
\affil[3]{Machine Learning Department, Mohamed bin Zayed University of Artificial Intelligence}
\affil[ ]{\texttt{zeyutang@cmu.edu}, \texttt{ychen592@ucsc.edu}, \texttt{yangliu@ucsc.edu}, \texttt{kunz1@cmu.edu}}

\begin{document}


\settitle

\begin{abstract}
    \todo[backgroundcolor=CustomIvory, bordercolor=RevisionBorder, linecolor=white, textcolor=black, author={\textbf{\textit{Thank All Reviewers!}}}, caption={Re: All Reviewers}]{
        We are extremely grateful to all the reviewers for the comments and the time devoted.
        In the revised manuscript, we have carefully considered and incorporated the review comments.
        We provide color-coded side notes that correspond to comments/questions by each reviewer (%
            \text{\textcolor{FontColorPurple}{Reviewer N3M4},}
            \text{\textcolor{FontColorBlue}{Reviewer pPYU},}
            \text{\textcolor{FontColorGreen}{Reviewer WEUm},}
            \text{\textcolor{FontColorOrange}{Reviewer qVYP}}%
        ).
        We summarize the list of responses on \textbf{\cpageref{list:toc_responses}} at the beginning of the Appendix for the convenience of navigation.
    }
    The pursuit of long-term fairness involves the interplay between decision-making and the underlying data generating process.
    In this paper, through causal modeling with a directed acyclic graph (DAG) on the decision-distribution interplay, we investigate the possibility of achieving long-term fairness from a dynamic perspective.
    We propose \emph{Tier Balancing}, a technically more challenging but more natural notion to achieve in the context of long-term, dynamic fairness analysis.
    Different from previous fairness notions that are defined purely on observed variables, our notion goes one step further, capturing behind-the-scenes situation changes on the unobserved latent causal factors that directly carry out the influence from the current decision to the future data distribution.
    Under the specified dynamics, we prove that in general one cannot achieve the long-term fairness goal only through one-step interventions.
    Furthermore, in the effort of approaching long-term fairness, we consider the mission of ``getting closer to'' the long-term fairness goal and present possibility and impossibility results accordingly.
\end{abstract}

\setlist[itemize]{topsep=0pt}

\begin{bibunit}
    \pagenumbering{arabic}

    \startcontents[mainpaper]
    \renewcommand\contentsname{Table of Contents: Main Paper}

    \section{Introduction}
The long-term fairness endeavor inevitably involves the interplay between decision policies and the underlying data generating process: when deriving a decision-making system, one usually makes use of data at hand; when we deploy such a system, the decision would impact how data will look in the future \citep{perdomo2020performative,liu2021induced}.
To understand why and how a data distribution responds to decision-making strategies, the investigation has to resort to causal modeling.
The pursuit of long-term fairness, in turn, should also consider the changes in the underlying causal factors.

Various fairness notions with different flavors have been proposed in the literature:
associative fairness notions that capture the correlation or dependence between variables, e.g., \textit{Demographic Parity} \citep{calders2009building}, \textit{Equalized Odds} \citep{hardt2016equality};
causal fairness notions that involve modeling causal relations between variables, e.g., \textit{Counterfactual Fairness} \citep{kusner2017counterfactual,russell2017worlds},
\textit{Path-Specific Counterfactual Fairness} \citep{chiappa2019path,wu2019pc}, \textit{Causal Multi-Level Fairness} \citep{mhasawade2021causal}.
The previously proposed fairness notions are with respect to a snapshot of the static reality, and do not have a built-in capacity to model the distribution-decision interplay in the long-term fairness pursuit.

In the effort of enforcing fairness in the dynamic setting, researchers have approached the problem from different angles:
they provide causal modeling for fairness notions \citep{creager2020causal},
analyze the delayed impact or downstream effect on utilities \citep{liu2018delayed,heidari2019long,kannan2019downstream,nilforoshan2022causal},
enforce fairness in sequential or online decision-making \citep{joseph2016fairness,liu2017calibrated,hashimoto2018fairness,heidari2018preventing,bechavod2019equal},
investigate the relation between the long-term population qualification and fair decisions \citep{zhang2020fair},
take into consideration the user behavior/action when deriving a decision policy \citep{zhang2019group,ustun2019actionable,miller2020strategic,von2022fairness},
provide fairness transferability guarantee across domains \citep{schumann2019transfer,singh2021fairness},
or derive robust fair predictors \citep{coston2019fair,rezaei2021robust}.
The proposed dynamic fairness enforcing procedures usually limit their scope of consideration to only observed variables, and the fairness audit is performed directly on the decision or statistics defined on observed data.

\textcolor{RevisionText}{
In order to have a built-in capacity to capture the influence from the current decision to future data distributions, and more importantly, to induce a fair future in the long run, in this paper, we propose \textit{Tier Balancing}, a long-term fairness notion that characterizes the interplay between decision-making and data dynamics through a detailed causal modeling with a directed acyclic graph (DAG).
For example, the latent socio-economic status (whose estimation can be the output of a FICO credit score model), although not directly measurable, plays an important role in credit applications.
We are motivated by the goal of inducing a fair future by actually balancing the inherent socio-economic status, i.e., the ``tier'', of agents from different groups.}
We summarize our contributions as follows:
\begin{itemize}
    \item We formulate \textit{Tier Balancing}, a fairness notion from the dynamic and long-term perspective that characterizes the decision-distribution interplay with a detailed causal modeling over both observed variables and latent causal factors.
    \item Under the specified data dynamics, we prove that in general, one cannot directly achieve the long-term fairness goal only through a one-step intervention, i.e., static decision-making.
    \item We consider the possibility of getting closer to the long-term fairness goal through a sequence of algorithmic interventions, and present possibility and impossibility results derived from the one-step analysis of the decision-distribution interplay.
\end{itemize}

    \section{Problem Setup}\label{sec:problem_setup}
In this section, we present the formulation of the problem of interest.
We first demonstrate in \Cref{sec:causal_DAG} a detailed causal modeling of the interplay between decision-making and data dynamics.
Then in \Cref{sec:notion_TB}, we formulate \textit{Tier Balancing}, a long-term fairness notion that captures the decision-distribution interplay with the presented causal modeling.

\subsection{Causal Modeling of Decision-Distribution Interplay on DAG}\label{sec:causal_DAG}
Let us denote the time step as $T$ with domain of value $\Nbb^+$.
At time step $T$, let us denote the protected feature as $A_T$ with domain of value $\Acal = \{0, 1\}$, additional feature(s) as $X_{T,i}$ with domain of value $\Xcal_i$, the (unmeasured) underlying causal factor $H_T$ (we call it ``tier'') with domain of value $\Hcal = (0, 1]$, the (unobserved) ground truth label $\Yori_T$ and the observed label $\Yobs_T$, with domain of value $\Ycal = \{0, 1\}$, and the decision $D_T$ with domain of value $\Dcal = \{0, 1\}$.
\Cref{fig:dynamic_model} shows the causal modeling of the interplay between decision-making and underlying data generating processes, which involves multiple dynamics (from $T = t$ to $T = t + 1$).\footnote{Due to the space limit, we provide additional discussions on decision-distribution interplay in \Cref{appendix:remark_interplay}.}
\paragraph{Underlying data dynamics (stationary components)}
Considering the fact that the underlying data dynamics are relatively stable with respect to the timescale of decision-making (e.g., the societal changes happen at a much larger time scale compared to a particular credit application decision), we assume that processes governing how ($\Yori_t, X_{t,i}$) are generated from $(H_t, A_t)$ for each individual in the population are stationary and do not change over different $T = t$.
We also assume that the underlying data generating process that governs how $H_{t + 1}$ is updated from $(H_t, \Yori_t, D_t)$ across time steps is stationary, and so are the process governing the observation of $\Yobs_{t+1}$ given $(D_t, \Yori_{t+1})$ and the process governing the update of $A_{t + 1}$ from $A_t$.

The tier $H_t$ fully captures the individual's key property that is directly relevant to the scenario of interest, and therefore is the cause of $\Yori_t$ and $X_{t,i}$'s instead of the other way around.
\rePurple{C1}{
    We model the causal process along the causal direction.
    Therefore $H_{t + 1}$ is updated from $(H_t, Y^{\mathrm{(ori)}}_t, D_t)$, instead of all variables that are dependent with $H_{t + 1}$.
    In light of your comment, we add a quick example illustrating directions in our modeling choice.}{0in}\rightSide
\textcolor{RevisionText}{
For example, the improvement in the socio-economic status can be reflected through an increase in income, while manipulating one’s income only by changing the recorded number does not affect the actual ability to repay the loan.}
The determination of causal direction aligns with causal modelings in previous literature (see, e.g., \citealt{zhang2020fair}).

\paragraph{Decision-making dynamics (non-stationary components)}
The institution (decision maker) assigns decision $D_t$ to each individual according to the observed features $(A_t, X_{t, i})$ and the outcome record $\Yobs_t$.
The interpretation of the aforementioned variables depends on the problem at hand.
For instance, in the credit application scenario where $D_t$ denotes the application decision (approval or denial), we can interpret the latent tier $H_t$ as the underlying socio-economic status of an individual.
Since the decision-making strategy can vary across different time steps, we explicitly introduce the underlying factor $\theta_t$ (e.g., a hyperparameter or an auxiliary variable) to indicate such (possible) non-stationary property of decision-making.
The causal path from $\theta_t$ to $\theta_{t + 1}$ indicates the similarity between decision-making strategies as time goes by (e.g., the continuing interest on utility), although strategies themselves are not necessarily identical across different time steps.

We interpret the variable $\Yori_t$ as ``whether or not one would repay the loan were he/she approved the credit at $T = t - 1$ (which might not be the case in reality)''.
\reBlue{C4}{The potential outcome $\Yori_t$ indicates the actual ability to repay, which is an inherent fact of the individual and is irrelevant to whether or not such fact is exhibited.}{-1in}\leftSide
The variable $\Yobs_t$ is observed only if this individual actually got approved at $T = t - 1$, i.e., $D_{t - 1} = 1$.
We distinguish between the underlying ground truth $\Yori_t$ and the observed record $\Yobs_t$ because of their different roles in the decision-distribution interplay.
On the one hand, only $\Yobs_t$ is observed and therefore accessible to the decision maker (e.g., for training and evaluation).
\reGreen{W3}{If the previous decision is ``denial'', $Y_t^{(\mathrm{obs})}$ is masked and not observed.
From the viewpoint of decision-making, such data point is not informative for training and evaluating purposes since $Y_t^{(\mathrm{obs})}$ is undefined.
This motivates us to distinguish the roles played by $Y_t^{(\mathrm{obs})}$ and $Y_t^{(\mathrm{ori})}$.}{0in}\leftSide
On the other, since the potential outcome $\Yori_t$ reflects individual's inherent characteristic, which is not directly relevant to whether it is observable, $\Yori_t$ is utilized when the underlying data generating process specifies the update from $H_t$ to $H_{t + 1}$.

\begin{figure}[t]
    \centering
    \captionsetup{format=hang}
    \includegraphics[width=0.8\textwidth]{
        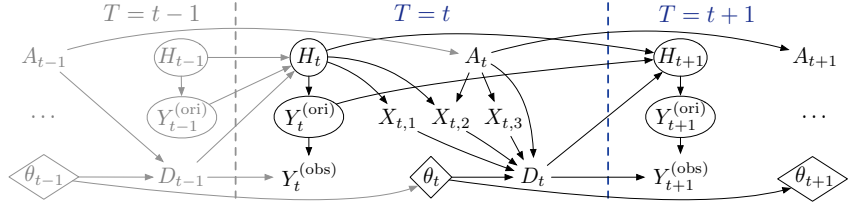}
    \caption{
        The causal modeling of the decision-distribution interplay.
        The circle (diamond) indicates that the corresponding variable (underlying factor) is unobservable.}
    \label{fig:dynamic_model}
\end{figure}

\subsection{The Notion of Tier Balancing}\label{sec:notion_TB}
\begin{definition}[\textbf{$K$-Step Tier Balancing}]\label{definition:TB}
    Under the specified dynamics, starting from any time step $T$ and a given $K \geq 0$, let us denote a sequence of $K$ decision-making strategies as $D_{T : T + K - 1}:=\{D_T, ...,D_{T + K - 1}\}$ (an empty set if $K = 0$), and the latest hidden tier after $K$-step decision-making as $H_{T + K}$.
    We say $D_{T:T+K-1}$ satisfies $K$\textit{-Step Tier Balancing}, if at $T+K$ the following condition holds true (where ``$\indep$'' denotes statistical independence):
    \vspace{-0.5em}
    \begin{equation}\label{equ:TB}
        H_{T+K} \indep A_{T+K}
        \text{, where } H_{T+K} \text{ is updated from } (H_T, \Yori_{T : T + K - 1}, D_{T : T + K - 1}).
    \end{equation}
    \vspace{-2em}
\end{definition}

\textcolor{RevisionText}{Equation \ref{equ:TB} captures the statistical consequence in the future (in the form of an associative relationship) induced by the interplay between the underlying data dynamics and decision-making policies along the way.
The causal modeling is essential in capturing our long-term fairness goal, since the attainment of \textit{Tier Balancing} is an induced outcome of a sequence of $K$ decision-making strategies $D_{T : T + K - 1}$ (which are indispensable although the fairness notion itself is not explicitly defined on decisions).}\rePurple{C6}{In light of your comment, we clarify the fairness goal immediately after introducing the fairness notion, focusing on the future statistical consequence ($H_{T + K}$) of the decision-distribution interplay along the way (which involves a sequence of decision-making $D_{T : T + K - 1}$).}{0in}\rightSide

Our \textit{Tier Balancing} notion of algorithmic fairness is distinguished from previously proposed fairness notions in several important ways.\footnote{Due to the space limit, we provide detailed discussions on related works in \Cref{appendix:discussion_related_works}.}
To begin with, \rePurple{C7}{We discuss important ways our \textit{Tier Balancing} notion differs from previous fairness notions, and provide additional detailed discussions in \Cref{appendix:discussion_related_works}.}{0in}\leftSide
\textit{Tier Balancing} has a built-in dynamic flavor, whose definition involves variables that span across multiple time steps.
Therefore the audit of \textit{Tier Balancing} inevitably requires long-term and dynamic analysis, which is very different from previously proposed (both associative and causal) fairness notions defined with respect to a static snapshot of reality (e.g., \citealp{calders2009building,hardt2016equality,kusner2017counterfactual,chiappa2019path}).

Besides, considering the fact that the decision-distribution interplay often involves situation changes in the hidden causal factors, \textit{Tier Balancing} extends the scope of fairness consideration beyond only observed variables to hidden causal factors, which makes our notion a technically more challenging but more natural long-term fairness goal to achieve.
The endeavor to explore the possibility of defining fairness in terms of latent causal factors is not an unrealistic fantasy.
Recent advances in causal discovery literature have established identifiability results (under certain assumptions) on causal structures among latent variables \citep{xie2020generalized,adams2021identification,kivva2021learning,xie2022identification}, which provide not only a theoretical justification, but also an indication of the potential, for our effort in exploring long-term fairness endeavor through modeling latent causal factors.

Furthermore, although \textit{Tier Balancing} is not directly defined in terms of the decisions themselves, \textit{Tier Balancing} is characterized with a detailed causal modeling that involves both decision-making and data dynamics.
The explicit causal modeling of the decision-distribution interplay offers both challenges and opportunities for more principled fairness inquiries in the long-term, dynamic context \citep{hu2018short,liu2018delayed,mouzannar2019fair,heidari2019long,zhang2020fair}.

In Definition \ref{definition:TB}, we specify the step $K$ at which \textit{Tier Balancing} is evaluated.
If $K = 0$, i.e., $H_T \indep A_T$ happens to be attained initially (although in general it may not be the case), \textit{Tier Balancing} is attained at the beginning.\footnote{
In \Cref{appendix:when_initially_attained}, we analyze the scenario in which \textit{Tier Balancing} is initially satisfied.}
When $H_T \nindep A_T$ and $K \geq 1$, $K$\textit{-Step Tier Balancing} is achieved with respect to the underlying causal factor $H_{T + K}$.

    \section{Characterizing Tier Balancing}\label{sec:TB}
In \Cref{sec:problem_setup}, we propose a detailed causal modeling of the interplay between decision-making and underlying data generating processes, based on which we formulate a novel long-term fairness notion \textit{Tier Balancing}.
Our model is applicable to a wide range of resource allocation scenarios, e.g., hiring practice \citep{hu2018short,kannan2019downstream}, credit application \citep{liu2018delayed}, predictive policing \citep{ensign2018runaway,elzayn2019fair}.

For the clarity of discussion, in this section we consider a running example of credit application where agents in a fixed population repeatedly apply for credit.
We first demonstrate how one can apply the proposed causal modeling in the credit application scenario in \Cref{sec:model_quantitative}.
Then in \Cref{sec:one_step_analysis_framework}, we characterize the \textit{Single-step Tier Imbalance Reduction (STIR)} term for the purpose of conducting one-step analysis on the \textit{Tier Balancing} notion of long-term fairness.

\subsection{Modeling Detail of Decision-Distribution Interplay}\label{sec:model_quantitative}
As shown in Figure \ref{fig:dynamic_model}, the unmeasured latent causal factor $H_t$ (the hidden socio-economic status) is the actual root cause of the ground truth label $\Yori_t$ as well as the (possibly) observed label $\Yobs_t$ (the repayment record).
For any given tier $H_t = h_t$, let us assume that the unobserved ground truth $\Yori_t$ is sampled from a Bernoulli distribution with $h_t$ as the success probability, and that the observed repayment record $\Yobs_t$ depends on both the ground truth label $\Yori_t$ and the previously received decision $D_{t - 1}$:
\rePurple{C4}{There is an arrow between $Y_t^{(\mathrm{ori})}$ and $Y_t^{(\mathrm{obs})}$ since $Y_t^{(\mathrm{obs})}$ is the observed copy of $Y_t^{(\mathrm{ori})}$ masked by $D_{t -1}$.
The mathematical expression is summarized in Equation \ref{equ:Y_definitions}.}{0in}\leftSide
\begin{equation}\label{equ:Y_definitions}
    \Yobs_t \begin{cases}
        = \Yori_t & \text{if } D_{t - 1} = 1, \\
        \textit{is undefined} & \text{if } D_{t - 1} = 0,
    \end{cases}
    \text{ where }
    \Yori_t \sim \mathrm{Bernoulli}(h_t).
\end{equation}
From Equation \ref{equ:Y_definitions} we can see that $\Yobs_t$ is a masked copy of $\Yori_t$ (masked by $D_{t - 1}$), and we have the following proposition capturing the property of the marginal distribution of $\Yobs_t$:
\begin{proposition}\label{proposition:Y_obs_bernoulli}
    At time step $T = t$, for any $H_t = h_t \in (0, 1]$, under the specified dynamics, among the population where ground truth is actually observable, i.e., $\Yobs_t$ is not undefined, we have:
    \begin{equation*}
        \Yobs_t \sim \mathrm{Bernoulli}(h_t).
    \end{equation*}
\end{proposition}
Proposition \ref{proposition:Y_obs_bernoulli} captures the fact that among the population where repayment record $\Yobs_t$ is observed, the marginal distributions of $\Yobs_t$ and $\Yori_t$ are actually identical.
This property indicates that although one does not have access to the unobserved tier, i.e., the socio-economic status $H$, one can still use the observed $\Yobs$ as a bridge to infer its behavior.

\begin{fact}\label{fact:var_form}
    Let $\Acal$ be the domain of value for the protected feature $A_t$, and $\Epsilon$ be the domain of value for all other exogenous noise terms $E_t$.
    For each time step $T = t$, we can represent $D_t$, $\Yori_t$, and $H_t$ via functions $g^D_t$, $g^{\Yori}_t$, and $f_t$ respectively, where $g^D_t: \Acal \times \Epsilon \rightarrow \{0, 1\}$, $g^{\Yori}_t: \Acal \times \Epsilon \rightarrow \{0, 1\}$, and $f_t: \Acal \times \Epsilon \rightarrow (0, 1]$, i.e., $D_t = g^D_t(A_t, E_t)$, $\Yori_t = g^{\Yori}_t(A_t, E_t)$, and $H_t = f_t(A_t, E_t)$.
\end{fact}
Notice that Fact \ref{fact:var_form} represents variables $D_t$, $\Yori_t$, and $H_t$ with functions of \emph{all} root causes (including the protected feature $A_t$ and exogenous noise terms $E_t$) in the system without explicitly specifying the respective functional forms, which may depend on further assumptions on the joint distribution and the time step $T = t$.\footnote{We implicitly adopt the assumption that the protected feature itself is not caused by other variables.}
Fact \ref{fact:var_form} is a direct result of representing causal relations with a functional causal model (FCM) \citep{spirtes1993causation,pearl2009causality}, and is for the purpose of notational convenience in later analysis.\footnote{In \Cref{appendix:remark_FCM}, we discuss the role of exogenous terms $E_t$ in Fact \ref{fact:var_form}.}

\begin{assumption}[\textbf{Multiplicative update of underlying tier}]\label{assumption:Hupdate}
    Let $\alpha_D, \alpha_Y \in [0, \frac{1}{2})$ be the parameters that capture the influences from current decision $D_t$ and ground truth $\Yori_t$ to next step:
    \begin{equation}\label{equ:Hupdate}
        H_{t + 1} = \min \big\{ 1, H_t \cdot \big[ 1 + \alpha_D (2 D_t - 1) + \alpha_Y (2 \Yori_t - 1) \big] \big\}.
    \end{equation}
    \vspace{-2em}
\end{assumption}

Assumption \ref{assumption:Hupdate} states that $H_{t + 1}$ treats $H_t$ as a baseline, with increase or decrease in a multiplicative form based on agent's received decision and repayment information.
We are inspired by the evolution theory where multiplicative updates have been a common modeling choice to capture updates in relevant statistics  \citep{friedman2016evolutionary,dawkins2017selfish}.
The explicit dependency on the update parameters, $\alpha_D$ and $\alpha_Y$, related to $D_t$ and $\Yori_t $ respectively, characterizes the two important aspects of our model: the update in individuals' tier potentially depends on the received decision $D_t$, as well as the ground truth $\Yori_t $ (even if unobserved).\footnote{
In Assumption \ref{assumption:Hupdate} we explicitly specify that we are considering $\Yori_t$ instead of $\Yobs_t$.
At $T = t$, every individual has a binary ground truth $\Yori_t$.
However, it might not be the case that everyone has an observable $\Yobs_t$ (since $\Yobs_t$ is undefined for an individual if its $D_{t - 1} = 0$).}
The condition $\alpha_D, \alpha_Y \in [0, \frac{1}{2})$ makes sure that $H_{t + 1} > 0$, and the $\min\{\cdot, 1\}$ operation makes sure that $H_{t+1}$ is upper-capped by $1$.

\textcolor{RevisionText}{In practical scenarios where agents repetitively apply for resource (e.g., in our running example of credit application) at each time step $T = t$, with the entire group remains unchanged, we have:}
\begin{assumption}\label{assumption:samegroup}
    The protected feature at time step $T = t + 1$ is an identical copy of that at $T = t$:
    \begin{equation}
        \forall ~ T = t: ~ A_{t + 1} = A_t.
    \end{equation}
\end{assumption}

\subsection{One-step Analysis Towards Tier Balancing}\label{sec:one_step_analysis_framework}
In Definition \ref{definition:TB} we established the long-term fairness goal.
Considering the dynamic property of this fairness notion, apart from defining ``what exactly is fair in the long run'', in order to bridge the cognitive gap we also need to clarify the meaning of ``getting closer to the long-term fairness goal'', i.e., achieving the long-term fairness goal through a sequence of algorithmic interventions.
In this section, we present the one-step theoretical analysis framework and characterize the \textit{Single-step Tier Imbalance Reduction (STIR)} term $\STIR$ for the purpose of investigating \textit{Tier Balancing}.

\subsubsection{Single-step Tier Imbalance Reduction (STIR)}\label{sec:introduce_STIR}
Recall that the long-term fairness objective is the independence between the protected feature $A_T$ and the hidden tier $H_T = f_T(A_T, E_T)$ (at a time step $T$).
Equivalently, we would like $H_T$ to \textit{not} be a function of $A_T$.
We can view $f_T(0, E_T)$ and $f_T(1, E_T)$ as two dependent random variables, and quantify the amount of  ``getting closer to the long-term fairness goal'' by comparing the absolute difference between $f_T(0, E_T)$ and $f_T(1, E_T)$ before (when $T = t$) and after (when $T = t + 1$) one-step update, and see if the gap decreases.
Since the individual-level exogenous noise term $E_T$ is the input, this comparison of absolute differences is on the individual level.
Therefore in order to quantitatively characterize the overall amount of ``getting closer to the long-term fairness goal'', we need to take into account different possible combinations of decision $D_t$ and outcome $\Yori_t$ (when $T = t$) for each individual, and aggregate the individual-level comparisons over the population.

Given combinations of $D_t$ and $\Yori_t$, when $E_t = \epsilon$, we denote the conditional joint probability density of $\big( f_t(0, E_t), f_t(1, E_t) \big)$ as \textcolor{RevisionText}{$q_t\big(f_t(0, \epsilon), f_t(1, \epsilon) \mid d, d', y, y' \big)
\coloneqq q_t\big(f_t(0, \epsilon), f_t(1, \epsilon)
\mid g_t^D(0, \epsilon) = d,
g_t^D(1, \epsilon) = d',
g_t^{\Yori}(0, \epsilon) = y,
g_t^{\Yori}(1, \epsilon) = y' \big)$,} and calculate the \textit{Single-step Tier Imbalance Reduction (STIR)} term from $T = t$ to $T = t + 1$, denoted by
\reBlue{C6}{Following your suggestion, we make the dependence between the STIR term and the time step (from $T = t$ to $T = t + 1$) explicit throughout the manuscript.}{0in}\leftSide
$\STIR$, as following:\footnote{The details of the derivation can be found in \Cref{appendix:derive_STIR}.}
\begin{equation}\label{equ:STIR}
    \begin{split}
        & \textcolor{RevisionText}{\STIR
        \coloneqq
        \Ebb \left[~\lvert
            f_{t + 1}(0, E_{t + 1}) - f_{t + 1}(1, E_{t + 1})
        \rvert~\right]
        -
        \Ebb \left[~\lvert
            f_{t}(0, E_{t}) - f_{t}(1, E_{t})
        \rvert~\right]} \\
        & = \sum_{d, d', y, y' \in \{0, 1\}}
        P_t (d, d', y, y')
        \cdot
        \int_{\epsilon \in \Epsilon}
        \int_{\xi \in \Epsilon} q_t \big(f_t(0, \epsilon), f_t(1, \epsilon) \mid d, d', y, y' \big) \\
        & ~~~~~~
        \cdot \big( \lvert \varphi_{t + 1}(\xi) \rvert - \lvert \varphi_t(\epsilon) \rvert \big)
        \cdot
        \textcolor{RevisionText}{\mathbbm{1}\{ \varphi_{t + 1}(\xi) = G_t(f_t, g_t^D, g_t^{\Yori}; d, d', y, y', \epsilon, \alpha_D, \alpha_Y) \}}
        d \xi d \epsilon,
    \end{split}
\end{equation}
where $\varphi_t(\epsilon) \coloneqq f_t(0, \epsilon) - f_t(1, \epsilon)$, \textcolor{RevisionText}{$G_t$ is a function whose value \textit{only} relies on the information available at time step $T = t$}, and $P_t (d, d', y, y')$ is the joint distribution of $(d, d', y, y')$:
\begin{equation*}
    \begin{split}
        P_t(d, d', y', y') \coloneqq P_t \big(g_t^D(0, E_t) = d, g_t^D(1, E_t) =
        d', g_t^{\Yori}(0, E_t) = y, g_t^{\Yori}(1, E_t) = y' \big).
    \end{split}
\end{equation*}
We can then characterize ``getting closer to long-term fairness goal'' via the inequality
$\STIR < 0$.

\subsubsection{Simplification Assumptions}
From Equation \ref{equ:STIR} we can see that the calculation of $\STIR$ requires knowledge about the gap comparison for each individual $\lvert \varphi_{t + 1}(e_{t + 1}) \rvert - \lvert \varphi_t(e_t)\rvert$, the conditional joint density before one-step dynamics $q_t\big(f_t(0, \textcolor{RevisionText}{\epsilon}), f_t(1, \textcolor{RevisionText}{\epsilon}) \mid d, d', y, y' \big)$, and the joint probability for different combinations of decision and ground truth before one-step dynamics $P_t (d, d', y, y')$.
\reBlue{C8}{Our results \Cref{theorem:one_step_attainability,theorem:perfect_one_step,theorem:CF_one_step} are closely related to the behavior of $\STIR$, whose quantitative analysis involves three different components. Therefore, it is essential to investigate each one of these three components separately.}{0in}\rightSide
In order to quantitatively analyze the property of $\STIR$, it is essential that we have access to all three aforementioned quantities.

To begin with, we need to know the instantiations of $\varphi_{t + 1}(e_{t + 1})$ given $\varphi_t(e_t)$ under the specified dynamics.
Luckily, as we have illustrated in Table \ref{table:delta_next_Dmajor}, Table \ref{table:delta_next_Ymajor}, and Table \ref{table:delta_next_nomajor} of Appendix, under the specified dynamics we can list all possible cases of the term $\varphi_{t + 1}(e_{t + 1})$ given $\varphi_t(e_t)$.

Besides, we need additional knowledge on the conditional joint density $q_t\big(f_t(0, \textcolor{RevisionText}{\epsilon}), f_t(1, \textcolor{RevisionText}{\epsilon}) \mid d, d', y, y' \big)$.
For the purpose of better elaboration, we present two assumptions on the behavior of this conditional joint density -- a qualitative assumption and a quantitative assumption:
\begin{assumption}[\textbf{Qualitative assumption}]\label{assumption:biased_weight_qualitative}
    For any time step $T = t$ and any exogenous term $E_t = \epsilon \in \Epsilon$, let us denote $y = g_t^{\Yori}(0, \epsilon)$ and $y' = g_t^{\Yori}(1, \epsilon)$.
    The following inequalities hold:
    \begin{equation}\label{equ:example_favoring}
        \begin{split}
            &P_t \big( f_t(0, \epsilon) > f_t(1, \epsilon) \mid y > y' \big)
            > P_t \big( f_t(0, \epsilon) < f_t(1, \epsilon) \mid y > y' \big), \\
            &P_t \big( f_t(0, \epsilon) < f_t(1, \epsilon) \mid y < y' \big)
            > P_t \big( f_t(0, \epsilon) > f_t(1, \epsilon) \mid y < y' \big).
        \end{split}
    \end{equation}
\end{assumption}
\begin{assumption}[\textbf{Quantitative assumption}]\label{assumption:biased_weight_quantitative}
    On top of Assumption \ref{assumption:biased_weight_qualitative}, let us further assume that the conditional joint density $q_t \big( f_t(0, \textcolor{RevisionText}{\epsilon}), f_t(1, \textcolor{RevisionText}{\epsilon}) \mid d, d', y, y' \big)$ satisfies the following condition:
    \vspace{-0.3em}
    \begin{equation}
        \begin{split}
            &q_t \big( f_t(0, \textcolor{RevisionText}{\epsilon}), f_t(1, \textcolor{RevisionText}{\epsilon}) \mid d, d', y, y' \big)
            =\begin{cases}
                \gamma_{dd'yy'}^{\text{(up)}} & \text{if } f_t(0, \epsilon) \leq f_t(1, \epsilon) \\[5pt]
                \gamma_{dd'yy'}^{\text{(low)}} & \text{if } f_t(0, \epsilon) > f_t(1, \epsilon)
            \end{cases},
        \end{split}
    \end{equation}
    where
    $\gamma_{dd'yy'}^{\text{(low)}} + \gamma_{dd'yy'}^{\text{(up)}} = 2$,
    $\gamma_{dd'yy'}^{\text{(low)}} < \gamma_{dd'yy'}^{\text{(up)}} \text{ when } y < y'$,
    and $\gamma_{dd'yy'}^{\text{(low)}} > \gamma_{dd'yy'}^{\text{(up)}}
            \text{ when } y > y'$.
\end{assumption}

Assumption \ref{assumption:biased_weight_qualitative} is rather mild, stating that for any given exogenous noise term (of an individual) $E_t = e_t$, whenever the ground truth $\Yori_t$ favors certain demographic group,
it is more likely that the underlying tier also favors the same group.
Assumption \ref{assumption:biased_weight_quantitative} is just a special case of Assumption \ref{assumption:biased_weight_qualitative}, with quantitative characteristics built-in for technical purposes.\footnote{
    In \Cref{appendix:illustration_assumption}, we present illustrative figures to demonstrate the connection between qualitative and quantitative assumptions.
}

Lastly, we need to know the (behavior of) joint probability density $P_t(d, d', y, y')$ for all combinations of $(d, d', y, y')$.
In fact, as we shall see in \Cref{sec:one_step_analysis}, when taking into consideration certain characteristics of the predictor, the joint probability $P_t(d, d', y, y')$ would follow some patterns that can simplify the analysis.
    \section{One-step Analysis Towards the Long-term Fairness Goal}
\label{sec:one_step_analysis}
In this section, we consider the possibility of attaining the long-term fairness goal with (a sequence of) one-step interventions.
We first present a negative result that in general one cannot hope to achieve the long-term fairness goal through a single one-step intervention, i.e., static decision-making.
In light of this result, we further investigate the possibility of getting closer to the long-term fairness goal with a sequence of one-step interventions, and present possibility and impossibility results accordingly.

\subsection{The General Impossibility of Achieving 1-Step Tier Balancing}\label{sec:one_step_not_enough}
In the following theorem, we prove that it is in general impossible to achieve $H_{t + 1} \indep A_{t + 1}$ when initially $H_t \nindep A_t$, i.e., in general we cannot achieve \textit{Tier Balancing} with a single one-step intervention.
\begin{theorem}\label{theorem:one_step_attainability}
    Let us consider the general situation where both $D_t$ and $\Yori_t$ are dependent with $A_t$, i.e., $D_t \nindep A_t, \Yori_t \nindep A_t$.
    Then under Fact \ref{fact:var_form}, Assumption \ref{assumption:Hupdate}, and Assumption \ref{assumption:samegroup}, as well as the specified dynamics, when $H_t \nindep A_t$, only if at least one of the following conditions holds true for all $e_t \in \Epsilon$ can we possibly attain $H_{t + 1} \indep A_{t + 1}$:
    \vspace{-0.5em}
    \setlist{nolistsep}
    \begin{enumerate}[label=(\arabic*), noitemsep, topsep=0pt]
        \item The ratio $\frac{f_t(0, e_t)}{f_t(1, e_t)}$ has a specific domain of value $\frac{f_t(0, e_t)}{f_t(1, e_t)} = \frac{1 \pm \alpha_D \pm \alpha_Y}{1 \pm \alpha_D \pm \alpha_Y}$; 
        \item Positive (negative) label is exclusive to the advantaged (disadvantaged) group, and everyone receives a positive decision (if $\alpha_D > \alpha_Y$);
        \item Negative (positive) labels only appear in the advantaged (disadvantaged) group, and everyone receives a positive decision (if $\alpha_D > \alpha_Y$);
        \item Everyone has a positive label, but the positive decision is exclusive to the advantaged group (if $\alpha_D < \alpha_Y$);
        \item Everyone has a positive label, but the positive decision is exclusive to the disadvantaged group (if $\alpha_D < \alpha_Y$).
    \end{enumerate}
    \vspace{-0.5em}
\end{theorem}

Theorem \ref{theorem:one_step_attainability} lists all possible ways to directly achieve the long-term fairness goal (under the specified condition).
As we can see that all of these conditions are rather restrictive: Condition (1) imposes strong conditions on the functional form of $f(\cdot)$.
In particular, when $f_t(0, e_t),f_t(1, e_t)$ are both continuous random variables with non-zero density everywhere on the support $(0,1]$, the ratio is still a continuous random variable (because the density is simply an integral over positive multiplications).
We can see that the event specified in Condition (1) is a zero-measure one.
Conditions (2-5) all require trivial decision-making policies.
Therefore, in general one cannot directly achieve the long-term fairness goal.
We need to consider the possibility of approaching the goal step-by-step.

\subsection{Possibility of Getting Closer to Tier Balancing via One-step Interventions}\label{sec:one_step_get_closer}
In \Cref{sec:one_step_not_enough}, we have seen that under the specified dynamics, single one-step intervention is in general not enough in order to achieve long-term fairness.
In this section, we investigate if certain strategy can get closer to the long-term fairness goal through a sequence of algorithmic interventions.
If we follow the same principle to derive the decision-making policy from the data at each time step, one-step analysis suffices for the purpose of studying the interplay between decision and distribution.

\subsubsection{One-step Analysis on Perfect Predictor}\label{sec:one_step_analysis_perfect}
In this section, we consider the perfect predictor, where the predicted output equals to the underlying ground truth, i.e., $D_t = \Yori_t$.
The output of the perfect predictor $D_t$ is fully specified by the value of ground truth $\Yori_t$, and therefore is conditionally independent from the protected feature $A_t$ given $\Yori_t$.
Based on the definition of \textit{Equalized Odds} \citep{hardt2016equality}, the prefect predictor is also the best possible Equalized Odds predictor (at time step $t$) with an accuracy of $100\%$.

Furthermore, the joint probability $P_t(d, d', y, y')$ for a perfect predictor is not always positive for any possible combination of $(d, d', y, y')$.
We have the following sufficient condition for this joint probability to be zero:
\vspace{-0.1em}
\begin{equation}\label{equ:perfect_predictor_joint}
    d \neq y, \text{ or } d' \neq y'
    \implies
    P_t (d, d', y, y') = 0.
\end{equation}
With the help of this additional knowledge on the joint probability $P_t(d, d', y, y')$,
we can quantitatively analyze the \textit{Single-step Tier Imbalance Reduction (STIR)} term $\STIR$ and present the following impossibility result for the perfect predictor:
\begin{theorem}\label{theorem:perfect_one_step}
    Let us consider the general situation where both $D_t$ and $\Yori_t$ are dependent with $A_t$, i.e., $D_t \nindep A_t, \Yori_t \nindep A_t$.
    Under Fact \ref{fact:var_form}, Assumption \ref{assumption:Hupdate}, and Assumption \ref{assumption:samegroup}, and Assumption \ref{assumption:biased_weight_quantitative}, as well as the specified dynamics, when $H_t \nindep A_t$, the perfect predictor does not have the potential to get closer to the long-term fairness goal after one-step intervention, i.e.,
    \vspace{-0.1em}
    \begin{equation}\label{equ:perfect_predictor_no_closer}
        D_t = \Yori_t \implies \Delta^{(\text{Perfect Predictor})}_{\mathrm{STIR}} \big\rvert_{t}^{t + 1} > 0.
    \end{equation}
\end{theorem}

Compared to Theorem \ref{theorem:one_step_attainability} that one in general cannot directly attain the long-term fairness goal (balanced tier) through a one-step intervention, Theorem \ref{theorem:perfect_one_step} provides further insights regarding what mission is possible through repetitive one-step interventions.
In particular, under relatively mild assumptions, Theorem \ref{theorem:perfect_one_step} establishes the impossibility of even getting closer to the long-term fairness goal through one-step interventions with the perfect predictor.

\begin{figure*}[t]
    \centering
    \captionsetup{format=hang}
    \includegraphics[width=0.95\textwidth]{
        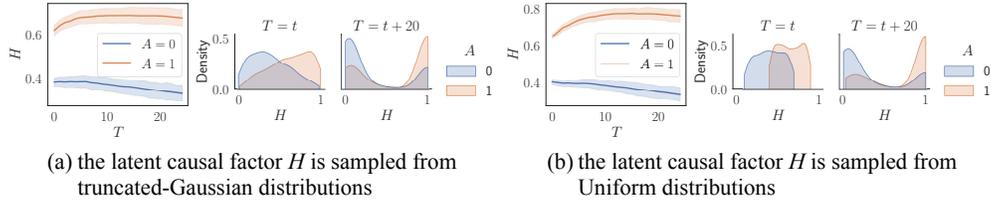}
    \caption{
        Illustration of the interplay between decision with perfect predictors and data dynamics (20 steps) on simulated data, with different initialization
        of tier $H_t$.
    }
    \label{fig:long_term_sim_perfect}
\end{figure*}

\subsubsection{One-step Analysis on Counterfactual Fair Predictor}\label{sec:one_step_analysis_CF}
In this section, we consider the \textit{Counterfactual Fair} \citep{kusner2017counterfactual} predictor.
Similar to the one-step analysis on perfect predictors, we need to make use of the characteristic of Counterfactual Fair predictors to simplify the quantitative analysis on the term $\STIR$.

The definition of Counterfactual Fairness requires the predictor to satisfy $g^D_t(0, E_t) = g^D_t(1, E_t)$ within each time step $T = t$.
Therefore, we have the following sufficient condition for the joint probability $P_t(d, d', y, y')$ to be zero:
\vspace{-0.2em}
\begin{equation}\label{equ:CF_predictor_joint}
    d \neq d' \implies P_t (d, d', y, y') = 0.
\end{equation}
\vspace{-1.5em}
\begin{theorem}\label{theorem:CF_one_step}
    Let us consider the general situation where both $D_t$ and $\Yori_t$ are dependent with $A_t$, i.e., $D_t \nindep A_t, \Yori_t \nindep A_t$.
    Let us further assume that the data dynamics satisfies $\alpha_D \in (0,\frac{1}{2}), \alpha_Y = 0$.
    Then under Fact \ref{fact:var_form}, Assumption \ref{assumption:Hupdate}, Assumption \ref{assumption:samegroup}, and Assumption \ref{assumption:biased_weight_quantitative}, as well as the specified dynamics, when $H_t \nindep A_t$, it is possible for the Counterfactual Fair predictor to get closer to the long-term fairness goal after one-step intervention, if certain properties of the data dynamics and the predictor behavior are satisfied simultaneously, i.e.,
    \vspace{-0.3em}
    \begin{equation}
        \begin{split}
            &\begin{cases}
                & g^D_t(0, E_t) = g^D_t(1, E_t), \frac{P_t(1, 1, 0, 1) + P_t(1, 1, 1, 0)}{P_t(0, 0, 0, 1) + P_t(0, 0, 1, 0)} < \frac{27}{8} \\
                & \alpha_D \in \left( \big(\frac{P_t(1, 1, 0, 1) + P_t(1, 1, 1, 0)}{P_t(0, 0, 0, 1) + P_t(0, 0, 1, 0)}\big)^{\frac{1}{3}} - 1, \frac{1}{2} \right), \alpha_Y = 0
            \end{cases}
            \implies \Delta^{(\text{Counterfactual Fair})}_{\mathrm{STIR}} \big\rvert_{t}^{t + 1} < 0.
        \end{split}
    \end{equation}
    \vspace{-1.5em}
\end{theorem}


Theorem \ref{theorem:CF_one_step} demonstrates the possibility (not guarantee) of getting closer to the long-term fairness goal through one-step interventions (under certain conditions) with Counterfactual Fair predictors.
Compared to the general impossibility results for perfect predictors (Theorem \ref{theorem:perfect_one_step}), there are additional requirements (on both data dynamics and $P_t(d, d', y, y')$) accompanying the possibility result for Counterfactual Fair predictors.
That being said, Theorem \ref{theorem:CF_one_step} clearly illustrates that the understanding of data dynamics through a detailed causal modeling, combined with a suitable decision-making strategy, can provide us with a promising way to approach the long-term dynamic fairness goal ($K$\textit{-Step Tier Balancing} with $K > 1$), step by step.

    \section{Experiments}\label{sec:exp}
In this section, we present experimental results on both simulated and real-world FICO data set from \citet{board2007report}.\footnote{
    Our code repository is available on Github: \url{https://github.com/zeyutang/TierBalancing}.}
In the sequence of decision-distribution interplay, the latent causal factor $H_T$ is updated according to the specified dynamics (Equation \ref{equ:Hupdate}) at each time step.
The output of the decision policy (at each time step) depends on the specific scenario.
In particular, we consider perfect predictors and Counterfactual Fair \citep[Level 1 implementation]{kusner2017counterfactual} predictors.

\begin{figure*}[t]
    \centering
    \captionsetup{format=hang}
    \includegraphics[width=1.\textwidth]{
        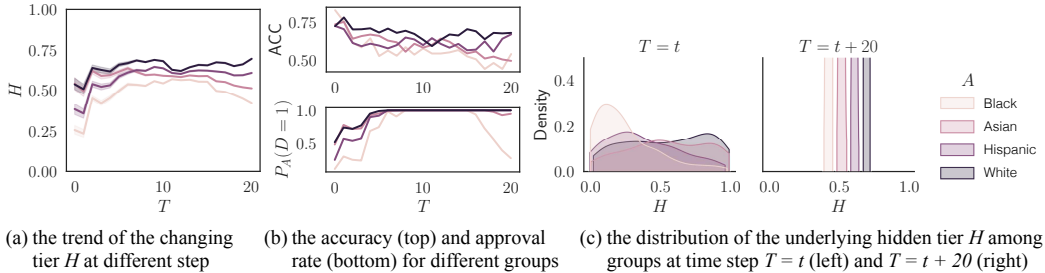}
    \caption{Illustration of the interplay between decision with Counterfactual Fair predictors and the data dynamics (20 steps) on the credit score data set.
    Panel (a) and (b) present the step-by-step tracks of update in tier, accuracy, and approval rates for different groups; panel (c) presents group-conditioned distributions of tier before (left) and after (right) 20 steps of interventions.
    The legend is shared across panel (a), (b), and (c).
    }
    \label{fig:long_term_creditscore_CF}
\end{figure*}

\subsection{Decision with Perfect Predictors on Simulated Data}
In Figure \ref{fig:long_term_sim_perfect} we present the 20-step interplay between decision and the underlying data generating process on the simulated data.
The distributions of $H_t$ for different groups are initialized with truncated Gaussian distributions and Uniform distributions, respectively.
During each time step $T = t$ we generate ground truth labels $\Yori_t$ according to data dynamics specified in \Cref{sec:model_quantitative} and set the decision $D_t$ to be equal to the ground truth $\Yori_t$ (perfect predictor as the decision-making policy); then the pair of $(\Yori_t, D_t)$ are utilized by the data dynamics to determine the tier $H_{t + 1}$ for next step.
As we can see from the left-hand-side figures in panels (a) and (b), the gap between tier for different groups is enlarged as the time goes by.
This indicates that interventions through decision with perfect predictors did not get closer to the long-term fairness goal.

\subsection{Decision with Counterfactual Fair Predictors on Credit Score Data}
The FICO credit score data set contains 301,536 records of TransUnion credit score from 2003 \citep{board2007report}.
In the preprocessed credit score data set \citep{hardt2016equality}, we convert the cumulative distribution function (CDF) of TransRisk score among different groups into group-wise density distributions of the credit score, and use them as the initial tier distributions for different groups.

In Figure \ref{fig:long_term_creditscore_CF} we present the summary of a 20-step interplay between decision with Counterfactual Fair predictors and the underlying data generating process on the credit score data set.
The Counterfactual Fair decision-making strategy is retrained after each one-step data dynamics.
\reBlue{C10}{Figure 3(a) for the Counterfactual Fair decision-making strategy indicates the potential (not guarantee) of getting closer to the long-term fairness goal, e.g., the gap is decreasing until around step 12.}{0in}\rightSide
From Figure \ref{fig:long_term_creditscore_CF}(a) we can observe that the gap between step-by-step tracks of tiers for different groups actually decreases before increasing again (around step 12).
This indicates that decision with Counterfactual Fair predictors does have the potential to get closer to the long-term fairness goal, if the data dynamics and the initial condition (for each one-step analysis at different time step) satisfy certain properties.
Figure \ref{fig:long_term_creditscore_CF}(a) also highlight the importance of our dynamic fairness analyzing framework: if one does not model the interplay between decision and data dynamics, one may well deviate from long-term fairness goal even after making some progress by getting closer to the fairness objective.
    \section{Conclusion and Future Work}\label{sec:conclusion}
In this paper, we propose \textit{Tier Balancing}, a dynamic fairness notion that characterizes the decision-distribution interplay through a detailed causal model over both observed variables and underlying causal factors.
We characterize \textit{Tier Balancing} in terms of a one-step analysis framework on \textit{Single-step Tier Imbalance Reduction (STIR)}.
We show that in general one cannot directly achieve the long-term fairness goal only through a one-step intervention, i.e., static decision-making.
We further show that under certain conditions it is possible (but not guaranteed) for one to get closer to the long-term fairness goal with (a sequence of) Counterfactual Fair decisions.

Our results highlight the challenges and opportunities of enforcing a fairness notion that has built-in capacity to model decision-distribution interplay over underlying causal factors.
Future works naturally include developing algorithms to effectively and efficiently enforce the \textit{Tier Balancing} notion of long-term, dynamic fairness for various practical scenarios.

    \subsection*{Ethics Statement}
The motivation of our work is to pursue long-term fairness.
The conduct of research is under full awareness of, and with adherence to, ICLR Code of Ethics.
We focus on the decision-distribution interplay and present a detailed causal modeling on both observed variables and latent causal factors.
We present challenges and opportunities offered by this new type of fairness endeavor, and hope our work can inspire further research to promote fairness in the long run.
    \subsection*{Acknowledgement}\label{sec:acknowledgement}
This project was partially supported by the National Institutes of Health (NIH) under Contract R01HL159805, by the NSF-Convergence Accelerator Track-D award \#2134901, by a grant from Apple Inc., a grant from KDDI Research Inc., and generous gifts from Salesforce Inc., Microsoft Research, and Amazon Research.
YC and YL are partially supported by the National Science Foundation (NSF) under grants IIS-2143895 and IIS-2040800, and CCF-2023495.

    \stopcontents[mainpaper]

    \bibliography{references}

\begin{thebibliography}{57}
\providecommand{\natexlab}[1]{#1}
\providecommand{\url}[1]{\texttt{#1}}
\expandafter\ifx\csname urlstyle\endcsname\relax
  \providecommand{\doi}[1]{doi: #1}\else
  \providecommand{\doi}{doi: \begingroup \urlstyle{rm}\Url}\fi

\bibitem[Adams et~al.(2021)Adams, Hansen, and Zhang]{adams2021identification}
Jeffrey Adams, Niels Hansen, and Kun Zhang.
\newblock Identification of partially observed linear causal models: Graphical conditions for the non-gaussian and heterogeneous cases.
\newblock \emph{Advances in Neural Information Processing Systems}, 34, 2021.

\bibitem[Bechavod et~al.(2019)Bechavod, Ligett, Roth, Waggoner, and Wu]{bechavod2019equal}
Yahav Bechavod, Katrina Ligett, Aaron Roth, Bo~Waggoner, and Steven~Z Wu.
\newblock Equal opportunity in online classification with partial feedback.
\newblock \emph{Advances in Neural Information Processing Systems}, 32, 2019.

\bibitem[{Board of Governors of the Federal Reserve System (US)}(2007)]{board2007report}
{Board of Governors of the Federal Reserve System (US)}.
\newblock \emph{Report to the congress on credit scoring and its effects on the availability and affordability of credit}.
\newblock Board of Governors of the Federal Reserve System, 2007.

\bibitem[Calders et~al.(2009)Calders, Kamiran, and Pechenizkiy]{calders2009building}
Toon Calders, Faisal Kamiran, and Mykola Pechenizkiy.
\newblock Building classifiers with independency constraints.
\newblock In \emph{2009 IEEE International Conference on Data Mining Workshops}, pp.\  13--18. IEEE, 2009.

\bibitem[Chiappa(2019)]{chiappa2019path}
Silvia Chiappa.
\newblock Path-specific counterfactual fairness.
\newblock In \emph{Proceedings of the AAAI Conference on Artificial Intelligence}, volume~33, pp.\  7801--7808, 2019.

\bibitem[Coston et~al.(2019)Coston, Ramamurthy, Wei, Varshney, Speakman, Mustahsan, and Chakraborty]{coston2019fair}
Amanda Coston, Karthikeyan~Natesan Ramamurthy, Dennis Wei, Kush~R. Varshney, Skyler Speakman, Zairah Mustahsan, and Supriyo Chakraborty.
\newblock Fair transfer learning with missing protected attributes.
\newblock In \emph{Proceedings of the 2019 AAAI/ACM Conference on AI, Ethics, and Society}, pp.\  91–98, 2019.

\bibitem[Coston et~al.(2020)Coston, Mishler, Kennedy, and Chouldechova]{coston2020counterfactual}
Amanda Coston, Alan Mishler, Edward~H Kennedy, and Alexandra Chouldechova.
\newblock Counterfactual risk assessments, evaluation, and fairness.
\newblock In \emph{Proceedings of the 2020 Conference on Fairness, Accountability, and Transparency}, pp.\  582--593, 2020.

\bibitem[Creager et~al.(2020)Creager, Madras, Pitassi, and Zemel]{creager2020causal}
Elliot Creager, David Madras, Toniann Pitassi, and Richard Zemel.
\newblock Causal modeling for fairness in dynamical systems.
\newblock In \emph{International Conference on Machine Learning}, pp.\  2185--2195. PMLR, 2020.

\bibitem[D'Amour et~al.(2020)D'Amour, Srinivasan, Atwood, Baljekar, Sculley, and Halpern]{d2020fairness}
Alexander D'Amour, Hansa Srinivasan, James Atwood, Pallavi Baljekar, D~Sculley, and Yoni Halpern.
\newblock Fairness is not static: deeper understanding of long term fairness via simulation studies.
\newblock In \emph{Proceedings of the 2020 Conference on Fairness, Accountability, and Transparency}, pp.\  525--534, 2020.

\bibitem[Dawkins \& Davis(2017)Dawkins and Davis]{dawkins2017selfish}
Richard Dawkins and Nicola Davis.
\newblock \emph{The selfish gene}.
\newblock Macat Library, 2017.

\bibitem[Elzayn et~al.(2019)Elzayn, Jabbari, Jung, Kearns, Neel, Roth, and Schutzman]{elzayn2019fair}
Hadi Elzayn, Shahin Jabbari, Christopher Jung, Michael Kearns, Seth Neel, Aaron Roth, and Zachary Schutzman.
\newblock Fair algorithms for learning in allocation problems.
\newblock In \emph{Proceedings of the Conference on Fairness, Accountability, and Transparency}, pp.\  170--179, 2019.

\bibitem[Ensign et~al.(2018)Ensign, Friedler, Neville, Scheidegger, and Venkatasubramanian]{ensign2018runaway}
Danielle Ensign, Sorelle~A Friedler, Scott Neville, Carlos Scheidegger, and Suresh Venkatasubramanian.
\newblock Runaway feedback loops in predictive policing.
\newblock In \emph{Conference on Fairness, Accountability and Transparency}, pp.\  160--171. PMLR, 2018.

\bibitem[Friedman \& Sinervo(2016)Friedman and Sinervo]{friedman2016evolutionary}
Daniel Friedman and Barry Sinervo.
\newblock \emph{Evolutionary games in natural, social, and virtual worlds}.
\newblock Oxford University Press, 2016.

\bibitem[Ge et~al.(2021)Ge, Liu, Gao, Xian, Li, Zhao, Pei, Sun, Ge, Ou, and Zhang]{ge2021towards}
Yingqiang Ge, Shuchang Liu, Ruoyuan Gao, Yikun Xian, Yunqi Li, Xiangyu Zhao, Changhua Pei, Fei Sun, Junfeng Ge, Wenwu Ou, and Yongfeng Zhang.
\newblock Towards long-term fairness in recommendation.
\newblock In \emph{Proceedings of the 14th ACM International Conference on Web Search and Data Mining}, pp.\  445--453, 2021.

\bibitem[Hardt et~al.(2016)Hardt, Price, and Srebro]{hardt2016equality}
Moritz Hardt, Eric Price, and Nati Srebro.
\newblock Equality of opportunity in supervised learning.
\newblock In \emph{Advances in Neural Information Processing Systems}, pp.\  3315--3323, 2016.

\bibitem[Hashimoto et~al.(2018)Hashimoto, Srivastava, Namkoong, and Liang]{hashimoto2018fairness}
Tatsunori Hashimoto, Megha Srivastava, Hongseok Namkoong, and Percy Liang.
\newblock Fairness without demographics in repeated loss minimization.
\newblock In \emph{International Conference on Machine Learning}, pp.\  1929--1938. PMLR, 2018.

\bibitem[Heidari \& Krause(2018)Heidari and Krause]{heidari2018preventing}
Hoda Heidari and Andreas Krause.
\newblock Preventing disparate treatment in sequential decision making.
\newblock In \emph{IJCAI}, pp.\  2248--2254, 2018.

\bibitem[Heidari et~al.(2019)Heidari, Nanda, and Gummadi]{heidari2019long}
Hoda Heidari, Vedant Nanda, and Krishna Gummadi.
\newblock On the long-term impact of algorithmic decision policies: Effort unfairness and feature segregation through social learning.
\newblock In \emph{36th International Conference on Machine Learning}, pp.\  2692--2701, 2019.

\bibitem[Hu \& Chen(2018)Hu and Chen]{hu2018short}
Lily Hu and Yiling Chen.
\newblock A short-term intervention for long-term fairness in the labor market.
\newblock In \emph{Proceedings of the 2018 World Wide Web Conference}, pp.\  1389--1398, 2018.

\bibitem[Imai \& Jiang(2020)Imai and Jiang]{imai2020principal}
Kosuke Imai and Zhichao Jiang.
\newblock Principal fairness for human and algorithmic decision-making.
\newblock \emph{arXiv preprint arXiv:2005.10400}, 2020.

\bibitem[Jabbari et~al.(2017)Jabbari, Joseph, Kearns, Morgenstern, and Roth]{jabbari2017fairness}
Shahin Jabbari, Matthew Joseph, Michael Kearns, Jamie Morgenstern, and Aaron Roth.
\newblock Fairness in reinforcement learning.
\newblock In \emph{International Conference on Machine Learning}, pp.\  1617--1626. PMLR, 2017.

\bibitem[Joseph et~al.(2016)Joseph, Kearns, Morgenstern, and Roth]{joseph2016fairness}
Matthew Joseph, Michael Kearns, Jamie~H Morgenstern, and Aaron Roth.
\newblock Fairness in learning: Classic and contextual bandits.
\newblock \emph{Advances in Neural Information Processing Systems}, 29, 2016.

\bibitem[Kamiran et~al.(2013)Kamiran, {\v{Z}}liobait{\.e}, and Calders]{kamiran2013quantifying}
Faisal Kamiran, Indr{\.e} {\v{Z}}liobait{\.e}, and Toon Calders.
\newblock Quantifying explainable discrimination and removing illegal discrimination in automated decision making.
\newblock \emph{Knowledge and Information Systems}, 35\penalty0 (3):\penalty0 613--644, 2013.

\bibitem[Kannan et~al.(2019)Kannan, Roth, and Ziani]{kannan2019downstream}
Sampath Kannan, Aaron Roth, and Juba Ziani.
\newblock Downstream effects of affirmative action.
\newblock In \emph{Proceedings of the Conference on Fairness, Accountability, and Transparency}, pp.\  240--248, 2019.

\bibitem[Kearns \& Roth(2019)Kearns and Roth]{kearns2019ethical}
Michael Kearns and Aaron Roth.
\newblock \emph{The ethical algorithm: The science of socially aware algorithm design}.
\newblock Oxford University Press, 2019.

\bibitem[Kilbertus et~al.(2017)Kilbertus, Rojas-Carulla, Parascandolo, Hardt, Janzing, and Sch{\"o}lkopf]{kilbertus2017avoiding}
Niki Kilbertus, Mateo Rojas-Carulla, Giambattista Parascandolo, Moritz Hardt, Dominik Janzing, and Bernhard Sch{\"o}lkopf.
\newblock Avoiding discrimination through causal reasoning.
\newblock In \emph{Proceedings of the 31st International Conference on Neural Information Processing Systems}, pp.\  656--666, 2017.

\bibitem[Kivva et~al.(2021)Kivva, Rajendran, Ravikumar, and Aragam]{kivva2021learning}
Bohdan Kivva, Goutham Rajendran, Pradeep Ravikumar, and Bryon Aragam.
\newblock Learning latent causal graphs via mixture oracles.
\newblock \emph{Advances in Neural Information Processing Systems}, 34, 2021.

\bibitem[Kusner et~al.(2017)Kusner, Loftus, Russell, and Silva]{kusner2017counterfactual}
Matt~J Kusner, Joshua Loftus, Chris Russell, and Ricardo Silva.
\newblock Counterfactual fairness.
\newblock In \emph{Advances in Neural Information Processing Systems}, pp.\  4066--4076, 2017.

\bibitem[Liu et~al.(2018)Liu, Dean, Rolf, Simchowitz, and Hardt]{liu2018delayed}
Lydia~T Liu, Sarah Dean, Esther Rolf, Max Simchowitz, and Moritz Hardt.
\newblock Delayed impact of fair machine learning.
\newblock In \emph{International Conference on Machine Learning}, pp.\  3150--3158. PMLR, 2018.

\bibitem[Liu et~al.(2017)Liu, Radanovic, Dimitrakakis, Mandal, and Parkes]{liu2017calibrated}
Yang Liu, Goran Radanovic, Christos Dimitrakakis, Debmalya Mandal, and David~C Parkes.
\newblock Calibrated fairness in bandits.
\newblock \emph{arXiv preprint arXiv:1707.01875}, 2017.

\bibitem[Liu et~al.(2021)Liu, Chen, Tang, and Zhang]{liu2021induced}
Yang Liu, Yatong Chen, Zeyu Tang, and Kun Zhang.
\newblock Model transferability with responsive decision subjects.
\newblock \emph{arXiv preprint arXiv:2107.05911}, 2021.

\bibitem[Mhasawade \& Chunara(2021)Mhasawade and Chunara]{mhasawade2021causal}
Vishwali Mhasawade and Rumi Chunara.
\newblock Causal multi-level fairness.
\newblock In \emph{Proceedings of the 2021 AAAI/ACM Conference on AI, Ethics, and Society}, pp.\  784--794, 2021.

\bibitem[Miller et~al.(2020)Miller, Milli, and Hardt]{miller2020strategic}
John Miller, Smitha Milli, and Moritz Hardt.
\newblock Strategic classification is causal modeling in disguise.
\newblock In \emph{International Conference on Machine Learning}, pp.\  6917--6926. PMLR, 2020.

\bibitem[Mishler et~al.(2021)Mishler, Kennedy, and Chouldechova]{mishler2021fairness}
Alan Mishler, Edward~H Kennedy, and Alexandra Chouldechova.
\newblock Fairness in risk assessment instruments: Post-processing to achieve counterfactual equalized odds.
\newblock In \emph{Proceedings of the 2021 ACM Conference on Fairness, Accountability, and Transparency}, pp.\  386--400, 2021.

\bibitem[Mouzannar et~al.(2019)Mouzannar, Ohannessian, and Srebro]{mouzannar2019fair}
Hussein Mouzannar, Mesrob~I Ohannessian, and Nathan Srebro.
\newblock From fair decision making to social equality.
\newblock In \emph{Proceedings of the Conference on Fairness, Accountability, and Transparency}, pp.\  359--368, 2019.

\bibitem[Nilforoshan et~al.(2022)Nilforoshan, Gaebler, Shroff, and Goel]{nilforoshan2022causal}
Hamed Nilforoshan, Johann~D Gaebler, Ravi Shroff, and Sharad Goel.
\newblock Causal conceptions of fairness and their consequences.
\newblock In \emph{International Conference on Machine Learning}, pp.\  16848--16887. PMLR, 2022.

\bibitem[Pearl(2009)]{pearl2009causality}
Judea Pearl.
\newblock \emph{Causality}.
\newblock Cambridge University Press, 2009.

\bibitem[Perdomo et~al.(2020)Perdomo, Zrnic, Mendler-D{\"u}nner, and Hardt]{perdomo2020performative}
Juan Perdomo, Tijana Zrnic, Celestine Mendler-D{\"u}nner, and Moritz Hardt.
\newblock Performative prediction.
\newblock In \emph{International Conference on Machine Learning}, pp.\  7599--7609. PMLR, 2020.

\bibitem[Rezaei et~al.(2021)Rezaei, Liu, Memarrast, and Ziebart]{rezaei2021robust}
Ashkan Rezaei, Anqi Liu, Omid Memarrast, and Brian~D Ziebart.
\newblock Robust fairness under covariate shift.
\newblock In \emph{Proceedings of the AAAI Conference on Artificial Intelligence}, volume~35, pp.\  9419--9427, 2021.

\bibitem[Russell et~al.(2017)Russell, Kusner, Loftus, and Silva]{russell2017worlds}
Chris Russell, Matt~J Kusner, Joshua Loftus, and Ricardo Silva.
\newblock When worlds collide: integrating different counterfactual assumptions in fairness.
\newblock In \emph{Advances in Neural Information Processing Systems}, pp.\  6414--6423, 2017.

\bibitem[Schumann et~al.(2019)Schumann, Wang, Beutel, Chen, Qian, and Chi]{schumann2019transfer}
Candice Schumann, Xuezhi Wang, Alex Beutel, Jilin Chen, Hai Qian, and Ed~H Chi.
\newblock Transfer of machine learning fairness across domains.
\newblock \emph{arXiv preprint arXiv:1906.09688}, 2019.

\bibitem[Shimizu et~al.(2006)Shimizu, Hoyer, Hyv{\"a}rinen, Kerminen, and Jordan]{shimizu2006linear}
Shohei Shimizu, Patrik~O Hoyer, Aapo Hyv{\"a}rinen, Antti Kerminen, and Michael Jordan.
\newblock A linear non-gaussian acyclic model for causal discovery.
\newblock \emph{Journal of Machine Learning Research}, 7\penalty0 (10), 2006.

\bibitem[Siddique et~al.(2020)Siddique, Weng, and Zimmer]{siddique2020learning}
Umer Siddique, Paul Weng, and Matthieu Zimmer.
\newblock Learning fair policies in multi-objective (deep) reinforcement learning with average and discounted rewards.
\newblock In \emph{International Conference on Machine Learning}, pp.\  8905--8915. PMLR, 2020.

\bibitem[Singh et~al.(2021)Singh, Singh, Mhasawade, and Chunara]{singh2021fairness}
Harvineet Singh, Rina Singh, Vishwali Mhasawade, and Rumi Chunara.
\newblock Fairness violations and mitigation under covariate shift.
\newblock In \emph{Proceedings of the 2021 ACM Conference on Fairness, Accountability, and Transparency}, pp.\  3--13, 2021.

\bibitem[Spirtes et~al.(1993)Spirtes, Glymour, and Scheines]{spirtes1993causation}
Peter Spirtes, Clark~N Glymour, and Richard Scheines.
\newblock \emph{Causation, prediction, and search}.
\newblock Springer New York, 1993.

\bibitem[Ustun et~al.(2019)Ustun, Spangher, and Liu]{ustun2019actionable}
Berk Ustun, Alexander Spangher, and Yang Liu.
\newblock Actionable recourse in linear classification.
\newblock In \emph{Proceedings of the Conference on Fairness, Accountability, and Transparency}, pp.\  10--19, 2019.

\bibitem[von K{\"u}gelgen et~al.(2022)von K{\"u}gelgen, Karimi, Bhatt, Valera, Weller, and Sch{\"o}lkopf]{von2022fairness}
Julius von K{\"u}gelgen, Amir-Hossein Karimi, Umang Bhatt, Isabel Valera, Adrian Weller, and Bernhard Sch{\"o}lkopf.
\newblock On the fairness of causal algorithmic recourse.
\newblock In \emph{Proceedings of the AAAI Conference on Artificial Intelligence}, volume~36, pp.\  9584--9594, 2022.

\bibitem[Wen et~al.(2021)Wen, Bastani, and Topcu]{wen2021algorithms}
Min Wen, Osbert Bastani, and Ufuk Topcu.
\newblock Algorithms for fairness in sequential decision making.
\newblock In \emph{International Conference on Artificial Intelligence and Statistics}, pp.\  1144--1152. PMLR, 2021.

\bibitem[Wu et~al.(2019)Wu, Zhang, Wu, and Tong]{wu2019pc}
Yongkai Wu, Lu~Zhang, Xintao Wu, and Hanghang Tong.
\newblock Pc-fairness: A unified framework for measuring causality-based fairness.
\newblock In \emph{Advances in Neural Information Processing Systems}, pp.\  3399--3409, 2019.

\bibitem[Xie et~al.(2020)Xie, Cai, Huang, Glymour, Hao, and Zhang]{xie2020generalized}
Feng Xie, Ruichu Cai, Biwei Huang, Clark Glymour, Zhifeng Hao, and Kun Zhang.
\newblock Generalized independent noise condition for estimating latent variable causal graphs.
\newblock \emph{Advances in Neural Information Processing Systems}, 33:\penalty0 14891--14902, 2020.

\bibitem[Xie et~al.(2022)Xie, Huang, Chen, He, Geng, and Zhang]{xie2022identification}
Feng Xie, Biwei Huang, Zhengming Chen, Yangbo He, Zhi Geng, and Kun Zhang.
\newblock Identification of linear non-gaussian latent hierarchical structure.
\newblock In \emph{International Conference on Machine Learning}, pp.\  24370--24387. PMLR, 2022.

\bibitem[Zhang \& Hyv{\"a}rinen(2009)Zhang and Hyv{\"a}rinen]{zhang2009identifiability}
Kun Zhang and Aapo Hyv{\"a}rinen.
\newblock On the identifiability of the post-nonlinear causal model.
\newblock In \emph{25th Conference on Uncertainty in Artificial Intelligence (UAI 2009)}, pp.\  647--655. AUAI Press, 2009.

\bibitem[Zhang et~al.(2011)Zhang, Peters, Janzing, and Sch{\"o}lkopf]{zhang2011kernel}
Kun Zhang, Jonas Peters, Dominik Janzing, and Bernhard Sch{\"o}lkopf.
\newblock Kernel-based conditional independence test and application in causal discovery.
\newblock In \emph{27th Conference on Uncertainty in Artificial Intelligence (UAI 2011)}, pp.\  804--813. AUAI Press, 2011.

\bibitem[Zhang et~al.(2017)Zhang, Wu, and Wu]{zhang2017causal}
Lu~Zhang, Yongkai Wu, and Xintao Wu.
\newblock A causal framework for discovering and removing direct and indirect discrimination.
\newblock In \emph{Proceedings of the 26th International Joint Conference on Artificial Intelligence}, pp.\  3929--3935, 2017.

\bibitem[Zhang et~al.(2019)Zhang, Khaliligarekani, Tekin, and Liu]{zhang2019group}
Xueru Zhang, Mohammadmahdi Khaliligarekani, Cem Tekin, and Mingyan Liu.
\newblock Group retention when using machine learning in sequential decision making: the interplay between user dynamics and fairness.
\newblock \emph{Advances in Neural Information Processing Systems}, 32:\penalty0 15269--15278, 2019.

\bibitem[Zhang et~al.(2020)Zhang, Tu, Liu, Liu, Kjellstrom, Zhang, and Zhang]{zhang2020fair}
Xueru Zhang, Ruibo Tu, Yang Liu, Mingyan Liu, Hedvig Kjellstrom, Kun Zhang, and Cheng Zhang.
\newblock How do fair decisions fare in long-term qualification?
\newblock \emph{Advances in Neural Information Processing Systems}, 33:\penalty0 18457--18469, 2020.

\bibitem[Zimmer et~al.(2021)Zimmer, Glanois, Siddique, and Weng]{zimmer2021learning}
Matthieu Zimmer, Claire Glanois, Umer Siddique, and Paul Weng.
\newblock Learning fair policies in decentralized cooperative multi-agent reinforcement learning.
\newblock In \emph{International Conference on Machine Learning}, pp.\  12967--12978. PMLR, 2021.

\end{thebibliography}
    \bibliographystyle{iclr2023_conference}
\end{bibunit}

\newpage

\appendix

\title{Supplement to \\``Tier Balancing: Towards Dynamic Fairness over Underlying Causal Factors''}
\settitle

\begin{bibunit}
    \counterwithin{equation}{section}
    \renewcommand{\theequation}{\thesection.\arabic{equation}}

    \listoftodos[\colorbox{white}{List of Responses to Reviewers}]\label{list:toc_responses}

    \startcontents[supplement]
    \renewcommand\contentsname{Table of Contents: Appendix}
    \printcontents[supplement]{l}{1}{\section*{\contentsname}\setcounter{tocdepth}{2}}

    \listoftables



\section{Detailed Discussions on Related Works}\label{appendix:discussion_related_works}
In this section, we provide detailed discussions on related works.
In particular, considering our focus on providing a novel long-term fairness notion with the help of the detailed causal modeling of involved dynamics, we compare our work with previous literature on causal notions of fairness, as well as fairness inquiries in dynamic settings.

\subsection{Causal Notions of Fairness}
Various causal notions of algorithmic fairness have been proposed in the literature, for instance,
fairness notions defined in terms of the (non-)existence of certain causal paths in the graph \citep{kamiran2013quantifying,kilbertus2017avoiding,zhang2017causal},
fairness notions defined through estimating or bounding causal effects \citep{kusner2017counterfactual,chiappa2019path,wu2019pc,mhasawade2021causal},
fairness notions defined with respect to statistics on certain factual/counterfactual groups \citep{imai2020principal,coston2020counterfactual,mishler2021fairness}.
The proposed causal notions audit fairness in an instantaneous manner, i.e., the fairness inquires are with respect to a snapshot of reality, and the scope of consideration is limited to observed variables only.
Our \textit{Tier Balancing} notion has a built-in capacity to inquire fairness in the long-term and dynamic setting, which is very different from instantaneous causal fairness notions (beyond the fact that our notion encompasses latent causal factors).

While we can detect and measure discrimination based on previous (instantaneous) causal notions of fairness (e.g., the existence of certain causal paths or causal effects), eliminating such existence of causal paths or causal effects is a valid goal to achieve but might not be the means one should opt for.
To begin with, there is no guarantee that eliminating a causal path or effect results in non-existence of such causal path or effect in the future under the interplay between decision-making and data dynamics.
Furthermore, the data generating processes represented in the causal model might not be easily manipulable under the same timescale of decision-making (e.g., the ones governed by nature and/or the mode and structure of a society).
One cannot expect that the manipulation on the causal model (for the purpose of enforcing fairness notions) directly translate to real-world changes in the underlying data generating processes.

Different from previous causal fairness notions, instead of directly ``going against'' the underlying data generating process (e.g., by eliminating certain causal path or causal effect), our \textit{Tier Balancing} notion encourages ``working with'' the underlying data generating processes.
With a detailed causal modeling of the decision-distribution interplay, \textit{Tier Balancing} emphases on the possibility of inducing a future data distribution that is fair in the long run.

\subsection{Fairness Inquires in Dynamic Settings}\label{appendix:dynamic_fairness_literature}
Previous literature have considered dynamic fairness in specific practical scenarios, for instance, opportunity allocation in labor market \citep{hu2018short}, a pipeline consisting of college admission followed by hiring \citep{kannan2019downstream}, opportunity allocation in credit application \citep{liu2018delayed}, and resource allocation in predictive policing \citep{ensign2018runaway}.
Different from previous literature, we present a detailed causal modeling of the decision-distribution interplay that is general enough to be applicable in various resource allocation problems (e.g., loan applications, hiring practices) while also being specific enough to encompass nuances in data dynamics for the particular practical scenario of interest.

In terms of the analyzing framework, closely related works have considered the one-step analysis \citep{liu2018delayed,kannan2019downstream,mouzannar2019fair,zhang2019group}.
However, previous works focus on the long-term effect of imposing certain fairness notions that are readily available, for example, \textit{Demographic Parity} \citep{calders2009building,liu2018delayed,mouzannar2019fair} and \textit{Equal Opportunity} \citep{hardt2016equality,liu2018delayed}.
In our work, we formulate a novel notion of long-term fairness, namely, \textit{Tier Balancing}, and explore the possibility of providing a fairness notion that characterizes the dynamic nature of decision-distribution interplay through detailed causal modeling on both observed variables and latent causal factors.

In terms of the modeling choice for data dynamics, most closely related works model data dynamics using variants of Markov Decision Processes (MDPs) \citep{jabbari2017fairness,siddique2020learning,zhang2020fair,d2020fairness,wen2021algorithms,zimmer2021learning,ge2021towards}.
For example, \citet{zhang2020fair} consider the partially observed Markov decision process (POMDP) model, and conduct evolution and equilibrium analysis with respect to \textit{Demographic Parity} and \textit{Equal Opportunity} notions of fairness.
The dynamics are modeled through transition matrices on group-level qualification rates.
Compared to the modeling of transition matrices in MDPs, our model is more fine-grained and on the individual level, answering the call for ``richer and more complex modelings [of involved dynamics]'' in previous literature \citep{hu2018short}.\reBlue{C3}{Our updating dynamics can incorporate stochastic data generating processes. As an analogy, previous literature includes dynamic modelings with (variants of) MDPs.
While the transition matrix is deterministic, the MDPs can model stochastic processes.}{0in}\rightSide

Another closely related work is the one-step analysis on the impact of causal fairness notions on downstream utilities conducted by \citet{nilforoshan2022causal}.
They consider a detailed causal modeling on the college admission running example and analyze previously proposed (instantaneous) causal fairness notions, namely, \textit{Counterfactual Predictive Parity} \citep{coston2020counterfactual}, \textit{Counterfactual Equalized Odds} \citep{coston2020counterfactual}, and \textit{Conditional Principal Fairness} \citep{imai2020principal}.
Our work is different in several ways:
instead of utilizing a static graph, we focus on the decision-distribution interplay and explicitly capture both observed and latent variables along the temporal axis;
different from analyzing one-step downstream consequence in terms of utility, we formulate a long-term fairness goal and investigate the challenges and opportunities revealed by the notion.

\section{Additional Results, Technical Details, and Discussions}\label{appendix:discussion_results}
In this section, we provide additional results, technical details, and discussions of our work.
In \Cref{appendix:remark_interplay}, we provide additional discussions on our causal modeling of the decision-distribution interplay;
in \Cref{appendix:when_initially_attained}, we analyze the situation where \textit{Tier Balancing} is initially attained;
in \Cref{appendix:remark_FCM}, we discuss the role of exogenous terms and provide a remark on Fact \ref{fact:var_form};
in \Cref{appendix:derive_STIR}, we present the detailed derivation of \textit{Single-step Tier Imbalance Reduction} (STIR) term $\STIR$;
in \Cref{appendix:illustration_assumption}, we illustrate the connection between Assumption \ref{assumption:biased_weight_qualitative} and Assumption \ref{assumption:biased_weight_quantitative};
in \Cref{appendix:additional_exps}, we present additional experimental results;
in \Cref{appendix:limitations}, we discuss potential limitations of our work.

\color{RevisionText}
\subsection{Discussions on the Causal Modeling of Decision-Distribution Interplay}\label{appendix:remark_interplay}
In \Cref{appendix:say_more_about_involved_dynamics}, we provide additional details of the involved dynamics in the causal modeling of the decision-distribution interplay.
In \Cref{appendix:practical_scenario_interest}, we discuss the relation between the practical scenarios and the modeled dynamics.

\subsubsection{Additional Modeling Details of the Involved Dynamics}\label{appendix:say_more_about_involved_dynamics}

We use $X_{t,i}$'s to represent three different patterns (instead of the number of count) of variables with respect to how observed features are caused by the protected feature $A_t$ and the latent causal factor $H_t$.
There are three types of observed features:
(1) features that only have the latent causal factor $H_t$ as the case, e.g., $X_{t, 1}$,
(2) features that have both the latent causal factor $H_t$ and the protected feature $A_t$ as cause, e.g., $X_{t,2}$,
and (3) features that only have the protected feature $A_t$ as the cause, e.g,. $X_{t,3}$.
For conciseness, we omit features that are not relevant to the practical scenario of interest, i.e., variables that are not causally relevant to $(H_t, A_t)$.
One can replace $X_{t,i}$'s with the actual number of additional features together with the causal relations among them in specific practical scenarios.\rePurple{C3}{In light of your comment, we provide additional details on the causal relations among $H_t$, $A_t$, and $X_{t,i}$.}{0in}\rightSide

At every time step $T = t$, the decision-making strategy $D_t$ is trained on the joint distribution $(A_t, X_{t,i}, Y_t^{(\mathrm{obs})})$.
However, when making the decision, $D_t$ only takes $(A_t, X_{t,i})$ as input.
Since we are modeling causal relations in data generating processes, we only include a directed edge in the DAG if there is a causal relation between variables.
Therefore, the data generating process of $D_t$ does not involve an edge between $Y_t^{(\mathrm{obs})}$ and $D_t$.\rePurple{C2}{In light of your comment, we add the discussion and explain why there is no arrow between $Y_t^{(\mathrm{obs})}$ and $D_t$.}{-0.9in}\leftSide

\subsubsection{The Practical Scenarios of Interest}\label{appendix:practical_scenario_interest}
As we can see from previous literature (discussed in \Cref{appendix:dynamic_fairness_literature}),
\reOrange{C2}{We provide additional discussion on practical scenarios where our causal modeling of the decision-distribution interplay can be applied.}{0in}\rightSide
the modeling choices are closely related to the practical scenarios of interest, and therefore, can be very different in terms of modeling details of the involved dynamics in long-term and dynamic settings.

Our causal modeling of repetitive resource application and allocation keeps track of individual-level situation changes, and enables informative and principled analysis on the decision-distribution interplay in different practical scenarios.
For example, 
\reGreen{C4}{Following your suggestion, we provide additional discussion on settings our causal modeling of decision-interplay can be applied.}{0in}\leftSide
in credit application (e.g., \citealt{liu2018delayed}), the agents are clients and the latent causal factor (tier) can be individual's socio-economic status or creditworthiness;
in predictive policing (e.g., \citealt{ensign2018runaway}), the agents are neighborhoods and the latent tier can be neighborhood's safety ratings;
in the dual market pipeline (e.g., temporary labor markets followed by the permanent labor market considered in \citealt{hu2018short}) or the admission-followed-by-hiring pipeline (e.g., \citealt{kannan2019downstream}), the agents are applicants who subject to a sequence of decisions and the latent tier can be the relevant qualification for the school program and the job.

However, when the decision received by the individual is once in a lifetime (or at least very long time compared to the timescale of the decision-making), repeated application and allocation of resource may not be a suitable modeling choice.
For example, college admission decisions are made on a yearly basis but an individual does not repeatedly apply for college every single year \citep{mouzannar2019fair}.
In this case, \reBlue{C1}{Our fairness notion is very different from the one considered by \citet{mouzannar2019fair}. Both approaches have suitable practical scenarios to serve appropriate purposes.}{-1.5in}\rightSide
if we focus on the decision made by a specific college, it is more natural to study changes in the population in terms of the group-level qualification profiles \citep{mouzannar2019fair}.
As another example,
\reGreen{W2}{Thank you for providing the reference! In light of your comment, we incorporate the analysis on potential application/limitation of repeated resource allocation in the context of health care. We note that suitable modeling choices depend on the practical scenarios of interest.}{0in}\rightSide
in the context of health care (e.g., \citealt{mhasawade2021causal}), when the resource takes the form of the medical treatment for the purpose of improving health outcome, not all treatment requires regular doses and therefore, repeated allocation modeled on the individual-level may not be an optimal choice.
One can, for instance, resort to the modeling at the level of subgroups as an alternative \citep{mhasawade2021causal}.

Considering the difference in semantics of fairness in various practical scenarios, previous literature has pointed out that there is in general no one-size-fits-all solution for algorithmic fairness (e.g., \citealt{kearns2019ethical}).
By presenting a detailed causal modeling for the decision-distribution interplay, we do not intend to provide a general framework to encompass long-term fairness considerations in all practical scenarios.
Instead, we would like to demonstrate the opportunities and challenges and hope our work can inspire further research.
\color{black}

\subsection{When Tier Balancing is Initially Satisfied}\label{appendix:when_initially_attained}
In the paper we have presented possibility and impossibility results to achieve, or get closer to, the long-term fairness goal when \textit{Tier Balancing} is not initially satisfied.
It is natural to wonder what we should do if we find out that \textit{Tier Balancing} happen to be satisfied during fairness audit.
In fact, as we shall see in Proposition \ref{proposition:when_initially_attained}, if \textit{Tier Balancing} is satisfied as the initial condition, under the specified dynamics, one can use \textit{Demographic Parity} \citep{calders2009building} decision-making strategy to maintain the status of satisfying \textit{Tier Balancing}.
This indicates that when \textit{Tier Balancing} is satisfied (as a lucky initial condition, or as a result of $K$-step interventions), we have at least one way to maintain the fair state of affairs.
\reOrange{C1}{We reveal the connection between \textit{Tier Balancing} (TB) and \textit{Demographic Parity} (DP) notions of fairness by analyzing the scenario where TB is satisfied initially.}{0in}\leftSide
\begin{proposition}\label{proposition:when_initially_attained}
    When \textit{Tier Balancing} is initially satisfied, i.e., $H_t \indep A_t$, under Fact \ref{fact:var_form}, Assumption \ref{assumption:Hupdate}, and Assumption \ref{assumption:samegroup}, as well as the specified dynamics, the \textit{Demographic Parity} decision-making strategy, i.e., $D_t \indep A_t$, can ensure \textit{Tier Balancing} still holds true for the next time step, i.e., $H_{t + 1} \indep A_{t + 1}$.
\end{proposition}
\begin{proof}
    To begin with, since $H_t \indep A_t$, by Fact \ref{fact:var_form}, $H_t$ is not a function of $A_t$.
    As a direct result, $\Yori_t$ is also not a function of $A_t$ (since the distribution of $\Yori_t$ is fully determined by the value of $H_t = h_t$).
    Besides, since $D_t$ satisfies \textit{Demographic Parity}, $D_t \indep A_t$, and therefore by Fact \ref{fact:var_form}, $D_t$ is not a function of $A_t$.

    According to Assumption \ref{assumption:Hupdate}, under the specified dynamics, $H_{t + 1}$ is fully determined by $(H_t, D_t, \Yori_t)$, among which none of them is a function of $A_t$.
    Then, we have $H_{t + 1}$ is not a function of $A_t$.

    Recall that in the specified dynamics, the same group of agents repetitively apply for credit with the entire group unchanged.
    According to Assumption \ref{assumption:samegroup}, $A_{t + 1}$ is an identical copy of $A_t$.
    Therefore we have $H_{t + 1}$ is not a function of $A_{t + 1}$, i.e., $H_{t + 1} \indep A_{t + 1}$.
\end{proof}

\color{RevisionText}
\subsection{A Remark on Fact \ref{fact:var_form}}\label{appendix:remark_FCM}\rePurple{C5}{Following your suggestion, we add a remark on Fact \ref{fact:var_form}. In particular, we include the definition of the functional causal model (FCM), present the DAG with exogenous terms explicitly modeled, and provide details with respect to the application of FCM in Fact \ref{fact:var_form}.}{0in}\leftSide
Let us first present the definition of a functional causal model \citep{spirtes1993causation,pearl2009causality}:
\begin{definition}[\textbf{Functional Causal Model}]\label{appendix:def:FCM}
    We can represent a causal model with a tuple $(E, V, \Fbf)$ such that:
    \begin{enumerate}[label=(\arabic*)]
        \item $V$ is a set of observed variables involved in the system of interest;
        \item $E$ is a set of exogenous variables that we cannot directly observe but contains the background information representing all other causes of $V$ and jointly follows a distribution $P(E)$;
        \item $\Fbf$ is a set of functions (also known as structural equations) $\{ f_1, f_2, \ldots, f_n\}$ where each $f_i$ corresponds to one variable $V_i \in V$ and is a mapping $E \cup V \setminus \{V_i\} \rightarrow V_i$.
    \end{enumerate}
    The triplet $(E, V, \Fbf)$ is known as the functional causal model (FCM).
    We can also capture causal relations among variables via a directed acyclic graph (DAG), where nodes (vertices) represent variables and edges represent functional relations between variables and the corresponding direct causes (i.e., observed parents and unobserved exogenous terms).
\end{definition}

\begin{figure*}[t]
    \centering
    \captionsetup{format=hang}
    \includegraphics[width=.8\textwidth]{
        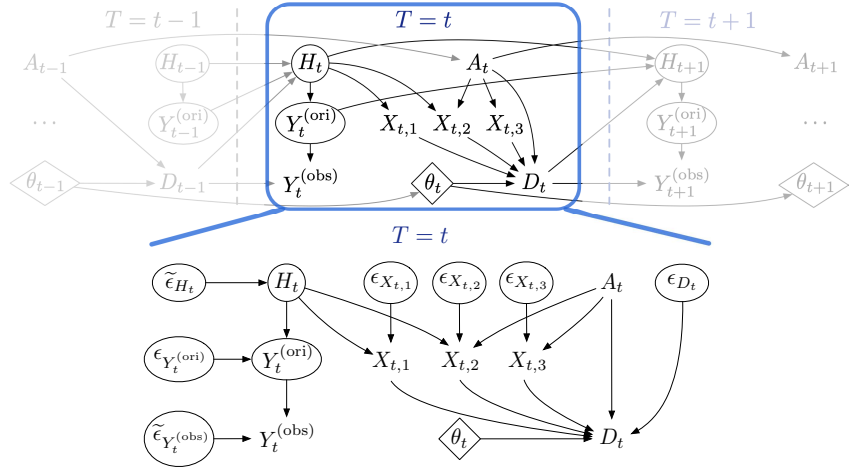}
    \caption{
        The causal modeling of the decision-distribution interplay.
        The circle indicates that the corresponding variable is unobserved.
        We use diamond to denote the underlying causal factor and explicitly indicate the (potential) non-stationary nature of the decision-making strategies across time.
    }
    \label{appendix:fig:dynamic_with_exogenous}
\end{figure*}

For the purpose of illustration, in Figure \ref{appendix:fig:dynamic_with_exogenous} we present the DAG (at time step $T = t$) with the exogenous terms $E_t$ explicitly modeled, where $E_t$ is the concatenation of individual exogenous terms:
\begin{equation}\label{appendix:equ:exogenous_concat}
     E_t = \Big( \widetilde{\epsilon}_{H_t}, \widetilde{\epsilon}_{Y_t^{(\mathrm{obs})}}, \epsilon_{Y_t^{(\mathrm{ori})}}, \epsilon_{X_{t, i}}, \epsilon_{D_t} \Big).
\end{equation}
We use the $~\widetilde{\cdot}~$ symbol on certain exogenous noise terms, e.g., $\widetilde{\epsilon}_{H_t}$ and $\widetilde{\epsilon}_{Y_t^{(\mathrm{obs})}}$, to denote the fact that the corresponding variables are affected by previous time step ($T = t - 1$), and such influence are encapsulated into exogenous terms from the standpoint of current time step ($T = t$).
For example, the influence from the randomness in $D_{t - 1}$ (when $T = t- 1$) on current $Y_t^{(\mathrm{obs})}$ is encapsulated into an exogenous term $\widetilde{\epsilon}_{Y_t^{(\mathrm{obs})}}$ when $T = t$.\reBlue{C5}{From the standpoint of current time step $T = t$, the historical influence can be modeled as exogenous randomness, e.g., $\widetilde{\epsilon}_{H_t}$ and $\widetilde{\epsilon}_{Y_t^{(\mathrm{obs})}}$ components of $E_t$.}{0in}\rightSide

As we can see from Figure \ref{appendix:fig:dynamic_with_exogenous}, $(A_t, E_t)$ are root causes of all other variables $(H_t, X_{t, i}, Y_t^{(\mathrm{ori})}, Y_t^{(\mathrm{obs})}, D_t)$.
Applying Definition \ref{appendix:def:FCM}, we can utilize the functional causal model and represent each variable with a function (the structural equation) of its direct causes (including observed parents and unobserved exogenous terms).
Then, we can iteratively replace variables with its corresponding structural equation and eventually represent variables in $(H_t, X_{t, i}, Y_t^{(\mathrm{ori})}, Y_t^{(\mathrm{obs})}, D_t)$ with functions of \textit{only} root causes $(A_t, E_t)$, as summarized in Fact \ref{fact:var_form}.

The noise terms $E_t$ are the unobserved exogenous terms that signify the unique characteristics of an individual.
The utilization of such uniqueness of individual can be found in the estimation of counterfactual causal effect by making use of the posterior distribution of exogenous noise terms conditioning on the observed features, e.g., \citet{kusner2017counterfactual}.
\color{black}

\subsection{Detailed Derivation of $\STIR$ (Section \ref{sec:introduce_STIR})}\label{appendix:derive_STIR}
In this section, we provide the derivation detail of the \textit{Single-step Tier Imbalance Reduction} (STIR) term:
\begin{equation}\label{appendix:equ:def_STIR}
    \begin{split}
        \STIR
        \coloneqq
        \Ebb \left[~\lvert
            f_{t + 1}(0, E_{t + 1}) - f_{t + 1}(1, E_{t + 1})
        \rvert~\right]
        -
        \Ebb \left[~\lvert
            f_{t}(0, E_{t}) - f_{t}(1, E_{t})
        \rvert~\right]
    \end{split}
\end{equation}
\textcolor{RevisionText}{
Firstly, in \Cref{appendix:calculate_qt}, we characterize the conditional joint density of $\big(f_T(0, E_T), f_T(1, E_T)\big)$.
Then, in \Cref{appendix:calculate_individual}, we focus on the situation changes of each individual from $T = t$ to $T = t + 1$ induced by the specified dynamics.
Finally, in \Cref{appendix:calculate_aggregate}, we can calculate the expectation in Equation \ref{appendix:equ:def_STIR} by aggregating situation changes for each individual from $T = t$ to $T = t + 1$.
}

\subsubsection{Characterizing Conditional Joint Density}\label{appendix:calculate_qt}
We can view $f_t(0, E_t)$ and $f_t(1, E_t)$ as two dependent random variables.
Given combinations of $D_t$ and $Y_t$, we can define their conditional joint probability density when $E_t = \epsilon$ as $q_t\big(f_t(0, \epsilon), f_t(1, \epsilon) \mid d, d', y, y' \big)$ and calculate it as following:
\begin{equation}\label{appendix:equ:joint_H_def}
    \begin{split}
        &~~~~~~q_t\big(f_t(0, \epsilon), f_t(1, \epsilon) \mid d, d', y, y' \big) \\
        &\coloneqq q_t\big(f_t(0, \epsilon), f_t(1, \epsilon)
        \mid g_t^D(0, \epsilon) = d,
        g_t^D(1, \epsilon) = d',
        g_t^{\Yori}(0, \epsilon) = y,
        g_t^{\Yori}(1, \epsilon) = y' \big) \\
        &= \int_{\xi \in \Epsilon}
            \mathbbm{1}\{ f_t(0, \xi) = f_t(0, \epsilon), f_t(1, \xi) = f_t(1, \epsilon) \} \\
        &~~~~~~~~~~~
        \cdot p_t \big( E_t = \xi \mid g_t^D(0, \epsilon) = d,
            g_t^D(1, \epsilon) = d',
            g_t^{\Yori}(0, \epsilon) = y,
            g_t^{\Yori}(1, \epsilon) = y' \big)
        d \xi,
    \end{split}
\end{equation}
where $\mathbbm{1}\{\cdot\}$ is the indicator function, and the subscript $t$ of the conditional probability densities (e.g., $q_t(\cdot)$ and $p_t(\cdot)$) indicates that they (might) change over time with different time step $T = t$.
The functional form of $f_t$ can be convoluted and it is not necessarily the case that $f_t(0, \cdot)$ and $f_t(1, \cdot)$ are injective mappings $\Epsilon \rightarrow (0, 1]$.
Therefore, for the purpose of generality, in Equation \ref{appendix:equ:joint_H_def} we explicitly introduce the identity function $\mathbbm{1}\{ f_t(0, \xi) = f_t(0, \epsilon), f_t(1, \xi) = f_t(1, \epsilon) \}$ when characterizing the conditional joint density $q_t$.

\subsubsection{Capturing Situation Changes for An Individual}\label{appendix:calculate_individual}
For a specific individual $(j)$, given the value of individual's exogenous terms $E_t^{(j)} = e_t^{(j)}$, let us denote the difference between $f_t(0, e_t^{(j)})$ and $f_t(1, e_t^{(j)})$ as $\varphi_t(e_t^{(j)}) \coloneqq f_t(0, e_t^{(j)}) - f_t(1, e_t^{(j)})$, and the sum of $f_t(0, e_t^{(j)})$ and $f_t(1, e_t^{(j)})$ as $\eta_t(e_t^{(j)}) \coloneqq f_t(0, e_t^{(j)}) + f_t(1, e_t^{(j)})$.
We introduce $\varphi_t(\cdot)$ and $\eta_t(\cdot)$ for the conciseness of notation, and we can always map $\big( \varphi_t(\cdot), \eta_t(\cdot) \big)$ back to $\big( f_t(0, \cdot), f_t(1, \cdot) \big)$ via a coordinate transformation:
\begin{equation}\label{appendix:equ:phi_eta_map_back}
    \begin{bmatrix}
        f_{t}(0, e_t^{(j)}) \\
        f_{t}(1, e_t^{(j)})
    \end{bmatrix} =
    \frac{\sqrt{2}}{2}
    \begin{bmatrix}
        \cos \frac{\pi}{4} & \sin \frac{\pi}{4} \\
        -\sin \frac{\pi}{4} & \cos \frac{\pi}{4}
    \end{bmatrix}
    \begin{bmatrix}
        \varphi_t(e_t^{(j)}) \\
        \eta_t(e_t^{(j)})
    \end{bmatrix}.
\end{equation}

Let us consider the connection between $\varphi_{t + 1}(e_{t + 1}^{(j)}) = f_{t + 1}(0, e_{t + 1}^{(j)}) - f_{t + 1}(1, e_{t + 1}^{(j)})$ in the time step $T = t + 1$ and $\varphi_t(e_t^{(j)}) = f_t(0, e_t^{(j)}) - f_t(1, e_t^{(j)})$ in the time step $T = t$.
We use different time step subscripts for the exogenous terms, e.g., $e_{t + 1}^{(j)}$ in $\varphi_{t + 1}(e_{t + 1}^{(j)})$ and $e_t^{(j)}$ in $\varphi_t(e_t^{(j)})$, since it is not necessarily the case that $e_{t + 1}^{(j)} = e_t^{(j)}$, even if we are focusing on the same individual from $T = t$ to $T = t + 1$.
Nevertheless, for the given functional forms of $f_t, g_t^D, g_t^{\Yori}$, the combination of $(d^{(j)}, d'^{(j)}, y^{(j)}, y'^{(j)})$, the value of exogenous term $e_t^{(j)}$ in the initial situation of the current one-step analysis (when $T = t$), and the hyperparameters $(\alpha_D, \alpha_Y)$, we can uniquely derive the value of $\varphi_{t + 1}(e_{t + 1}^{(j)}) = f_{t + 1}(0, e_{t + 1}^{(j)}) - f_{t + 1}(1, e_{t + 1}^{(j)})$, and list all possible instantiations of $\varphi_{t + 1}(e_{t + 1}^{(j)})$ in Table \ref{table:delta_next_Dmajor} (if $\alpha_D > \alpha_Y$), Table \ref{table:delta_next_Ymajor} (if $\alpha_D < \alpha_Y$), and Table \ref{table:delta_next_nomajor} (if $\alpha_D = \alpha_Y$).

Let us denote such mapping from $\varphi_{t}(e_{t}^{(j)})$ to $\varphi_{t + 1}(e_{t + 1}^{(j)})$ with the function $G_t$.
For the purpose of simplifying notations, we can omit the superscript $(j)$ if without ambiguity, since the value of exogenous terms $e_T$ signify the unique characteristics of an individual:
\begin{equation}\label{appendix:equ:mapping_across_t}
    \begin{split}
        \varphi_{t + 1}(e_{t + 1})
        \coloneqq f_{t + 1}(0, e_{t + 1}) - f_{t + 1}(1, e_{t + 1})
        = G_t(f_t, g_t^D, g_t^{\Yori}; d, d', y, y', e_t, \alpha_D, \alpha_Y).
    \end{split}
\end{equation}
Notice that the value of the function $G_t$ \textit{only} relies on the information available at time step $T = t$.

\subsubsection{Aggregating Individual-Level Situation Changes}\label{appendix:calculate_aggregate}
\color{RevisionText}
We can calculate \textit{Single-step Tier Imbalance Reduction}, i.e., the term $\STIR$, as following:
\reBlue{Q11}{Following your suggestion, we update the expression of the STIR term $\STIR$. We also update the detailed derivation in the appendix.}{0in}\rightSide
\begin{equation}\label{appendix:equ:calculate_STIR}
    \begin{split}
        & \STIR
        \coloneqq
        \Ebb \left[~\lvert
            f_{t + 1}(0, E_{t + 1}) - f_{t + 1}(1, E_{t + 1})
        \rvert~\right]
        -
        \Ebb \left[~\lvert
            f_{t}(0, E_{t}) - f_{t}(1, E_{t})
        \rvert~\right] \\
        & \overset{(i)}{=}
        \Ebb \left[~
            \lvert \varphi_{t + 1}(E_{t + 1}) \rvert
            - \lvert \varphi_{t}(E_{t}) \rvert
        ~\right] \\
        & \overset{(ii)}{=}
        \Ebb \Big\{
            \Ebb \Big[~
                \lvert \varphi_{t + 1}(\xi) \rvert
                - \lvert \varphi_{t}(\epsilon) \rvert
            ~~ \Big\vert
                \underbrace{E_{t + 1} = \xi,}_{
                \text{The value of exogenous terms of an individual take value $\xi$ at $T = t + 1$}.} \\
        & ~~~~~~~~~~~~~~~~~~~~~~~~~~~~~~~~~~~~~~~~~~~~~~~~~~~~~
                \underbrace{E_{t} = \epsilon,}_{
                \text{The value of exogenous terms of an individual take value $\epsilon$ at $T = t$.}} \\
        & ~~~~~~~~~~~~~~~~~~~~~~~~~~~~~~~~~~~~~~~~~~~~~~~~
                \underbrace{\varphi_{t + 1}(\xi) = G_t(f_t, g_t^D, g_t^{\Yori}; d, d', y, y', \epsilon, \alpha_D, \alpha_Y)}_{\substack{
                    \text{This is to make sure that we are keeping track of the same individual in the sense that,}
                    \\
                    \text{$\varphi_{t + 1}(\xi)$ when $T = t + 1$ is indeed a valid instantiation from $\varphi_{t}(\epsilon)$ when $T = t$.}
                    \\
                    \text{If $\varphi_{t + 1}(\cdot)$ is not a valid instantiation from $\varphi_t(\cdot)$, the contribution to the expectation is $0$.}
                }}
            ~\Big] \Big\} \\
        & \overset{(iii)}{=} \sum_{d, d', y, y' \in \{0, 1\}}
        P_t (d, d', y, y')
        \cdot
        \int_{\epsilon \in \Epsilon}
        \int_{\xi \in \Epsilon} q_t \big(f_t(0, \epsilon), f_t(1, \epsilon) \mid d, d', y, y' \big) \\
        & ~~~~~~~~
        \cdot \big( \lvert \varphi_{t + 1}(\xi) \rvert - \lvert \varphi_t(\epsilon) \rvert \big)
        \cdot \mathbbm{1}\{ \varphi_{t + 1}(\xi) = G_t(f_t, g_t^D, g_t^{\Yori}; d, d', y, y', \epsilon, \alpha_D, \alpha_Y) \}
        d \xi d \epsilon,
    \end{split}
\end{equation}
where the equality (i) is based on the definition of $\varphi_t(\cdot)$ and $\varphi_{t + 1}(\cdot)$;
the equality (ii) is derived from the Law of Iterated Expectation, keeping track of individual-level situation changes in the inner conditional expectation;
the equality (iii) is the aggregation of individual-level situation changes by plugging in the conditional joint density $q_t$ calculated in \Cref{appendix:calculate_qt}, joint probability $P_t$, and the individual-level situation changes discussed in \Cref{appendix:calculate_individual}.

\begin{figure}[t]
    \centering
    \captionsetup{format=hang}
    \includegraphics[width=1.\textwidth]{
        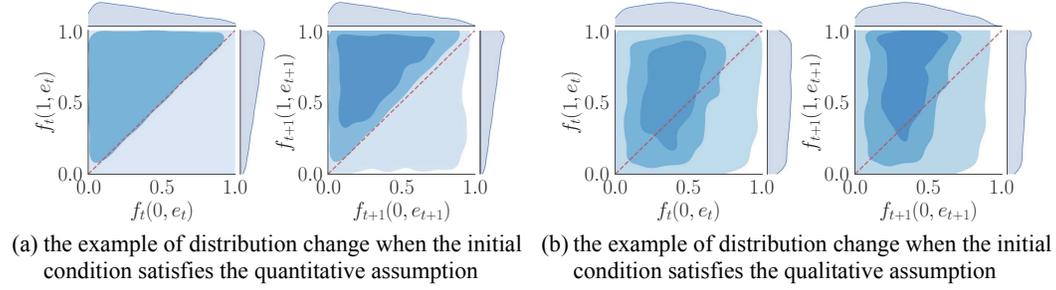}
    \caption{An illustration of the connection between qualitative and quantitative assumptions in terms of the one-step update of the conditional joint distribution $q_T\big(f_T(0, e_T), f_T(1, e_T) \mid d, d', y, y' \big)$ (when $y < y'$, and from $T = t$ to $T = t + 1$).}
    \label{fig:sim_assumption_illustration}
\end{figure}

The indicator function $\mathbbm{1}\{ \varphi_{t + 1}(\xi) = G_t(f_t, g_t^D, g_t^{\Yori}; d, d', y, y', \epsilon, \alpha_D, \alpha_Y) \}$ makes sure that we are keeping track of the same individual
\reBlue{Q12}{We use the indicator function to explicitly state that we keep track of the same individual before and after one-step update.}{0in}\leftSide
(whose exogenous noise term equals to $\epsilon$ at time $t$) before and after the one-step dynamic, even if his/her exogenous noise term equals to $\xi$ at time $t + 1$, and that $\epsilon$ might not be equal to $\xi$.

\color{black}  

The fact that we are keeping track of the same individual also justifies the practice of only integrating over (conditional) densities with subscript $t$, e.g., $q_t(\cdot)$ and $P_t(\cdot)$, instead of both $t$ and $t + 1$.
To see this from a different angle, keeping track of situation changes of each individual (when comparing $\varphi_{t + 1}(e_{t + 1})$ with $\varphi_t(e_t)$) also alleviates us from the trouble of estimating (conditional) densities that involve future information.
At the time step $t$, we do not know the densities
$q_{t + 1}\big(f_{t + 1}(0, E_{t + 1}), f_{t + 1}(1, E_{t + 1}) \mid d, d', y, y' \big)$
and $P_{t + 1}(d, d', y, y')$
since they involve future information $D_{t + 1}$ and $Y_{t + 1}$ at the standpoint of time step $T = t$.

\subsection{Further Illustration on Assumption \ref{assumption:biased_weight_qualitative} and Assumption \ref{assumption:biased_weight_quantitative}}\label{appendix:illustration_assumption}
In this subsection, we provide further illustrations of the connection between Assumption \ref{assumption:biased_weight_qualitative} and Assumption \ref{assumption:biased_weight_quantitative}.
In Figure \ref{fig:sim_assumption_illustration} we present the one-step update of the conditional joint distribution $q_T\big(f_T(0, e_T), f_T(1, e_T) \mid d, d', y, y' \big)$ from $T = t$ to $T = t + 1$ (we present the case when $y < y'$ as an example).
For panel (a) and (b), the joint distribution of $( f_T(0, E_T), f_T(1, E_T))$ is plotted before and after one-step dynamics, with quantitative and qualitative assumptions respectively.
The distributions are color-coded, the deeper the color, the larger the value of the joint density.

Compared to the qualitative assumption (Assumption \ref{assumption:biased_weight_qualitative}, illustrated in Figure \ref{fig:sim_assumption_illustration}b), the quantitative assumption (Assumption \ref{assumption:biased_weight_quantitative}, illustrated in Figure \ref{fig:sim_assumption_illustration}a) is just a special case, with quantitative characteristics built-in for technical purposes (we will make use of Assumption \ref{assumption:biased_weight_quantitative} in the proofs for Theorem \ref{theorem:perfect_one_step} and Theorem \ref{theorem:CF_one_step} in \Cref{appendix:proofs}).
From the illustrations in Figure \ref{fig:sim_assumption_illustration}, we can also see that the behaviors of the one-step update of conditional joint density under qualitative and quantitative assumptions are similar, with deeper color patterns occurring on the upper-left corner, indicating similar changes in the corresponding conditional joint densities $q_T\big(f_T(0, e_T), f_T(1, e_T) \mid d, d', y, y' \big)$ (when $y < y'$, and from $T = t$ to $T = t + 1$).

\color{RevisionText}
\subsection{Additional Experimental Results}\label{appendix:additional_exps}
In this section, we present additional experimental results 
\reOrange{C6}{Following your suggestion, we add additional empirical results when certain assumptions are violated.}{0in}\rightSide
on the preprocessed FICO credit score data set \citep{board2007report,hardt2016equality}.
Similar to the experiment summarized in Figure \ref{fig:long_term_creditscore_CF}, we convert the cumulative distribution function (CDF) of group-wise TransRisk scores into group-wise density distributions of the credit score, and use them as the initial tier distributions for different groups.

We consider utility-maximizing decision-making strategies, i.e., the decision-making policy is accuracy oriented and there is no explicit fairness consideration.
In Figure \ref{fig:long_term_creditscore_accuracy_only} we present the summary of a 20-step interplay between decision with accuracy-oriented predictors and the underlying data generating process on the credit score data set.
The accuracy-oriented decision-making strategy is retrained after each one-step data dynamics.
From Figure \ref{fig:long_term_creditscore_accuracy_only}(a), there is no obvious evidence that the gap between step-by-step tracks of tiers for different groups is decreasing over time.
This observation aligns with our theoretical analysis (Theorem \ref{theorem:perfect_one_step}) and simulation results (Figure \ref{fig:long_term_sim_perfect}) for perfect predictors.

\begin{figure}[t]
    \centering
    \captionsetup{format=hang}
    \includegraphics[width=1.\textwidth]{
        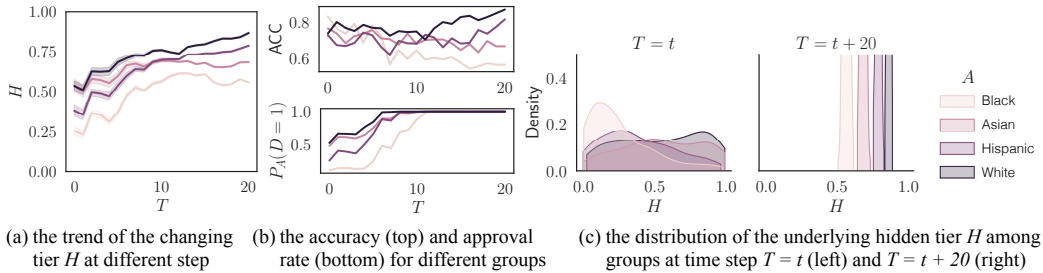}
    \caption{Illustration of the interplay between decision with accuracy-oriented predictors and the data dynamics (20 steps) on the credit score data set.
    Panel (a) and (b) present the step-by-step tracks of update in tier, accuracy, and approval rates for different groups; panel (c) presents group-conditioned distributions of tier before (left) and after (right) 20 steps of interventions.
    The legend is shared across panel (a), (b), and (c).}
    \label{fig:long_term_creditscore_accuracy_only}
\end{figure}

\color{black}
\subsection{Potential Limitations of Our Work}\label{appendix:limitations}
In this subsection, we discuss potential limitations of our work.

\subsubsection{Specified Dynamics vs. Causal Discovery}
In this paper, we present a detailed causal modeling of decision-distribution interplay on DAG (\Cref{sec:causal_DAG}) and formulate the dynamic fairness notion, \textit{Tier Balancing}, that captures the long-term fairness goal over the underlying causal factor.

The research of causal discovery, where the goal is to discover the causal relations among variables \citep{spirtes1993causation,shimizu2006linear,zhang2009identifiability,zhang2011kernel}, is a highly relevant area but is out of the scope of our paper.
Our \textit{Tier Balancing} notion of dynamic fairness, as well as our analyzing framework, does not rely on a causal model derived from causal discovery.
As we discussed in the comparison with previous literature in dynamic fairness studies (\Cref{appendix:discussion_related_works}), our causal model is richer and more complex, which provides the potential of a more principled reasoning of the essential decision-distribution interplay in the pursuit of long-term fairness.
We acknowledge the fact that it is nice to have the ability to discover the underlying causal model of the involved dynamics, which would provide further refinements of our dynamic modeling based on the specific practical scenario of interest.
Causal discovery can act as the icing on the cake, but not a necessary component, of our analysis.

\subsubsection{The Number and Dimension of Latent Causal Factors}
In the causal modeling of decision-distribution interplay we present in the paper, we consider one latent causal factor that carries on the influence of current decision to future distributions.
Recent developments in the identification of causal structures that involve (more than one) latent factors \citep{xie2020generalized,adams2021identification,kivva2021learning,xie2022identification} provide not only a theoretical justification, but also an indication of the potential, for our effort in exploring long-term fairness inquires over latent causal factors.
We believe that our detailed causal modeling of decision-distribution interplay (on both observed variables and latent causal factors) and our formulation of \textit{Tier Balancing} notion of long-term fairness act as an important first step.

\section{Proof of Results}\label{appendix:proofs}
In this section, we provide proofs for results presented in the paper.
For better readability, we provide an additional \textit{Proof (sketch)} before proving Theorem \ref{theorem:one_step_attainability} (proof in \Cref{appendix:proof_thm_one_step}), Theorem \ref{theorem:perfect_one_step} (proof in \Cref{appendix:proof_thm_perfect}), and Theorem \ref{theorem:CF_one_step} (proof in \Cref{appendix:proof_thm_CF}), respectively.


\subsection{Proof for Proposition \ref{proposition:Y_obs_bernoulli}}\label{appendix:proof_prop_Y_obs}
\begin{prop}
    At time step $T = t$, for any $H_t = h_t \in (0, 1]$, under the specified dynamics, among the population where ground truth is actually observable, i.e., $\Yobs_t$ is not undefined, we have:
    \begin{equation*}
        \Yobs_t \sim \mathrm{Bernoulli}(h_t).
    \end{equation*}
\end{prop}
\begin{proof}
    To begin with, according the d-separation relation among $D_{t - 1}$, $H_t$, and $\Yori_t$ on Figure \ref{fig:dynamic_model}, we notice that $\Yori_t \indep D_{t - 1} \mid H_t$.
    Therefore we have:
    \begin{equation*}
        \begin{split}
            &\Yori_t \sim \mathrm{Bernoulli}(h_t), \\
            &P(\Yori_t = 1 \mid H_t = h_t) = h_t, \\
            &P(D_{t - 1} = 1 \mid H_t = h_t) = d(h_t),
        \end{split}
    \end{equation*}
    where $d(\cdot)$ is a function $d:(0, 1] \rightarrow [0, 1]$.

    Notice that there is no claim that $D_{t - 1}$ can be uniquely determined by a function of only $h_t$.
    We only represent the conditional probability mass $P(D_{t - 1} = 1 \mid H_t = h_t)$ with a function of $h_t$ without specifying the exact functional form.
    In fact, as we shall see in the later part of this proof, the exact functional form of $d(\cdot)$ does not affect the validity of the result.

    Since $\Yobs_t$ is in fact $\Yori_t$ masked by $D_{t - 1}$, i.e., $\Yobs_t$ is observable only when $D_{t - 1} = 1$ and is undefined when $D_{t - 1} = 0$, we have:
    \begin{equation*}
        P(D_{t - 1} = 0, \Yobs_t = 0 \mid H_t = h_t) = P(D_{t - 1} = 0, \Yobs_t = 1 \mid H_t = h_t) = 0.
    \end{equation*}
    This indicates that among the population where $\Yobs_t$ is not undefined (the population itself may change at different time step), $\forall y \in \{0, 1\}$:
    \begin{equation*}
        \begin{split}
            & ~~~~~ P(\Yobs_t = y \mid H_t = h_t) \\
            & = P(D_{t - 1} = 0, \Yobs_t = y \mid H_t = h_t) + P(D_{t - 1} = 1, \Yobs_t = y \mid H_t = h_t) \\
            & = P(D_{t - 1} = 1, \Yobs_t = y \mid H_t = h_t) \\
            & = P(D_{t - 1} = 1, \Yori_t = y \mid H_t = h_t).
        \end{split}
    \end{equation*}
    Then, when $d(h_t) \in (0, 1)$, we can calculate the following probability:
    \begin{equation*}
        \begin{split}
            & P(\Yobs_t = 1 \mid H_t = h_t) \\
            = & \frac{P(\Yobs_t = 1 \mid H_t = h_t)}{
                P(\Yobs_t = 1 \mid H_t = h_t)
                +
                P(\Yobs_t = 0 \mid H_t = h_t)} \\
            = & \frac{P(D_{t - 1} = 1, \Yori_t = 1 \mid H_t = h_t)}{
            P(D_{t - 1} = 1, \Yori_t = 1 \mid H_t = h_t)
            +
            P(D_{t - 1} = 1, \Yori_t = 0 \mid H_t = h_t)} \\
            = & \frac{d(h_t) h_t}{
            d(h_t) h_t + d(h_t) (1 - h_t)} \\
            = & h_t \\
            = & P(\Yori_t = 1 \mid H_t = h_t);
        \end{split}
    \end{equation*}
    when $d(h_t) = 1$, this indicates that if $H_t = h_t$, we know for sure that this individual received a positive decision in the previous time step (when $T = t - 1$), and we have $\Yori_t = \Yobs_t$ by definition;
    when $d(h_t) = 0$, this indicates that if $H_t = h_t$, we know for sure that this individual did not receive a positive decision in the previous time step (when $T = t - 1$), and in this case $\Yobs_t$ is undefined.

    Therefore, among the population where ground truth is actually observable, i.e., $\Yobs_t$ is not undefined, we have:
    \begin{equation*}
        \Yobs_t \sim \mathrm{Bernoulli}(h_t).
    \end{equation*}
\end{proof}

\subsection{Proof for Theorem \ref{theorem:one_step_attainability}}\label{appendix:proof_thm_one_step}
\begin{thm}
    Let us consider the general situation where both $D_t$ and $\Yori_t$ are dependent with $A_t$, i.e., $D_t \nindep A_t, \Yori_t \nindep A_t$.
    Then under Fact \ref{fact:var_form}, Assumption \ref{assumption:Hupdate}, and Assumption \ref{assumption:samegroup}, as well as the specified dynamics, when $H_t \nindep A_t$, only if at least one of the following conditions holds true for all $e_t \in \Epsilon$ can we possibly attain $H_{t + 1} \indep A_{t + 1}$:
    \begin{enumerate}[label=(\arabic*)]
        \item The ratio $\frac{f_t(0, e_t)}{f_t(1, e_t)}$ has a specific domain of value:
            \begin{equation*}
                \frac{f_t(0, e_t)}{f_t(1, e_t)}
                = \frac{1 \pm \alpha_D \pm \alpha_Y}{1 \pm \alpha_D \pm \alpha_Y};
            \end{equation*}
        \item Positive (negative) labels only appear in the advantaged (disadvantaged) group, and the decision for everyone is positive (if $\alpha_D > \alpha_Y$):
            \begin{equation*}
                \begin{cases}
                    & f_t(0, e_t) \in [\frac{1}{1 + \alpha_D - \alpha_Y}, 1], \\
                    & f_t(1, e_t) \in [\frac{1}{1 + \alpha_D + \alpha_Y}, 1], \\
                    & g^{\Yori}_t(0, e_t) = 0, g^{\Yori}_t(1, e_t) = 1, \\
                    & g^D_t(0, e_t) = g^D_t(1, e_t) = 1;
                \end{cases}
            \end{equation*}
        \item Negative (positive) labels only appear in the advantaged (disadvantaged) group, and the decision for everyone is positive (if $\alpha_D > \alpha_Y$):
            \begin{equation*}
                \begin{cases}
                    & f_t(0, e_t) \in [\frac{1}{1 + \alpha_D + \alpha_Y}, 1], \\
                    & f_t(1, e_t) \in [\frac{1}{1 + \alpha_D - \alpha_Y}, 1], \\
                    & g^{\Yori}_t(0, e_t) = 1, g^{\Yori}_t(1, e_t) = 0, \\
                    & g^D_t(0, e_t) = g^D_t(1, e_t) = 1;
                \end{cases}
            \end{equation*}
        \item Everyone has a positive label, but the positive decision is exclusive to the advantaged group (if $\alpha_D < \alpha_Y$):
            \begin{equation*}
                \begin{cases}
                    & f_t(0, e_t) \in [\frac{1}{1 - \alpha_D + \alpha_Y}, 1], \\
                    & f_t(1, e_t) \in [\frac{1}{1 + \alpha_D + \alpha_Y}, 1], \\
                    & g^{\Yori}_t(0, e_t) = g^{\Yori}_t(1, e_t) = 1, \\
                    & g^D_t(0, e_t) = 0, g^D_t(1, e_t) = 1;
                \end{cases}
            \end{equation*}
        \item Everyone has a positive label, but the positive decision is exclusive to the disadvantaged group (if $\alpha_D < \alpha_Y$):
            \begin{equation*}
                \begin{cases}
                    & f_t(0, e_t) \in [\frac{1}{1 + \alpha_D + \alpha_Y}, 1], \\
                    & f_t(1, e_t) \in [\frac{1}{1 - \alpha_D + \alpha_Y}, 1], \\
                    & g^{\Yori}_t(0, e_t) = g^{\Yori}_t(1, e_t) = 1, \\
                    & g^D_t(0, e_t) = 1, g^D_t(1, e_t) = 0.
                \end{cases}
            \end{equation*}
    \end{enumerate}
\end{thm}
\begin{proof}[Proof (sketch)]
    In order to see the exact condition under which it is possible to achieve $H_{t + 1} \indep A_{t + 1}$, we consider the necessary and sufficient condition such that $H_{t + 1} = f_{t + 1}(A_{t + 1}, E_{t + 1})$ is not a function of $A_{t + 1}$.
    This, together with Fact \ref{fact:var_form}, Assumption \ref{assumption:Hupdate}, and Assumption \ref{assumption:samegroup}, indicates that we need to consider the condition under which $H_{t + 1} = \min \big\{ 1, f_t(A_t, E_t) \big[ 1 + \alpha_D (2 D_t - 1) + \alpha_Y (2 \Yori_t - 1) \big] \big\}$ is not a function of $A_t$.

    Since both $D_t$ and $\Yori_t$ are binary, we can exhaustively consider all value combinations of $D_t$ and $\Yori_t$, and list every possible value $H_{t + 1}$ can take in each case in Table \ref{table:H_next_Dmajor} (if $\alpha_D > \alpha_Y$), Table \ref{table:H_next_Ymajor} (if $\alpha_D < \alpha_Y$), or Table \ref{table:H_next_nomajor} (if $\alpha_D = \alpha_Y$).
    By exhaustively going through possible cases, we can have a full picture of the update of $H_{t + 1}$ based on $(H_t, \Yori_t, D_t)$, and then derive conditions under which $H_{t + 1}$ is \textit{not} a function of $A_t$, i.e., we have the conditions under which it is possible to attain $H_{t + 1} \indep A_{t + 1}$.
\end{proof}
\begin{proof}[Proof (full)]
    In order to see the exact condition under which it is possible to achieve $H_{t + 1} \indep A_{t + 1}$, we consider the necessary and sufficient condition such that $H_{t + 1} = f_{t + 1}(A_{t + 1}, E_{t + 1})$ is not a function of $A_{t + 1}$.
    By Fact \ref{fact:var_form}, Assumption \ref{assumption:Hupdate}, and Assumption \ref{assumption:samegroup}, it is necessary and sufficient to consider the condition under which $H_{t + 1} = \min \big\{ 1, f_t(A_t, E_t) \big[ 1 + \alpha_D (2 D_t - 1) + \alpha_Y (2 \Yori_t - 1) \big] \big\}$ is not a function of $A_t$.

    Considering the fact that both $D_t$ and $\Yori_t$ are binary, we can compare the values of $H_{t + 1}$ when $A_{t} = 0$ and $A_{t} = 1$ for all possible combinations of $D_t$ and $\Yori_t$.
    For any fixed $e_t \in \Epsilon$, we can list all the cases in Table \ref{table:H_next_Dmajor} (if $\alpha_D > \alpha_Y$), Table \ref{table:H_next_Ymajor} (if $\alpha_D < \alpha_Y$), or Table \ref{table:H_next_nomajor} (if $\alpha_D = \alpha_Y$), and see if for all $e_t \in \Epsilon$, there is no difference in the value of $H_{t + 1}$ between the cases when $A_t = 0$ and $A_t = 1$.

    From Table \ref{table:H_next_Dmajor}, Table \ref{table:H_next_Ymajor}, and Table \ref{table:H_next_nomajor}, we can see that if and only the following hold true can we achieve $H_{t + 1} \indep A_{t + 1}$: for every $e_t \in \Epsilon$, whenever the joint probability $P\big( g^D_t(0, e_t) = d, g^{\Yori}_t(0, e_t) = y, g^D_t(1, e_t) = d', g^{\Yori}_t(1, e_t) = y' \big)$ is nonzero, the last two columns of the corresponding row(s) in the table, i.e., the exact values of $H_{t + 1}$, need to match.
    For example, when $\alpha_D > \alpha_Y$, if we know $P\big( g^D_t(0, e_t) = 0, g^{\Yori}_t(0, e_t) = 0, g^D_t(1, e_t) = 0, g^{\Yori}_t(1, e_t) = 0 \big) \neq 0$, we need the last two columns of Case (i) of Table \ref{table:H_next_Dmajor} to equal to each other, i.e., we need $f_t(0, e_t) = f_t(1, e_t)$ to hold true.

    Without further assumptions on the joint distribution of the data, we do not know which combination of $(d, y, d', y')$ will result in a nonzero joint probability:
    \begin{equation*}
        P\big( g^D_t(0, e_t) = d, g^{\Yori}_t(0, e_t) = y, g^D_t(1, e_t) = d', g^{\Yori}_t(1, e_t) = y' \big) \neq 0.
    \end{equation*}
    However, considering the fact that $\sum_{d, d' \in \Dcal, y, y' \in \Ycal} P\big( g^D_t(0, e_t) = d, g^{\Yori}_t(0, e_t) = y, g^D_t(1, e_t) = d', g^{\Yori}_t(1, e_t) = y' \big) = 1$ holds for all $e_t \in \Epsilon$, we do know that for any fixed $e_t \in \Epsilon$, there is at least one possible instantiation of $(d^*, y^*, d'^*, y'^*)$ such that:
    \begin{equation}
            P\big( g^D_t(0, e_t) = d^*, g^{\Yori}_t(0, e_t) = y^*,
            g^D_t(1, e_t) = d'^*, g^{\Yori}_t(1, e_t) = y'^* \big) \neq 0.
    \end{equation}

    Let us first consider situations where $\alpha_D > \alpha_Y$ and focus on Table \ref{table:H_next_Dmajor}.
    The analysis on situations where $\alpha_D < \alpha_Y$ (i.e., Table \ref{table:H_next_Ymajor}) or $\alpha_D = \alpha_Y$ (i.e, Table \ref{table:H_next_nomajor}), is of the same flavor and therefore we omit the detail in the proof.

    To begin with, we can observe that not every entry of the last two columns explicitly keeps the $\min\{\cdot, 1\}$ operator.
    On the one hand, since $\alpha_D > \alpha_Y$ ($\alpha_D, \alpha_Y \in [0, \frac{1}{2})$, as of Assumption \ref{assumption:Hupdate}), we have $(1 - \alpha_D \pm \alpha_Y) \in (0, 1)$ and $f_t(a_t, e_t) (1 - \alpha_D \pm \alpha_Y) \in (0, 1)$ (since $H_t = f_t(a_t, e_t) \in (0, 1]$); therefore, we do not need to keep the $\min\{\cdot, 1\}$ operator explicit, for instance, in the second to last column of Case (v - viii).
    On the other hand, when the coefficients $(1 + \alpha_D \pm \alpha_Y) > 1$ we are not sure if $f_t(a_t, e_t) (1 + \alpha_D \pm \alpha_Y)$ exceed $1$; therefore, we need to keep the $\min\{\cdot, 1\}$ operator explicit, for instance, in the last column of Case (v - viii).

    Besides, if only one entry of the last two columns explicitly has the $\min\{\cdot, 1\}$ operator, it is equivalent to require that the terms themselves (before applying the operator) are equal (since the one without the $\min\{\cdot, 1\}$ operator is known to be within the $(0, 1)$ interval).
    For instance, Case (ix) requires that $\min\{ f_t(0, e_t)(1 + \alpha_D - \alpha_Y), 1 \} = f_t(1, e_t)(1 - \alpha_D - \alpha_Y)$, which is equivalent to requiring $f_t(0, e_t)(1 + \alpha_D - \alpha_Y) = f_t(1, e_t)(1 - \alpha_D - \alpha_Y)$.

    Furthermore, if both entries of the last two columns explicitly has the $\min\{\cdot, 1\}$ operator, the exact condition of matching the last two columns depends on the actual value of $f_t(0, e_t)$ and $f_t(1, e_t)$.
    For instance, Case (xv) requires that $\min\{ f_t(0, e_t)(1 + \alpha_D + \alpha_Y), 1 \} = \min\{ f_t(1, e_t)(1 + \alpha_D - \alpha_Y), 1 \}$, which could be equivalent to one of the following conditions (recall that $1 + \alpha_D \pm \alpha_Y > 1$):
    \begin{itemize}
        \item if we have $f_t(0, e_t) \in [\frac{1}{1 + \alpha_D + \alpha_Y}, 1]$ and $f_t(1, e_t) \in [\frac{1}{1 + \alpha_D - \alpha_Y}, 1]$, we require $1 = 1$, which trivially holds true;
        \item if we have $f_t(0, e_t) \in (0, \frac{1}{1 + \alpha_D + \alpha_Y})$ and $f_t(1, e_t) \in [\frac{1}{1 + \alpha_D - \alpha_Y}, 1]$, we require $f_t(0, e_t)(1 + \alpha_D + \alpha_Y) = 1$, which cannot hold true;
        \item if we have $f_t(0, e_t) \in [\frac{1}{1 + \alpha_D + \alpha_Y}, 1]$ and $f_t(1, e_t) \in (0, \frac{1}{1 + \alpha_D - \alpha_Y})$, we require $1 = f_t(1, e_t)(1 + \alpha_D - \alpha_Y)$ which cannot hold true;
        \item if we have $f_t(0, e_t) \in (0, \frac{1}{1 + \alpha_D + \alpha_Y})$ and $f_t(1, e_t) \in (0, \frac{1}{1 + \alpha_D - \alpha_Y})$, we require $\frac{f_t(0, e_t)}{f_t(1, e_t)} = \frac{1 + \alpha_D - \alpha_Y}{1 + \alpha_D + \alpha_Y}$.
    \end{itemize}

    Recall that without further assumptions on the data distribution, we do not know which row(s) of the table correspond to a nonzero probability $P\big( g^D_t(0, e_t) = d, g^{\Yori}_t(0, e_t) = y, g^D_t(1, e_t) = d', g^{\Yori}_t(1, e_t) = y' \big)$.
    As a result, in general, we do not know which set of requirements we should enforce for each $e_t \in \Epsilon$.
    Therefore, we cannot derive a necessary and sufficient condition for attaining $H_{t + 1} \indep A_{t + 1}$ in general cases.
    Nevertheless, we can summarize the previous analysis and derive the necessary condition of attaining $H_{t + 1} \indep A_{t + 1}$, i.e., only if at least one of the following conditions holds true for all $e_t \in \Epsilon$ can we possibly attain $H_{t + 1} \indep A_{t + 1}$:
    \begin{enumerate}[label=(\arabic*)]
        \item The ratio $\frac{f_t(0, e_t)}{f_t(1, e_t)}$ has a specific domain of value:
            \begin{equation*}
                \frac{f_t(0, e_t)}{f_t(1, e_t)}
                = \frac{1 \pm \alpha_D \pm \alpha_Y}{1 \pm \alpha_D \pm \alpha_Y};
            \end{equation*}
        \item Positive (negative) labels only appear in the advantaged (disadvantaged) group, and the decision for everyone is positive (if $\alpha_D > \alpha_Y$):
            \begin{equation*}
                \begin{cases}
                    & f_t(0, e_t) \in [\frac{1}{1 + \alpha_D - \alpha_Y}, 1], \\
                    & f_t(1, e_t) \in [\frac{1}{1 + \alpha_D + \alpha_Y}, 1], \\
                    & g^{\Yori}_t(0, e_t) = 0, g^{\Yori}_t(1, e_t) = 1, \\
                    & g^D_t(0, e_t) = g^D_t(1, e_t) = 1;
                \end{cases}
            \end{equation*}
        \item Negative (positive) labels only appear in the advantaged (disadvantaged) group, and the decision for everyone is positive (if $\alpha_D > \alpha_Y$):
            \begin{equation*}
                \begin{cases}
                    & f_t(0, e_t) \in [\frac{1}{1 + \alpha_D + \alpha_Y}, 1], \\
                    & f_t(1, e_t) \in [\frac{1}{1 + \alpha_D - \alpha_Y}, 1], \\
                    & g^{\Yori}_t(0, e_t) = 1, g^{\Yori}_t(1, e_t) = 0, \\
                    & g^D_t(0, e_t) = g^D_t(1, e_t) = 1;
                \end{cases}
            \end{equation*}
        \item Everyone has a positive label, but the positive decision is exclusive to the advantaged group (if $\alpha_D < \alpha_Y$):
            \begin{equation*}
                \begin{cases}
                    & f_t(0, e_t) \in [\frac{1}{1 - \alpha_D + \alpha_Y}, 1], \\
                    & f_t(1, e_t) \in [\frac{1}{1 + \alpha_D + \alpha_Y}, 1], \\
                    & g^{\Yori}_t(0, e_t) = g^{\Yori}_t(1, e_t) = 1, \\
                    & g^D_t(0, e_t) = 0, g^D_t(1, e_t) = 1;
                \end{cases}
            \end{equation*}
        \item Everyone has a positive label, but the positive decision is exclusive to the disadvantaged group (if $\alpha_D < \alpha_Y$):
            \begin{equation*}
                \begin{cases}
                    & f_t(0, e_t) \in [\frac{1}{1 + \alpha_D + \alpha_Y}, 1], \\
                    & f_t(1, e_t) \in [\frac{1}{1 - \alpha_D + \alpha_Y}, 1], \\
                    & g^{\Yori}_t(0, e_t) = g^{\Yori}_t(1, e_t) = 1, \\
                    & g^D_t(0, e_t) = 1, g^D_t(1, e_t) = 0.
                \end{cases}
            \end{equation*}
    \end{enumerate}
\end{proof}

\subsection{Proof for Theorem \ref{theorem:perfect_one_step}}\label{appendix:proof_thm_perfect}
\begin{thm}
    Let us consider the general situation where both $D_t$ and $\Yori_t$ are dependent with $A_t$, i.e., $D_t \nindep A_t, \Yori_t \nindep A_t$.
    Under Fact \ref{fact:var_form}, Assumption \ref{assumption:Hupdate}, Assumption \ref{assumption:samegroup}, and Assumption \ref{assumption:biased_weight_quantitative}, as well as the specified dynamics, when $H_t \nindep A_t$, the perfect predictor does not have the potential to get closer to the long-term fairness goal after one-step intervention, i.e.,
    \begin{equation*}
        D_t = \Yori_t \implies \Delta^{(\text{Perfect Predictor})}_{\mathrm{STIR}} \big\rvert_{t}^{t + 1} > 0.
    \end{equation*}
\end{thm}
\begin{proof}[Proof (sketch)]
    The goal is to calculate if it is possible for \textit{Single-step Tier Imbalance Reduction} $\STIR$ to be smaller than $0$ when using perfect predictors.
    As defined in Equation \ref{equ:STIR}, $\STIR$ is a weighted aggregation (integration followed by summation) of $\lvert \varphi(e_{t + 1}) \rvert - \lvert \varphi(e_t)\rvert$.
    The quantitative analysis involves three key components: instantiations of $\varphi_{t + 1}(e_{t + 1})$, the knowledge/assumptions on $q_t\big(f_t(0, \epsilon), f_t(1, \epsilon) \mid d, d', y, y' \big)$, and characteristics of $P_t(d, d', y', y')$.

    For the first component, we can list all possible instantiations of $\varphi_{t + 1}(e_{t + 1})$ in Table \ref{table:delta_next_Dmajor} (if $\alpha_D > \alpha_Y$), Table \ref{table:delta_next_Ymajor} (if $\alpha_D < \alpha_Y$), and Table \ref{table:delta_next_nomajor} (if $\alpha_D = \alpha_Y$), respectively.
    For the second component, we can introduce a quantitative assumption on $q_t\big(f_t(0, \epsilon), f_t(1, \epsilon) \mid d, d', y, y' \big)$ (Assumption \ref{assumption:biased_weight_quantitative}).
    For the third component, we need to exploit the characteristic of the predictor of interest to gain further insight into the joint distribution $P_t(d, d', y, y')$.
    For perfect predictors, we have $P_t(d, d', y, y')$ satisfies Equation \ref{equ:perfect_predictor_joint} (as we have discussed in \Cref{sec:one_step_analysis_perfect}).

    For the purpose of calculating the value of $\STIR$, the proof contains two steps: (1) exhaustively derive the value of $\lvert \varphi(e_{t + 1}) \rvert - \lvert \varphi(e_t)\rvert$ after one-step dynamics in all possible cases, and (2) aggregate the difference $\lvert \varphi(e_{t + 1}) \rvert - \lvert \varphi(e_t)\rvert$ with the help of the additional knowledge/assumptions on $q_t\big(f_t(0, \epsilon), f_t(1, \epsilon) \mid d, d', y, y' \big)$ and $P_t(d, d', y, y')$.
\end{proof}
\begin{proof}[Proof (full)]
    For the perfect predictor $D_t = \Yori_t$, among all possible instantiations of $\varphi_{t + 1}(e_{t + 1})$ as listed in Table \ref{table:delta_next_Dmajor}, Table \ref{table:delta_next_Ymajor}, and Table \ref{table:delta_next_nomajor}, not every case corresponds to a nonzero $P_t(d, d', y, y')$ and therefore may not contribute to the computation of $\STIR$ as detailed in Equation \ref{equ:STIR}.
    By applying Equation \ref{equ:perfect_predictor_joint} we only need to consider Case (i), Case (vi), Case (xi), and Case (xvi) in Table \ref{table:delta_next_Dmajor} (if $\alpha_D > \alpha_Y$), Table \ref{table:delta_next_Ymajor} (if $\alpha_D < \alpha_Y$), and Table \ref{table:delta_next_nomajor} (if $\alpha_D = \alpha_Y$), respectively.
    We are interested in cases where at least one of $g^D_t(\cdot, E_t)$ and $g^{\Yori}_t(\cdot, E_t)$ are functions of $A_t$, and we can list all possible values of $\lvert \varphi(e_{t + 1}) \rvert - \lvert \varphi(e_t)\rvert$ for each of the aforementioned cases (the result applies to scenarios where $\alpha_D > \alpha_Y$, $\alpha_D < \alpha_Y$, or $\alpha_D = \alpha_Y$).

    When $(d, d', y, y') = (0, 1, 0, 1)$, i.e., for Case (vi):
    \begin{enumerate}[label=(vi.1.\arabic*)]
        \item $\lvert \varphi(e_{t + 1}) \rvert - \lvert \varphi(e_t)\rvert = - (\alpha_D + \alpha_Y) \big(f_t(0, e_t) + f_t(1, e_t)\big) < 0$
            \begin{equation*}
                \text{if we have }
                \begin{cases}
                    f_t(0, e_t) \in (0, 1], f_t(1, e_t) \in (0, \frac{1}{1 + \alpha_D + \alpha_Y}) & \\
                    f_t(1, e_t) \leq \tan \big( \arctan \frac{1}{\alpha_D + \alpha_Y} - \frac{\pi}{4} \big) \cdot f_t(0, e_t) &
                \end{cases};
            \end{equation*}
        \item $\lvert \varphi(e_{t + 1}) \rvert - \lvert \varphi(e_t)\rvert = (\alpha_D + \alpha_Y - 2) f_t(0, e_t) + (\alpha_D + \alpha_Y + 2) f_t(1, e_t) < 0$
            \begin{equation*}
                \text{if we have }
                \begin{cases}
                    f_t(0, e_t) \in (0, 1], f_t(1, e_t) \in (0, \frac{1}{1 + \alpha_D + \alpha_Y}) & \\
                    f_t(1, e_t) < \tan \big( \arctan \frac{2}{\alpha_D + \alpha_Y} - \frac{\pi}{4} \big) \cdot f_t(0, e_t) & \\
                    f_t(1, e_t) \geq \tan \big( \arctan \frac{1}{\alpha_D + \alpha_Y} - \frac{\pi}{4} \big) \cdot f_t(0, e_t) &
                \end{cases};
            \end{equation*}
        \item $\lvert \varphi(e_{t + 1}) \rvert - \lvert \varphi(e_t)\rvert = (\alpha_D + \alpha_Y - 2) f_t(0, e_t) + (\alpha_D + \alpha_Y + 2) f_t(1, e_t) > 0$
            \begin{equation*}
                \text{if we have }
                \begin{cases}
                    f_t(0, e_t) \in (0, 1], f_t(1, e_t) \in (0, \frac{1}{1 + \alpha_D + \alpha_Y}) & \\
                    f_t(1, e_t) < f_t(0, e_t) & \\
                    f_t(1, e_t) \geq \tan \big( \arctan \frac{2}{\alpha_D + \alpha_Y} - \frac{\pi}{4} \big) \cdot f_t(0, e_t) &
                \end{cases};
            \end{equation*}
        \item $\lvert \varphi(e_{t + 1}) \rvert - \lvert \varphi(e_t)\rvert = (\alpha_D + \alpha_Y) \big(f_t(0, e_t) + f_t(1, e_t)\big) > 0$
            \begin{equation*}
                \text{if we have }
                \begin{cases}
                    f_t(0, e_t) \in (0, 1], f_t(1, e_t) \in (0, \frac{1}{1 + \alpha_D + \alpha_Y}) & \\
                    f_t(1, e_t) \geq f_t(0, e_t) &
                \end{cases};
            \end{equation*}
    \end{enumerate}
    \begin{enumerate}[label=(vi.2.\arabic*)]
        \item $\lvert \varphi(e_{t + 1}) \rvert - \lvert \varphi(e_t)\rvert = 1 - (2 - \alpha_D - \alpha_Y) f_t(0, e_t) + f_t(1, e_t) > 0$
            \begin{equation*}
                \text{if we have }
                \begin{cases}
                    f_t(0, e_t) \in (0, 1], f_t(1, e_t) \in [\frac{1}{1 + \alpha_D + \alpha_Y}, 1] & \\
                    f_t(1, e_t) < f_t(0, e_t) &
                \end{cases};
            \end{equation*}
        \item $\lvert \varphi(e_{t + 1}) \rvert - \lvert \varphi(e_t)\rvert = 1 + (\alpha_D + \alpha_Y) f_t(0, e_t) - f_t(1, e_t) > 0$
            \begin{equation*}
                \text{if we have }
                \begin{cases}
                    f_t(0, e_t) \in (0, 1], f_t(1, e_t) \in [\frac{1}{1 + \alpha_D + \alpha_Y}, 1] & \\
                    f_t(1, e_t) \geq f_t(0, e_t) &
                \end{cases}.
            \end{equation*}
    \end{enumerate}
    When $(d, d', y, y') = (1, 0, 1, 0)$, i.e., for Case (xi):
    \begin{enumerate}[label=(xi.1.\arabic*)]
        \item $\lvert \varphi(e_{t + 1}) \rvert - \lvert \varphi(e_t)\rvert = - (\alpha_D + \alpha_Y) \big(f_t(0, e_t) + f_t(1, e_t)\big) < 0$
            \begin{equation*}
                \text{if we have }
                \begin{cases}
                    f_t(0, e_t) \in (0, \frac{1}{1 + \alpha_D + \alpha_Y}), f_t(1, e_t) \in (0, 1] & \\
                    f_t(1, e_t) \geq \tan \big( \frac{3 \pi}{4} - \arctan \frac{1}{\alpha_D + \alpha_Y} \big) \cdot f_t(0, e_t) &
                \end{cases};
            \end{equation*}
        \item $\lvert \varphi(e_{t + 1}) \rvert - \lvert \varphi(e_t)\rvert = (\alpha_D + \alpha_Y + 2) f_t(0, e_t) + (\alpha_D + \alpha_Y - 2) f_t(1, e_t) < 0$
            \begin{equation*}
                \text{if we have }
                \begin{cases}
                    f_t(0, e_t) \in (0, \frac{1}{1 + \alpha_D + \alpha_Y}), f_t(1, e_t) \in (0, 1] & \\
                    f_t(1, e_t) < \tan \big( \frac{3 \pi}{4} - \arctan \frac{1}{\alpha_D + \alpha_Y} \big) \cdot f_t(0, e_t) & \\
                    f_t(1, e_t) \geq \tan \big( \frac{3 \pi}{4} - \arctan \frac{2}{\alpha_D + \alpha_Y} \big) \cdot f_t(0, e_t) &
                \end{cases};
            \end{equation*}
        \item $\lvert \varphi(e_{t + 1}) \rvert - \lvert \varphi(e_t)\rvert = (\alpha_D + \alpha_Y + 2) f_t(0, e_t) + (\alpha_D + \alpha_Y - 2) f_t(1, e_t) > 0$
            \begin{equation*}
                \text{if we have }
                \begin{cases}
                    f_t(0, e_t) \in (0, \frac{1}{1 + \alpha_D + \alpha_Y}), f_t(1, e_t) \in (0, 1] & \\
                    f_t(1, e_t) < \tan \big( \frac{3 \pi}{4} - \arctan \frac{2}{\alpha_D + \alpha_Y} \big) \cdot f_t(0, e_t) & \\
                    f_t(1, e_t) \geq f_t(0, e_t) &
                \end{cases};
            \end{equation*}
        \item $\lvert \varphi(e_{t + 1}) \rvert - \lvert \varphi(e_t)\rvert = (\alpha_D + \alpha_Y) \big(f_t(0, e_t) + f_t(1, e_t)\big) > 0$
            \begin{equation*}
                \text{if we have }
                \begin{cases}
                    f_t(0, e_t) \in (0, \frac{1}{1 + \alpha_D + \alpha_Y}), f_t(1, e_t) \in (0, 1] & \\
                    f_t(1, e_t) < f_t(0, e_t) &
                \end{cases};
            \end{equation*}
    \end{enumerate}

    \begin{figure}[t]
        \centering
        \captionsetup{format=hang}
        \begin{subfigure}[t]{.46\textwidth}
            \centering
            \includegraphics[width=1.\columnwidth]{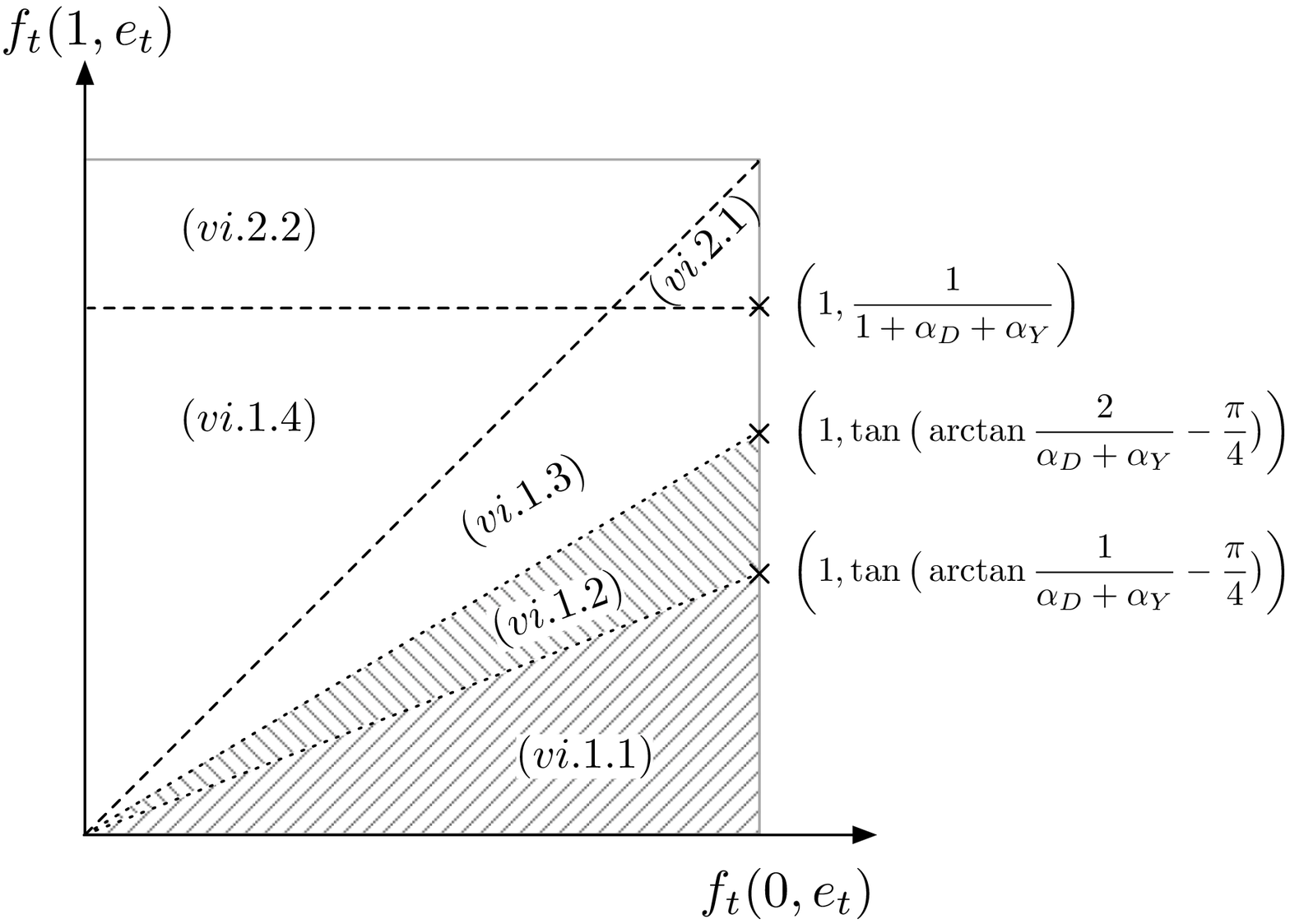}
            \caption{Case (vi), $(d, d', y, y') = (0, 1, 0, 1)$}
            \label{fig:appendix_case_vi_slice}
        \end{subfigure}
        \begin{subfigure}[t]{.5\textwidth}
            \centering
            \includegraphics[width=.67\columnwidth]{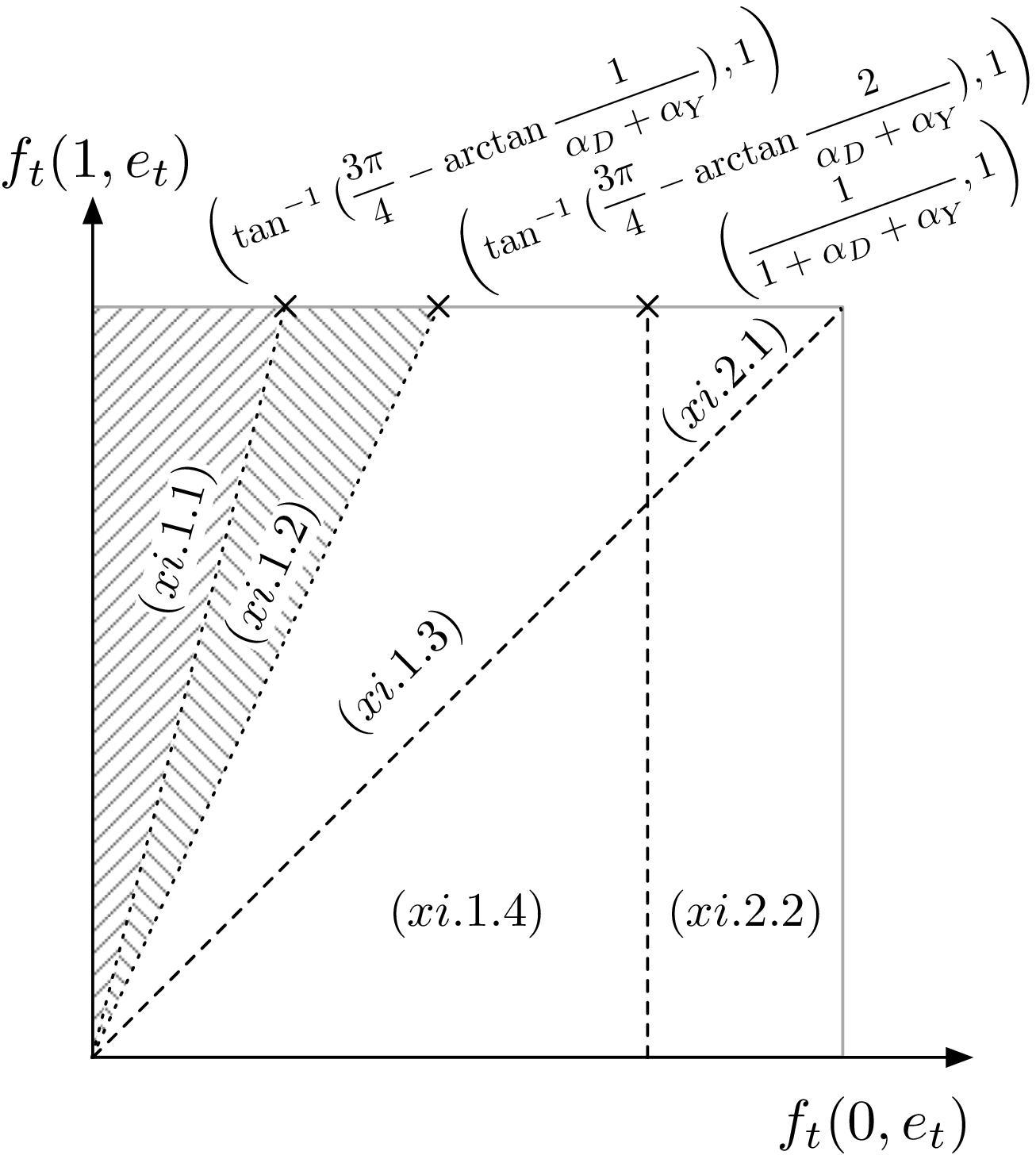}
            \caption{Case (xi), $(d, d', y, y') = (1, 0, 1, 0)$}
            \label{fig:appendix_case_xi_slice}
        \end{subfigure}
        \caption{Illustration of the sliced squares on the $\big( f_t(0, e_t), f_t(1, e_t) \big)$ plane.
        Depending on the initial situation, i.e., the slice that the $\big( f_t(0, e_t), f_t(1, e_t) \big)$ pair falls upon, the term $\lvert \varphi(e_{t + 1}) \rvert - \lvert \varphi(e_t)\rvert$ takes different values.
        The shaded slices indicate that if the initial situation satisfies the corresponding condition, the calculated $\lvert \varphi(e_{t + 1}) \rvert - \lvert \varphi(e_t)\rvert < 0$.
        }
        \label{fig:square_slice}
    \end{figure}

    \begin{enumerate}[label=(xi.2.\arabic*)]
        \item $\lvert \varphi(e_{t + 1}) \rvert - \lvert \varphi(e_t)\rvert = 1 + f_t(0, e_t) - (2 - \alpha_D - \alpha_Y) f_t(1, e_t) > 0$
            \begin{equation*}
                \text{if we have }
                \begin{cases}
                    f_t(0, e_t) \in [\frac{1}{1 + \alpha_D + \alpha_Y}, 1], f_t(1, e_t) \in (0, 1] & \\
                    f_t(1, e_t) \geq f_t(0, e_t) &
                \end{cases};
            \end{equation*}
        \item $\lvert \varphi(e_{t + 1}) \rvert - \lvert \varphi(e_t)\rvert = 1 - f_t(0, e_t) + (\alpha_D + \alpha_Y) f_t(1, e_t) > 0$
            \begin{equation*}
                \text{if we have }
                \begin{cases}
                    f_t(0, e_t) \in [\frac{1}{1 + \alpha_D + \alpha_Y}, 1], f_t(1, e_t) \in (0, 1] & \\
                    f_t(1, e_t) < f_t(0, e_t) &
                \end{cases}.
            \end{equation*}
    \end{enumerate}

    Now we proceed to the second step and aggregate $\lvert \varphi(e_{t + 1}) \rvert - \lvert \varphi(e_t)\rvert$ terms.
    According to Equation \ref{equ:STIR} and Equation \ref{appendix:equ:mapping_across_t}, for the perfect predictor we have:
    \begin{equation}\label{equ:STIR_perfect}
        \begin{split}
            \Delta^{(\text{Perfect Predictor})}_{\mathrm{STIR}} \big\rvert_{t}^{t + 1}
            &= P_t(0, 1, 0, 1) \cdot
            \int_{\epsilon \in \Epsilon}
            \int_{\xi \in \Epsilon}
            \big( \lvert \varphi_{t + 1}(\xi) \rvert - \lvert \varphi_t(\epsilon) \rvert \big) \\
            & ~~~~~~~~~~
            \cdot \mathbbm{1}\{ \varphi_{t + 1}(\xi) = G(f_t, g_t^D, g_t^{\Yori}; 0, 1, 0, 1, \epsilon, \alpha_D, \alpha_Y) \} \\
            & ~~~~~~~~~~
            \cdot q_t\big(f_t(0, \epsilon), f_t(1, \epsilon) \mid 0, 1, 0, 1 \big)
            d \xi d \epsilon \\
            & ~~ + P_t(1, 0, 1, 0) \cdot
            \int_{\epsilon \in \Epsilon}
            \int_{\xi \in \Epsilon}
            \big( \lvert \varphi_{t + 1}(\xi) \rvert - \lvert \varphi_t(\epsilon) \rvert \big) \\
            & ~~~~~~~~~~
            \cdot \mathbbm{1}\{ \varphi_{t + 1}(\xi) = G(f_t, g_t^D, g_t^{\Yori}; 1, 0, 1, 0, \epsilon, \alpha_D, \alpha_Y) \} \\
            & ~~~~~~~~~~
            \cdot q_t\big(f_t(0, \epsilon), f_t(1, \epsilon) \mid 1, 0, 1, 0 \big)
            d \xi d \epsilon.
        \end{split}
    \end{equation}

    As we can see from Equation \ref{equ:STIR_perfect}, we need to perform two-dimensional integrations on the $\big(f_t(0, e_t), f_t(1, e_t)\big)$ plane, calculating the expectation of the term $\lvert \varphi(e_{t + 1}) \rvert - \lvert \varphi(e_t)\rvert$ over the conditional densities $q_t\big(f_t(0, \epsilon), f_t(1, \epsilon) \mid 0, 1, 0, 1 \big)$ and $q_t\big(f_t(0, \epsilon), f_t(1, \epsilon) \mid 1, 0, 1, 0 \big)$.
    Since these conditional joint densities could be convoluted in general cases, the calculation of conditional expectations in Equation \ref{equ:STIR_perfect} could be rather complicated.
    Therefore, we propose to take advantage of Assumption \ref{assumption:biased_weight_quantitative} to quantitatively simplify the calculation yet remain consistent with the rather mild qualitative assumption (Assumption \ref{assumption:biased_weight_qualitative}), and derive a result that is numerically clear and informative.
    For the purpose of better illustrating the connection between (qualitative and quantitative) assumptions on $q_t\big(f_t(0, \epsilon), f_t(1, \epsilon) \mid d, d', y, y' \big)$ and the computation of $\STIR$, we also provide illustrative figures as shown in Figure \ref{fig:square_slice}.

    With the help of Assumption \ref{assumption:biased_weight_quantitative}, we convert the conditional expectations in Equation \ref{equ:STIR_perfect} into calculations of multiple integrals on slices within a $1 \times 1$ square on the 2-D plane, where $\phi_0$ and $\phi_1$ axes correspond to the value of $f_t(0, E_t)$ and $f_t(1, E_t)$ respectively:

    \begin{equation*}
        \begin{split}
            & ~~ \int_{\epsilon \in \Epsilon}
            \int_{\xi \in \Epsilon}
            \big( \lvert \varphi_{t + 1}(\xi) \rvert - \lvert \varphi_t(\epsilon) \rvert \big)
            \cdot \mathbbm{1}\{ \varphi_{t + 1}(\xi) = G(f_t, g_t^D, g_t^{\Yori}; 0, 1, 0, 1, \epsilon, \alpha_D, \alpha_Y) \} \\
            & ~~~~~~~~~~
            \cdot q_t\big(f_t(0, \epsilon), f_t(1, \epsilon) \mid 0, 1, 0, 1 \big) d \xi d \epsilon \\
            & =
            \gamma_{0101}^{\text{(low)}} \cdot \bigg\{
                \int_0^1
                \int_0^{\tan \big( \arctan \frac{1}{\alpha_D + \alpha_Y} - \frac{\pi}{4} \big) \phi_0}
                - (\alpha_D + \alpha_Y) (\phi_0 + \phi_1)
                ~ d \phi_1 d \phi_0 \\
            & ~~~~~~~~~~
                + \int_0^{\frac{1}{1 + \alpha_D + \alpha_Y}}
                \int_{\tan \big( \arctan \frac{1}{\alpha_D + \alpha_Y} - \frac{\pi}{4} \big) \phi_0}^{\phi_0}
                (\alpha_D + \alpha_Y - 2) \phi_0 + (\alpha_D + \alpha_Y + 2) \phi_1
                ~ d \phi_1 d \phi_0 \\
            & ~~~~~~~~~~
                + \int_{\frac{1}{1 + \alpha_D + \alpha_Y}}^1
                \int_{\tan \big( \arctan \frac{1}{\alpha_D + \alpha_Y} - \frac{\pi}{4} \big) \phi_0}^{\frac{1}{1 + \alpha_D + \alpha_Y}}
                (\alpha_D + \alpha_Y - 2) \phi_0 + (\alpha_D + \alpha_Y + 2) \phi_1
                ~ d \phi_1 d \phi_0 \\
            & ~~~~~~~~~~
                + \int_{\frac{1}{1 + \alpha_D + \alpha_Y}}^1
                \int_{\frac{1}{1 + \alpha_D + \alpha_Y}}^{\phi_0}
                1 - (2 - \alpha_D - \alpha_Y) \phi_0 + \phi_1
                ~ d \phi_1 d \phi_0 \bigg\} \\
            & ~~ + \gamma_{0101}^{\text{(up)}} \cdot \bigg\{
                \int_0^{\frac{1}{1 + \alpha_D + \alpha_Y}}
                \int_{\phi_0}^{\frac{1}{1 + \alpha_D + \alpha_Y}}
                (\alpha_D + \alpha_Y) (\phi_0 + \phi_1)
                ~ d \phi_1 d \phi_0 \\
            & ~~~~~~~~~~
                + \int_0^{\frac{1}{1 + \alpha_D + \alpha_Y}}
                \int_{\frac{1}{1 + \alpha_D + \alpha_Y}}^1
                1 + (\alpha_D + \alpha_Y) \phi_0 - \phi_1
                ~ d \phi_1 d \phi_0 \\
            & ~~~~~~~~~~
                + \int_{\frac{1}{1 + \alpha_D + \alpha_Y}}^1
                \int_{\phi_0}^1
                1 + (\alpha_D + \alpha_Y) \phi_0 - \phi_1
                ~ d \phi_1 d \phi_0 \bigg\},
        \end{split}
    \end{equation*}
    \raggedbottom

    \begin{equation*}
        \begin{split}
            & ~~ \int_{\epsilon \in \Epsilon}
            \int_{\xi \in \Epsilon}
            \big( \lvert \varphi_{t + 1}(\xi) \rvert - \lvert \varphi_t(\epsilon) \rvert \big)
            \cdot \mathbbm{1}\{ \varphi_{t + 1}(\xi) = G(f_t, g_t^D, g_t^{\Yori}; 1, 0, 1, 0, \epsilon, \alpha_D, \alpha_Y) \} \\
            & ~~~~~~~~~~
            \cdot q_t\big(f_t(0, \epsilon), f_t(1, \epsilon) \mid 1, 0, 1, 0 \big) d \xi d \epsilon \\
            & =
            \gamma_{1010}^{\text{(up)}} \cdot \bigg\{
                \int_0^1
                \int_0^{\tan^{-1} \big( \frac{3 \pi}{4} - \arctan \frac{1}{\alpha_D + \alpha_Y} \big) \phi_1}
                - (\alpha_D + \alpha_Y) (\phi_0 + \phi_1)
                ~ d \phi_0 d \phi_1 \\
            & ~~~~~~~~~~
                + \int_0^{\frac{1}{1 + \alpha_D + \alpha_Y}}
                \int_{\tan^{-1} \big( \frac{3 \pi}{4} - \arctan \frac{1}{\alpha_D + \alpha_Y} \big) \phi_1}^{\phi_1}
                (\alpha_D + \alpha_Y + 2) \phi_0 + (\alpha_D + \alpha_Y - 2) \phi_1
                ~ d \phi_0 d \phi_1 \\
            & ~~~~~~~~~~
                + \int_{\frac{1}{1 + \alpha_D + \alpha_Y}}^1
                \int_{\tan^{-1} \big( \frac{3 \pi}{4} - \arctan \frac{1}{\alpha_D + \alpha_Y} \big) \phi_1}^{\frac{1}{1 + \alpha_D + \alpha_Y}}
                (\alpha_D + \alpha_Y + 2) \phi_0 + (\alpha_D + \alpha_Y - 2) \phi_1
                ~ d \phi_0 d \phi_1 \\
            & ~~~~~~~~~~
                + \int_{\frac{1}{1 + \alpha_D + \alpha_Y}}^1
                \int_{\frac{1}{1 + \alpha_D + \alpha_Y}}^{\phi_1}
                1 + \phi_0 - (2 - \alpha_D - \alpha_Y) \phi_1
                ~ d \phi_0 d \phi_1 \bigg\} \\
            & ~~ + \gamma_{1010}^{\text{(low)}} \cdot \bigg\{
                \int_0^{\frac{1}{1 + \alpha_D + \alpha_Y}}
                \int_{\phi_1}^{\frac{1}{1 + \alpha_D + \alpha_Y}}
                (\alpha_D + \alpha_Y) (\phi_0 + \phi_1)
                ~ d \phi_0 d \phi_1 \\
            & ~~~~~~~~~~
                + \int_0^{\frac{1}{1 + \alpha_D + \alpha_Y}}
                \int_{\frac{1}{1 + \alpha_D + \alpha_Y}}^1
                1 - \phi_0 + (\alpha_D + \alpha_Y) \phi_1
                ~ d \phi_0 d \phi_1 \\
            & ~~~~~~~~~~
                + \int_{\frac{1}{1 + \alpha_D + \alpha_Y}}^1
                \int_{\phi_1}^1
                1 - \phi_0 + (\alpha_D + \alpha_Y) \phi_1
                ~ d \phi_0 d \phi_1 \bigg\}.
        \end{split}
    \end{equation*}

    Since $\gamma_{0101}^{\text{(low)}} + \gamma_{0101}^{\text{(up)}} = 2$ and $\gamma_{1010}^{\text{(low)}} + \gamma_{1010}^{\text{(up)}} = 2$, we can derive the form of $\Delta^{(\text{Perfect Predictor})}_{\mathrm{STIR}} \big\rvert_{t}^{t + 1}$:
    \begin{equation*}
        \begin{split}
            \Delta^{(\text{Perfect Predictor})}_{\mathrm{STIR}} \big\rvert_{t}^{t + 1}
            & = \big(
                P_t(0, 1, 0, 1) \cdot \gamma_{0101}^{\text{(low)}}
                + P_t(1, 0, 1, 0) \cdot \gamma_{1010}^{\text{(up)}}
            \big)
            \cdot \bigg\{ \\
                & ~~~~~~~~~~
                - \frac{1 + \alpha_D + \alpha_Y}{3}
                \cdot \tan^2 \big( \arctan \frac{1}{\alpha_D + \alpha_Y} - \frac{\pi}{4} \big) \\
                & ~~~~~~~~~~
                + \frac{2 (1 - \alpha_D - \alpha_Y)}{3}
                \cdot \tan \big( \arctan \frac{1}{\alpha_D + \alpha_Y} - \frac{\pi}{4} \big) \\
                & ~~~~~~~~~~
                - \frac{1 - \alpha_D - \alpha_Y}{6}
                + \frac{3 - \alpha_D - \alpha_Y}{2 (1 + \alpha_D + \alpha_Y) } \\
                & ~~~~~~~~~~
                + \frac{
                    3 (\alpha_D + \alpha_Y)^3
                    - 6 (\alpha_D + \alpha_Y)^2
                    - 19 (\alpha_D + \alpha_Y) - 10}{
                        6 (1 + \alpha_D + \alpha_Y)^3}
            \bigg\} \\
            & ~~ + \big( P_t(0, 1, 0, 1) + P_t(1, 0, 1, 0) \big) \cdot \bigg[ \\
                & ~~~~~~~~~~
                \frac{\alpha_D + \alpha_Y - 2}{3}
                + \frac{\alpha_D + \alpha_Y}{1 + \alpha_D + \alpha_Y}
                + \frac{
                    3 (\alpha_D + \alpha_Y)^2
                    +5 (\alpha_D + \alpha_Y) + 2}{
                        3 (1 + \alpha_D + \alpha_Y)^3}
            \bigg],  
        \end{split}
    \end{equation*}

    where $\gamma_{0101}^{\text{(low)}}, \gamma_{1010}^{\text{(up)}} \in (0, 1)$ (according to Assumption \ref{assumption:biased_weight_quantitative}), and $\alpha_D, \alpha_Y \in [0, \frac{1}{2})$ (according to Assumption \ref{assumption:Hupdate}).

    Let us denote $\beta(\alpha_D, \alpha_Y) \coloneqq \tan \big( \arctan \frac{1}{\alpha_D + \alpha_Y} - \frac{\pi}{4} \big)$ to simplify the notation.
    Without loss of generality let us assume that
    $P_t(0, 1, 0, 1) \cdot \gamma_{0101}^{\text{(low)}}
    + P_t(1, 0, 1, 0) \cdot \gamma_{1010}^{\text{(up)}} > 0$.

    We can further compute the partial derivatives and find out that:
    \begin{equation*}
        \begin{split}
            &~~~~~
            \frac{\partial \Delta^{(\text{Perfect Predictor})}_{\mathrm{STIR}} \big\rvert_{t}^{t + 1}}{\partial \big(
                P_t(0, 1, 0, 1) \cdot \gamma_{0101}^{\text{(low)}}
                + P_t(1, 0, 1, 0) \cdot \gamma_{1010}^{\text{(up)}}
            \big)} \\
            &=
            - \frac{1 + \alpha_D + \alpha_Y}{3} \cdot \beta^2(\alpha_D, \alpha_Y) \\
            & ~~~~~
            + \frac{2 (1 - \alpha_D - \alpha_Y)}{3} \cdot \beta(\alpha_D, \alpha_Y) \\
            & ~~~~~
            - \frac{1 - \alpha_D - \alpha_Y}{6}
            + \frac{3 - \alpha_D - \alpha_Y}{2 (1 + \alpha_D + \alpha_Y) } \\
            & ~~~~~
            + \frac{
                3 (\alpha_D + \alpha_Y)^3
                - 6 (\alpha_D + \alpha_Y)^2
                - 19 (\alpha_D + \alpha_Y) - 10}{
                    6 (1 + \alpha_D + \alpha_Y)^3} \\
            &
            < 0, ~ \forall ~ \alpha_D, \alpha_Y \in [0, \frac{1}{2}),
        \end{split}
    \end{equation*}

    and that:

    \begin{equation*}
        \begin{split}
            \frac{\partial \Delta^{(\text{Perfect Predictor})}_{\mathrm{STIR}} \big\rvert_{t}^{t + 1}}{\partial (\alpha_D + \alpha_Y)}
            & =
            \big(
                P_t(0, 1, 0, 1) + P_t(1, 0, 1, 0)
            \big) \cdot \bigg[
                \frac{1}{3} + \frac{2}{3 (1 + \alpha_D + \alpha_Y)^3}
            \bigg] \\
            & ~~~~~~
            + \big(
                P_t(0, 1, 0, 1) \cdot \gamma_{0101}^{\text{(low)}}
                + P_t(1, 0, 1, 0) \cdot \gamma_{1010}^{\text{(up)}}
            \big)
            \cdot \bigg[ \\
            & ~~~~~~~~~~
                - \frac{2 (1 + \alpha_D + \alpha_Y)}{3}
                \cdot \beta(\alpha_D, \alpha_Y)
                \cdot \frac{\partial \beta(\alpha_D, \alpha_Y)}{\partial (\alpha_D + \alpha_Y)} \\
            & ~~~~~~~~~~
                + \frac{2 (1 - \alpha_D - \alpha_Y)}{3}
                \cdot \frac{\partial \beta(\alpha_D, \alpha_Y)}{\partial (\alpha_D + \alpha_Y)} \\
            & ~~~~~~~~~~
                + \frac{1}{6} - \frac{2}{(1 + \alpha_D + \alpha_Y)^2}
                + \frac{15 (\alpha_D + \alpha_Y) + 11}{6 (1 + \alpha_D + \alpha_Y)^3} \bigg] \\
            & =
            \big(
                P_t(0, 1, 0, 1) + P_t(1, 0, 1, 0)
            \big) \cdot \bigg[
                \frac{1}{3} + \frac{2}{3 (1 + \alpha_D + \alpha_Y)^3}
            \bigg] \\
            & ~~~~~~
            + \big(
                P_t(0, 1, 0, 1) \cdot \gamma_{0101}^{\text{(low)}}
                + P_t(1, 0, 1, 0) \cdot \gamma_{1010}^{\text{(up)}}
            \big)
            \cdot \bigg[ \\
            & ~~~~~~~~~~
                \frac{
                    2 \big( 1 + \beta^2(\alpha_D, \alpha_Y) \big)
                    \cdot \big[ (1 + \alpha_D + \alpha_Y ) \beta(\alpha_D, \alpha_Y)
                    + \alpha_D + \alpha_Y - 1 \big]
                }{3 (1 + \alpha_D + \alpha_Y)^3} \\
            & ~~~~~~~~~~
                + \frac{(\alpha_D + \alpha_Y)^3 + 3 (\alpha_D + \alpha_Y)^2 + 6 (\alpha_D + \alpha_Y)}{6 (1 + \alpha_D + \alpha_Y)^3} \bigg] \\
            & > 0, ~ \forall ~ \gamma_{0101}^{\text{(low)}}, \gamma_{1010}^{\text{(up)}} \in (0, 1), \alpha_D, \alpha_Y \in [0, \frac{1}{2}),
        \end{split}
    \end{equation*}
    where we utilize the fact that
    $\frac{\partial \beta(\alpha_D, \alpha_Y)}{\partial (\alpha_D + \alpha_Y)}
    = \big( 1 + \beta^2(\alpha_D, \alpha_Y) \big) \cdot \frac{1}{1 + (\alpha_D + \alpha_Y)^2}$.

    Therefore, we can conclude that
    \begin{equation*}
        \begin{split}
            \Delta^{(\text{Perfect Predictor})}_{\mathrm{STIR}} \big\rvert_{t}^{t + 1}
            & >
            \lim_{\substack{
                \gamma_{0101}^{\text{(low)}} \rightarrow 1 \\
                \gamma_{1010}^{\text{(up)}} \rightarrow 1 \\
                \alpha_D + \alpha_Y \rightarrow 0}}
            \big(
                P_t(0, 1, 0, 1) \cdot \gamma_{0101}^{\text{(low)}}
                + P_t(1, 0, 1, 0) \cdot \gamma_{1010}^{\text{(up)}}
            \big)
            \cdot \bigg\{ \\
                & ~~~~~~~~~~~~~
                - \frac{1 + \alpha_D + \alpha_Y}{3}
                \cdot \tan^2 \big( \arctan \frac{1}{\alpha_D + \alpha_Y} - \frac{\pi}{4} \big) \\
                & ~~~~~~~~~~~~~
                + \frac{2 (1 - \alpha_D - \alpha_Y)}{3}
                \cdot \tan \big( \arctan \frac{1}{\alpha_D + \alpha_Y} - \frac{\pi}{4} \big) \\
                & ~~~~~~~~~~~~~
                - \frac{1 - \alpha_D - \alpha_Y}{6}
                + \frac{3 - \alpha_D - \alpha_Y}{2 (1 + \alpha_D + \alpha_Y) } \\
                & ~~~~~~~~~~~~~
                + \frac{
                    3 (\alpha_D + \alpha_Y)^3
                    - 6 (\alpha_D + \alpha_Y)^2
                    - 19 (\alpha_D + \alpha_Y) - 10}{
                        6 (1 + \alpha_D + \alpha_Y)^3}
            \bigg\} \\
            & ~~~~~~~~~~~~~
            + \big( P_t(0, 1, 0, 1) + P_t(1, 0, 1, 0) \big) \cdot \bigg[ \\
                & ~~~~~~~~~~~~~
                \frac{\alpha_D + \alpha_Y - 2}{3}
                + \frac{\alpha_D + \alpha_Y}{1 + \alpha_D + \alpha_Y}
                + \frac{
                    3 (\alpha_D + \alpha_Y)^2
                    +5 (\alpha_D + \alpha_Y) + 2}{
                        3 (1 + \alpha_D + \alpha_Y)^3}
            \bigg] \\
            & = 0,
        \end{split}
    \end{equation*}
    i.e., under the specified assumptions and dynamics, we have
    \begin{equation}\label{equ:STIR_perfect_result}
        \forall ~ \gamma_{0101}^{\text{(low)}}, \gamma_{1010}^{\text{(up)}} \in (0, 1), \alpha_D, \alpha_Y \in [0, \frac{1}{2}):
        \Delta^{(\text{Perfect Predictor})}_{\mathrm{STIR}} \big\rvert_{t}^{t + 1} > 0.
    \end{equation}
\end{proof}

\subsection{Proof for Theorem \ref{theorem:CF_one_step}}\label{appendix:proof_thm_CF}
\begin{thm}
    Let us consider the general situation where both $D_t$ and $\Yori_t$ are dependent with $A_t$, i.e., $D_t \nindep A_t, \Yori_t \nindep A_t$.
    Let us further assume that the data dynamics satisfies $\alpha_D \in (0,\frac{1}{2}), \alpha_Y = 0$.
    Then under Fact \ref{fact:var_form}, Assumption \ref{assumption:Hupdate}, and Assumption \ref{assumption:samegroup}, and Assumption \ref{assumption:biased_weight_quantitative}, as well as the specified dynamics, when $H_t \nindep A_t$, it is possible for the Counterfactual Fair predictor to get closer to the long-term fairness goal after one-step intervention, if certain properties of the data dynamics and the predictor behavior are satisfied simultaneously, i.e.,
    \begin{equation*}
        \begin{split}
            &\begin{cases}
                & g^D_t(0, E_t) = g^D_t(1, E_t) \\
                & \frac{P_t(1, 1, 0, 1) + P_t(1, 1, 1, 0)}{P_t(0, 0, 0, 1) + P_t(0, 0, 1, 0)} < \frac{27}{8} \\
                & \alpha_D \in \left( \big(\frac{P_t(1, 1, 0, 1) + P_t(1, 1, 1, 0)}{P_t(0, 0, 0, 1) + P_t(0, 0, 1, 0)}\big)^{\frac{1}{3}} - 1, \frac{1}{2} \right) \\
                & \alpha_Y = 0
            \end{cases} \\
            &\implies \Delta^{(\text{Counterfactual Fair})}_{\mathrm{STIR}} \big\rvert_{t}^{t + 1} < 0.
        \end{split}
    \end{equation*}
\end{thm}
\begin{proof}[Proof (sketch)]
    Similar to proving Theorem \ref{theorem:perfect_one_step} (proof in \Cref{appendix:proof_thm_perfect}), the goal is to calculate if it is possible for the \textit{Single-step Tier Imbalance Reduction} $\STIR$ to be smaller than $0$ when using Counterfactual Fair predictors.

    Since $\STIR$ is a weighted aggregation of $\lvert \varphi(e_{t + 1}) \rvert - \lvert \varphi(e_t)\rvert$ (as defined in Equation \ref{equ:STIR}), the quantitative analysis involves three key components: instantiations of $\varphi_{t + 1}(e_{t + 1})$, the knowledge/assumptions on $q_t\big(f_t(0, \epsilon), f_t(1, \epsilon) \mid d, d', y, y' \big)$, and characteristics of $P_t(d, d', y', y')$.

    For the first component, 
    \reBlue{C9}{We start from considering $\alpha_D > \alpha_Y$ (which contains $\alpha_Y = 0$ as a special case) and only apply $\alpha_Y = 0$ in the end of the proof.
    Our proof does not involve additional constraints on $\gamma$ or $q_t()$ beyond Assumption \ref{assumption:biased_weight_quantitative} and the fact that we are considering \textit{Counterfactual Fair} predictors, both of which are specified in the Theorem's hypothesis.}{0in}\leftSide
    since $\alpha_Y = 0$ is a special case of scenarios where $\alpha_D > \alpha_Y$, we can list all possible instantiations of $\varphi_{t + 1}(e_{t + 1})$ in Table \ref{table:delta_next_Dmajor} (when $\alpha_D > \alpha_Y$).
    For the second component, we can introduce a quantitative assumption on $q_t\big(f_t(0, \epsilon), f_t(1, \epsilon) \mid d, d', y, y' \big)$ (Assumption \ref{assumption:biased_weight_quantitative}).
    For the third component, we need to exploit the characteristic of the predictor of interest to gain further insight into the joint distribution $P_t(d, d', y, y')$.
    For Counterfactual Fair predictors, we have $P_t(d, d', y, y')$ satisfies Equation \ref{equ:CF_predictor_joint} (as we have discussed in \Cref{sec:one_step_analysis_CF}).

    For the purpose of calculating the value of $\STIR$, the proof contains two steps: (1) exhaustively derive the value of $\lvert \varphi(e_{t + 1}) \rvert - \lvert \varphi(e_t)\rvert$ after one-step dynamics (finished in \Cref{appendix:proof_thm_perfect} when proving Theorem \ref{theorem:perfect_one_step}), and (2) aggregate the difference $\lvert \varphi(e_{t + 1}) \rvert - \lvert \varphi(e_t)\rvert$ with the help of the additional knowledge/assumptions on $q_t\big(f_t(0, \epsilon), f_t(1, \epsilon) \mid d, d', y, y' \big)$ and $P_t(d, d', y, y')$.
\end{proof}
\begin{proof}[Proof (full)]
    Based on the definition of $\STIR$, the proof calculates the aggregation (integration followed by summation) of the difference $\lvert \varphi(e_{t + 1}) \rvert - \lvert \varphi(e_t)\rvert$ with the help of the additional knowledge/assumptions on $q_t\big(f_t(0, \epsilon), f_t(1, \epsilon) \mid d, d', y, y' \big)$ and $P_t(d, d', y, y')$.

    Since we assume $\alpha_D \in (0,\frac{1}{2}), \alpha_Y = 0$, we focus on possible instantiations of $\varphi_{t + 1}(e_{t + 1})$ as listed in Table \ref{table:delta_next_Dmajor} ($\alpha_D > \alpha_Y$).
    For the Counterfactual Fair predictor that satisfies $g^D_t(0, E_t) = g^D_t(1, E_t)$, not every case in Table \ref{table:delta_next_Dmajor} corresponds to a nonzero $P_t(d, d', y, y')$ and therefore may not contribute to the computation of $\STIR$ as detailed in Equation \ref{equ:STIR}.
    By applying Equation \ref{equ:CF_predictor_joint} we need to consider Case (i), Case (ii), Case (iii), Case (iv), Case (xiii), Case (xiv), Case (xv), and Case (xvi) in Table \ref{table:delta_next_Dmajor}.

    We are interested in cases where at least one of $g^D_t(\cdot, E_t)$ and $g^{\Yori}_t(\cdot, E_t)$ are functions of $A_t$.
    Therefore we only need to calculate $\Delta^{(\text{Counterfactual Fair})}_{\mathrm{STIR}} \big\rvert_{t}^{t + 1}$ for Case (ii), Case (iii), Case (xiv), and Case (xv) (although $\alpha_Y = 0$, we explicitly keep the hyperparameter $\alpha_Y$ in the proof for the purpose of notation consistency).

    According to Equation \ref{equ:STIR} and Equation \ref{appendix:equ:mapping_across_t}, for the Counterfactual Fair predictor we have:
    \begin{equation}\label{equ:STIR_CF}
        \begin{split}
            \Delta^{(\text{Counterfactual Fair})}_{\mathrm{STIR}} \big\rvert_{t}^{t + 1}
            &= P_t(0, 0, 0, 1) \cdot
            \int_{\epsilon \in \Epsilon}
            \int_{\xi \in \Epsilon}
            \big( \lvert \varphi_{t + 1}(\xi) \rvert - \lvert \varphi_t(\epsilon) \rvert \big) \\
            & ~~~~~~~~~~
            \cdot \mathbbm{1}\{ \varphi_{t + 1}(\xi) = G(f_t, g_t^D, g_t^{\Yori}; 0, 0, 0, 1, \epsilon, \alpha_D, \alpha_Y) \} \\
            & ~~~~~~~~~~
            \cdot q_t\big(f_t(0, \epsilon), f_t(1, \epsilon) \mid 0, 0, 0, 1 \big)
            d \xi d \epsilon \\
            & ~~ + P_t(0, 0, 1, 0) \cdot
            \int_{\epsilon \in \Epsilon}
            \int_{\xi \in \Epsilon}
            \big( \lvert \varphi_{t + 1}(\xi) \rvert - \lvert \varphi_t(\epsilon) \rvert \big) \\
            & ~~~~~~~~~~
            \cdot \mathbbm{1}\{ \varphi_{t + 1}(\xi) = G(f_t, g_t^D, g_t^{\Yori}; 0, 0, 1, 0, \epsilon, \alpha_D, \alpha_Y) \} \\
            & ~~~~~~~~~~
            \cdot q_t\big(f_t(0, \epsilon), f_t(1, \epsilon) \mid 0, 0, 1, 0 \big)
            d \xi d \epsilon \\
            & ~~ + P_t(1, 1, 0, 1) \cdot
            \int_{\epsilon \in \Epsilon}
            \int_{\xi \in \Epsilon}
            \big( \lvert \varphi_{t + 1}(\xi) \rvert - \lvert \varphi_t(\epsilon) \rvert \big) \\
            & ~~~~~~~~~~
            \cdot \mathbbm{1}\{ \varphi_{t + 1}(\xi) = G(f_t, g_t^D, g_t^{\Yori}; 1, 1, 0, 1, \epsilon, \alpha_D, \alpha_Y) \} \\
            & ~~~~~~~~~~
            \cdot q_t\big(f_t(0, \epsilon), f_t(1, \epsilon) \mid 1, 1, 0, 1 \big)
            d \xi d \epsilon \\
            & ~~ + P_t(1, 1, 1, 0) \cdot
            \int_{\epsilon \in \Epsilon}
            \int_{\xi \in \Epsilon}
            \big( \lvert \varphi_{t + 1}(\xi) \rvert - \lvert \varphi_t(\epsilon) \rvert \big) \\
            & ~~~~~~~~~~
            \cdot \mathbbm{1}\{ \varphi_{t + 1}(\xi) = G(f_t, g_t^D, g_t^{\Yori}; 1, 1, 1, 0, \epsilon, \alpha_D, \alpha_Y) \} \\
            & ~~~~~~~~~~
            \cdot q_t\big(f_t(0, \epsilon), f_t(1, \epsilon) \mid 1, 1, 1, 0 \big)
            d \xi d \epsilon.
        \end{split}
    \end{equation}

    Similar to the proof of the result for perfect predictors presented in \Cref{appendix:proof_thm_perfect}, with the help of Assumption \ref{assumption:biased_weight_quantitative}, we convert the conditional expectations in Equation \ref{equ:STIR_CF} into calculations of multiple integrals on slices within a $1 \times 1$ square on the 2-D plane, where $\phi_0$ and $\phi_1$ axes correspond to the value of $f_t(0, E_t)$ and $f_t(1, E_t)$ respectively:
    \raggedbottom

    \begin{equation*}
        \begin{split}
            & ~~ \int_{\epsilon \in \Epsilon}
            \int_{\xi \in \Epsilon}
            \big( \lvert \varphi_{t + 1}(\xi) \rvert - \lvert \varphi_t(\epsilon) \rvert \big)
            \cdot \mathbbm{1}\{ \varphi_{t + 1}(\xi) = G(f_t, g_t^D, g_t^{\Yori}; 0, 0, 0, 1, \epsilon, \alpha_D, \alpha_Y) \} \\
            & ~~~~~~~~~~
            \cdot q_t\big(f_t(0, \epsilon), f_t(1, \epsilon) \mid 0, 0, 0, 1 \big) d \xi d \epsilon \\
            & =
            \gamma_{0001}^{\text{(low)}} \cdot \bigg\{
                \int_0^1
                \int_0^{\frac{1 - \alpha_D - \alpha_Y}{1 - \alpha_D + \alpha_Y} \phi_0}
                - (\alpha_D + \alpha_Y) \phi_0 + (\alpha_D - \alpha_Y) \phi_1
                ~ d \phi_1 d \phi_0 \\
            & ~~~~~~~~~~
                + \int_0^1
                \int_{\frac{1 - \alpha_D - \alpha_Y}{1 - \alpha_D + \alpha_Y} \phi_0}^{\phi_0}
                - (2 - \alpha_D - \alpha_Y) \phi_0 + (2 - \alpha_D + \alpha_Y) \phi_1
                ~ d \phi_1 d \phi_0 \bigg\} \\
            & ~~ + \gamma_{0001}^{\text{(up)}} \cdot
                \int_0^1
                \int_0^{\phi_1}
                (\alpha_D + \alpha_Y) \phi_0 - (\alpha_D - \alpha_Y) \phi_1
                ~ d \phi_0 d \phi_1,
        \end{split}
    \end{equation*}

    \begin{equation*}
        \begin{split}
            & ~~ \int_{\epsilon \in \Epsilon}
            \int_{\xi \in \Epsilon}
            \big( \lvert \varphi_{t + 1}(\xi) \rvert - \lvert \varphi_t(\epsilon) \rvert \big)
            \cdot \mathbbm{1}\{ \varphi_{t + 1}(\xi) = G(f_t, g_t^D, g_t^{\Yori}; 0, 0, 1, 0, \epsilon, \alpha_D, \alpha_Y) \} \\
            & ~~~~~~~~~~
            \cdot q_t\big(f_t(0, \epsilon), f_t(1, \epsilon) \mid 0, 0, 1, 0 \big) d \xi d \epsilon \\
            & =
            \gamma_{0010}^{\text{(up)}} \cdot \bigg\{
                \int_0^1
                \int_0^{\frac{1 - \alpha_D - \alpha_Y}{1 - \alpha_D + \alpha_Y} \phi_1}
                (\alpha_D - \alpha_Y) \phi_0 - (\alpha_D + \alpha_Y) \phi_1
                ~ d \phi_0 d \phi_1 \\
            & ~~~~~~~~~~
                + \int_0^1
                \int_{\frac{1 - \alpha_D - \alpha_Y}{1 - \alpha_D + \alpha_Y} \phi_1}^{\phi_1}
                (2 - \alpha_D + \alpha_Y) \phi_0 - (2 - \alpha_D - \alpha_Y) \phi_1
                ~ d \phi_0 d \phi_1 \bigg\} \\
            & ~~ + \gamma_{0010}^{\text{(low)}} \cdot
                \int_0^1
                \int_0^{\phi_0}
                - (\alpha_D - \alpha_Y) \phi_0 + (\alpha_D + \alpha_Y) \phi_1
                ~ d \phi_1 d \phi_0,
        \end{split}
    \end{equation*}

    \newpage
    \begin{equation*}
        \begin{split}
            & ~~ \int_{\epsilon \in \Epsilon}
            \int_{\xi \in \Epsilon}
            \big( \lvert \varphi_{t + 1}(\xi) \rvert - \lvert \varphi_t(\epsilon) \rvert \big)
            \cdot \mathbbm{1}\{ \varphi_{t + 1}(\xi) = G(f_t, g_t^D, g_t^{\Yori}; 1, 1, 0, 1, \epsilon, \alpha_D, \alpha_Y) \} \\
            & ~~~~~~~~~~
            \cdot q_t\big(f_t(0, \epsilon), f_t(1, \epsilon) \mid 1, 1, 0, 1 \big) d \xi d \epsilon \\
            & =
            \gamma_{1101}^{\text{(low)}} \cdot \bigg\{
                \int_0^{\frac{1}{1 + \alpha_D + \alpha_Y}}
                \int_{\phi_1}^{\frac{1 + \alpha_D + \alpha_Y}{1 + \alpha_D - \alpha_Y} \phi_1}
                - (2 + \alpha_D - \alpha_Y) \phi_0 + (2 + \alpha_D + \alpha_Y) \phi_1
                ~ d \phi_0 d \phi_1 \\
            & ~~~~~~~~~~
                + \int_0^{\frac{1}{1 + \alpha_D - \alpha_Y}}
                \int_0^{\frac{1 + \alpha_D - \alpha_Y}{1 + \alpha_D + \alpha_Y} \phi_0}
                (\alpha_D - \alpha_Y) \phi_0 - (\alpha_D + \alpha_Y) \phi_1
                ~ d \phi_1 d \phi_0 \\
            & ~~~~~~~~~~
                + \int_{\frac{1}{1 + \alpha_D + \alpha_Y}}^{\frac{1}{1 + \alpha_D - \alpha_Y}}
                \int_{\frac{1}{1 + \alpha_D + \alpha_Y}}^{\phi_0}
                1 - (2 + \alpha_D - \alpha_Y) \phi_0 + \phi_1
                ~ d \phi_1 d \phi_0 \bigg\} \\
            & ~~ + \gamma_{1101}^{\text{(up)}} \cdot \bigg\{
                \int_0^{\frac{1}{1 + \alpha_D + \alpha_Y}}
                \int_0^{\phi_1}
                - (\alpha_D - \alpha_Y) \phi_0 + (\alpha_D + \alpha_Y) \phi_1
                ~ d \phi_0 d \phi_1 \\
            & ~~~~~~~~~~
                + \int_0^{\frac{1}{1 + \alpha_D + \alpha_Y}}
                \int_{\frac{1}{1 + \alpha_D + \alpha_Y}}^1
                1 - (\alpha_D - \alpha_Y) \phi_0 - \phi_1
                ~ d \phi_1 d \phi_0 \\
            & ~~~~~~~~~~
                + \int_{\frac{1}{1 + \alpha_D + \alpha_Y}}^{\frac{1}{1 + \alpha_D - \alpha_Y}}
                \int_{\phi_0}^1
                1 - (\alpha_D - \alpha_Y) \phi_0 - \phi_1
                ~ d \phi_1 d \phi_0 \bigg\},
        \end{split}
    \end{equation*}

    \begin{equation*}
        \begin{split}
            & ~~ \int_{\epsilon \in \Epsilon}
            \int_{\xi \in \Epsilon}
            \big( \lvert \varphi_{t + 1}(\xi) \rvert - \lvert \varphi_t(\epsilon) \rvert \big)
            \cdot \mathbbm{1}\{ \varphi_{t + 1}(\xi) = G(f_t, g_t^D, g_t^{\Yori}; 1, 1, 1, 0, \epsilon, \alpha_D, \alpha_Y) \} \\
            & ~~~~~~~~~~
            \cdot q_t\big(f_t(0, \epsilon), f_t(1, \epsilon) \mid 1, 1, 1, 0 \big) d \xi d \epsilon \\
            & =
            \gamma_{1110}^{\text{(up)}} \cdot \bigg\{
                \int_0^{\frac{1}{1 + \alpha_D + \alpha_Y}}
                \int_{\phi_0}^{\frac{1 + \alpha_D + \alpha_Y}{1 + \alpha_D - \alpha_Y} \phi_0}
                (2 + \alpha_D + \alpha_Y) \phi_0 - (2 + \alpha_D - \alpha_Y) \phi_1
                ~ d \phi_1 d \phi_0 \\
            & ~~~~~~~~~~
                + \int_0^{\frac{1}{1 + \alpha_D - \alpha_Y}}
                \int_0^{\frac{1 + \alpha_D - \alpha_Y}{1 + \alpha_D + \alpha_Y} \phi_1}
                - (\alpha_D + \alpha_Y) \phi_0 + (\alpha_D - \alpha_Y) \phi_1
                ~ d \phi_0 d \phi_1 \\
            & ~~~~~~~~~~
                + \int_{\frac{1}{1 + \alpha_D - \alpha_Y}}^1
                \int_0^{\frac{1}{1 + \alpha_D + \alpha_Y}}
                1 - \phi_0 - (\alpha_D + \alpha_Y) \phi_1
                ~ d \phi_1 d \phi_0 \bigg\} \\
            & ~~ + \gamma_{1110}^{\text{(low)}} \cdot
            \int_0^{\frac{1}{1 + \alpha_D + \alpha_Y}}
                \int_0^{\phi_0}
                (\alpha_D + \alpha_Y) \phi_0 - (\alpha_D - \alpha_Y) \phi_1
                ~ d \phi_1 d \phi_0.
        \end{split}
    \end{equation*}

    \newpage
    Since $\gamma_{0001}^{\text{(low)}} + \gamma_{0001}^{\text{(up)}} = 2$, $\gamma_{0010}^{\text{(low)}} + \gamma_{0010}^{\text{(up)}} = 2$, $\gamma_{1101}^{\text{(low)}} + \gamma_{1101}^{\text{(up)}} = 2$, and $\gamma_{1110}^{\text{(low)}} + \gamma_{1110}^{\text{(up)}} = 2$, we can derive the form of the term $\Delta^{(\text{Counterfactual Fair})}_{\mathrm{STIR}} \big\rvert_{t}^{t + 1}$:
    \begin{equation*}
        \begin{split}
            & \Delta^{(\text{Counterfactual Fair})}_{\mathrm{STIR}} \big\rvert_{t}^{t + 1}
            =
            \big( P_t(0, 0, 0, 1) \cdot \gamma_{0001}^{\text{(low)}}
            + P_t(0, 0, 1, 0) \cdot \gamma_{0010}^{\text{(up)}} \big)
            \cdot \bigg\{ \\
                & ~~~~~~~~
                \frac{(\alpha_D - \alpha_Y) (1 - \alpha_D - \alpha_Y)^2}{6 (1 - \alpha_D + \alpha_Y)^2}
                - \frac{(\alpha_D + \alpha_Y) (1 - \alpha_D - \alpha_Y)}{3 (1 - \alpha_D + \alpha_Y)} \\
                & ~~~~~~~~
                + \frac{2 \alpha_Y}{3 (1 - \alpha_D + \alpha_Y)} \bigg[ -(2 - \alpha_D - \alpha_Y) + \frac{(2 - \alpha_D + \alpha_Y) (1 - \alpha_D)}{1 - \alpha_D + \alpha_Y} \bigg]
                + \frac{\alpha_D}{6} - \frac{\alpha_Y}{2} \bigg\} \\
            & ~~
            + P_t(1, 1, 0, 1) \cdot \gamma_{1101}^{\text{(low)}}
            \cdot \bigg\{ \\
                & ~~~~~~~~
                \frac{2 \alpha_Y}{3 (1 + \alpha_D - \alpha_Y) (1 + \alpha_D + \alpha_Y)^3} \bigg[ (2 + \alpha_D + \alpha_Y) - \frac{(2 + \alpha_D - \alpha_Y) (1 + \alpha_D)}{1 + \alpha_D - \alpha_Y} \bigg] \\
                & ~~~~~~~~
                + \frac{1}{3 (1 + \alpha_D + \alpha_Y) (1 + \alpha_D - \alpha_Y)^3} \bigg[ \\
                & ~~~~~~~~~~~~~~
                (\alpha_D - \alpha_Y)(1 + \alpha_D - \alpha_Y) - \frac{(\alpha_D + \alpha_Y) (1 + \alpha_D - \alpha_Y)^2}{2 (1 + \alpha_D + \alpha_Y)} \bigg] \\
                & ~~~~~~~~
                - \frac{1}{(1 + \alpha_D + \alpha_Y)^3} \big( \frac{\alpha_D}{6} + \frac{\alpha_Y}{2} \big) \\
                & ~~~~~~~~
                -\frac{2}{3} ( 1 + \alpha_D - \alpha_Y ) \cdot \bigg[ \frac{1}{(1 + \alpha_D - \alpha_Y)^3} - \frac{1}{(1 + \alpha_D + \alpha_Y)^3} \bigg] \\
                & ~~~~~~~~
                + \bigg[ \frac{3 + 2 \alpha_D}{2 (1 + \alpha_D + \alpha_Y)} + \frac{1 + \alpha_D - \alpha_Y}{2}
                \bigg] \cdot \bigg[ \frac{1}{(1 + \alpha_D - \alpha_Y)^2} - \frac{1}{(1 + \alpha_D + \alpha_Y)^2} \bigg] \\
                & ~~~~~~~~
                - \bigg[ \frac{3 + 2 \alpha_D + 2 \alpha_Y}{2 (1 + \alpha_D + \alpha_Y)^2} + \frac{1}{2} \bigg] \cdot \bigg[ \frac{1}{1 + \alpha_D - \alpha_Y} - \frac{1}{1 + \alpha_D + \alpha_Y} \bigg]
                - \frac{(\alpha_D + \alpha_Y) \alpha_Y}{(1 + \alpha_D + \alpha_Y)^3} \bigg\} \\
            & ~~
            + P_t(1, 1, 1, 0) \cdot \gamma_{1110}^{\text{(up)}}
            \cdot \bigg\{ \\
                & ~~~~~~~~
                \frac{2 \alpha_Y}{3 (1 + \alpha_D - \alpha_Y) (1 + \alpha_D + \alpha_Y)^3} \bigg[ (2 + \alpha_D + \alpha_Y) - \frac{(2 + \alpha_D - \alpha_Y) (1 + \alpha_D)}{1 + \alpha_D - \alpha_Y} \bigg] \\
                & ~~~~~~~~
                + \frac{1}{3 (1 + \alpha_D + \alpha_Y) (1 + \alpha_D - \alpha_Y)^3} \bigg[ \\
                & ~~~~~~~~~~~~~~
                (\alpha_D - \alpha_Y)(1 + \alpha_D - \alpha_Y) - \frac{(\alpha_D + \alpha_Y) (1 + \alpha_D - \alpha_Y)^2}{2 (1 + \alpha_D + \alpha_Y)} \bigg] \\
                & ~~~~~~~~
                + \frac{\alpha_D - \alpha_Y}{(1 + \alpha_D + \alpha_Y) (1 + \alpha_D - \alpha_Y)} - \frac{(\alpha_D + \alpha_Y) (\alpha_D - \alpha_Y)}{2 (1 + \alpha_D + \alpha_Y)^2 (1 + \alpha_D - \alpha_Y)} \\
                & ~~~~~~~~
                - \frac{1}{2 (1 + \alpha_D + \alpha_Y)} \bigg[ 1 - \frac{1}{(1 + \alpha_D - \alpha_Y)^2} \bigg]
                - \big( \frac{\alpha_D}{6} + \frac{\alpha_Y}{2} \big)
                \frac{1}{(1 + \alpha_D + \alpha_Y)^3} \bigg\} \\
            & ~~
            + \big( P_t(0, 0, 0, 1) + P_t(0, 0, 1, 0) \big) \cdot \big( - \frac{\alpha_D}{3} + \alpha_Y \big) \\
            & ~~
            + P_t(1, 1, 0, 1) \cdot \bigg\{
                \frac{1}{(1 + \alpha_D + \alpha_Y)^3} \big( \frac{\alpha_D}{3} + \alpha_Y \big) \\
                & ~~~~~~~~
                + \frac{2 (\alpha_D + \alpha_Y) \alpha_Y}{(1 + \alpha_D + \alpha_Y)^3}
                + \frac{1}{1 + \alpha_D - \alpha_Y} - \frac{1}{1 + \alpha_D + \alpha_Y} \\
                & ~~~~~~~~
                + \big( \frac{1}{3} + \frac{2 \alpha_D}{3} - \frac{2 \alpha_Y}{3} \big) \cdot \bigg[ \frac{1}{(1 + \alpha_D - \alpha_Y)^3} - \frac{1}{(1 + \alpha_D + \alpha_Y)^3} \bigg] \\
                & ~~~~~~~~
                - (1 + \alpha_D - \alpha_Y) \cdot \bigg[ \frac{1}{(1 + \alpha_D - \alpha_Y)^2} - \frac{1}{(1 + \alpha_D + \alpha_Y)^2} \bigg]
            \bigg\} \\
            & ~~
            + P_t(1, 1, 1, 0) \cdot \big( \frac{\alpha_D}{3} + \alpha_Y \big)
            \frac{1}{(1 + \alpha_D + \alpha_Y)^3},
        \end{split}
    \end{equation*}
    where $\gamma_{0101}^{\text{(low)}}, \gamma_{1010}^{\text{(up)}} \in (0, 1)$ (according to Assumption \ref{assumption:biased_weight_quantitative}), and $\alpha_D, \alpha_Y \in [0, \frac{1}{2})$ (according to Assumption \ref{assumption:Hupdate}).

    Now let us consider the data dynamics where $\alpha_Y = 0$ and simplify the form of $\Delta^{(\text{Counterfactual Fair})}_{\mathrm{STIR}} \big\rvert_{t}^{t + 1}$:
    \begin{equation*}
        \begin{split}
            &~~~~~
            \Delta^{(\text{Counterfactual Fair})}_{\mathrm{STIR}} \big\rvert_{t}^{t + 1} \\
            &=
            - \frac{\alpha_D}{3} \cdot \big(
                P_t(0, 0, 0, 1) + P_t(0, 0, 1, 0) \big)
            + \frac{\alpha_D}{3 (1 + \alpha_D)^3} \cdot \big(
                P_t(1, 1, 0, 1) + P_t(1, 1, 1, 0) \big).
        \end{split}
    \end{equation*}
    As we can see, as long as we have $\frac{P_t(1, 1, 0, 1) + P_t(1, 1, 1, 0)}{P_t(0, 0, 0, 1) + P_t(0, 0, 1, 0)} < \frac{27}{8}$ and at the same time the parameter satisfies $\alpha_D \in ( \big(\frac{P_t(1, 1, 0, 1) + P_t(1, 1, 1, 0)}{P_t(0, 0, 0, 1) + P_t(0, 0, 1, 0)}\big)^{\frac{1}{3}} - 1, \frac{1}{2})$, it is possible for the counterfactual fair predictor to achieve a negative value for $\STIR$ after a one-step intervention:
    \begin{equation*}
        \begin{cases}
            & \frac{P_t(1, 1, 0, 1) + P_t(1, 1, 1, 0)}{P_t(0, 0, 0, 1) + P_t(0, 0, 1, 0)} < \frac{27}{8} \\
            & \alpha_D \in ( \big(\frac{P_t(1, 1, 0, 1) + P_t(1, 1, 1, 0)}{P_t(0, 0, 0, 1) + P_t(0, 0, 1, 0)}\big)^{\frac{1}{3}} - 1, \frac{1}{2}) \\
            & \alpha_Y = 0
        \end{cases}
        \implies \Delta^{(\text{Counterfactual Fair})}_{\mathrm{STIR}} \big\rvert_{t}^{t + 1} < 0.
    \end{equation*}
\end{proof}

    \afterpage{
\renewcommand{\arraystretch}{1.5}
\begin{landscape}\begin{table}[htbp]
    \caption{
        When $\alpha_D > \alpha_Y$, compare cases of $H_{t + 1}$ with different values of $A_t$.
    }
    \label{table:H_next_Dmajor}
    \centering
    \begin{tabular}{l|cc|cc|ll}
        \toprule
        \multirow{2}{*}{Case} & \multicolumn{2}{c|}{$D_t$} & \multicolumn{2}{c|}{$\Yori_t$} & \multicolumn{2}{c}{$H_{t + 1}$} \\
         \cmidrule{2-7}
         & if $A_t = 0$
        & if $A_t = 1$
        & if $A_t = 0$
        & if $A_t = 1$
        & if $A_t = 0$
        & if $A_t = 1$ \\
        \midrule
        (\romannum{1}) & 0 & 0 & 0 & 0 & $f_t(0, e_t)(1 - \alpha_D - \alpha_Y)$ & $f_t(1, e_t)(1 - \alpha_D - \alpha_Y)$ \\
        (\romannum{2}) & 0 & 0 & 0 & 1 & $f_t(0, e_t)(1 - \alpha_D - \alpha_Y)$ & $f_t(1, e_t)(1 - \alpha_D + \alpha_Y)$ \\
        (\romannum{3}) & 0 & 0 & 1 & 0 & $f_t(0, e_t)(1 - \alpha_D + \alpha_Y)$ & $f_t(1, e_t)(1 - \alpha_D - \alpha_Y)$ \\
        (\romannum{4}) & 0 & 0 & 1 & 1 & $f_t(0, e_t)(1 - \alpha_D + \alpha_Y)$ & $f_t(1, e_t)(1 - \alpha_D + \alpha_Y)$ \\
        (\romannum{5}) & 0 & 1 & 0 & 0 & $f_t(0, e_t)(1 - \alpha_D - \alpha_Y)$ & $\min\{ f_t(1, e_t)(1 + \alpha_D - \alpha_Y), 1 \}$ \\
        (\romannum{6}) & 0 & 1 & 0 & 1 & $f_t(0, e_t)(1 - \alpha_D - \alpha_Y)$ & $\min\{ f_t(1, e_t)(1 + \alpha_D + \alpha_Y), 1 \}$ \\
        (\romannum{7}) & 0 & 1 & 1 & 0 & $f_t(0, e_t)(1 - \alpha_D + \alpha_Y)$ & $\min\{ f_t(1, e_t)(1 + \alpha_D - \alpha_Y), 1 \}$ \\
        (\romannum{8}) & 0 & 1 & 1 & 1 & $f_t(0, e_t)(1 - \alpha_D + \alpha_Y)$ & $\min\{ f_t(1, e_t)(1 + \alpha_D + \alpha_Y), 1 \}$ \\
        (\romannum{9}) & 1 & 0 & 0 & 0 & $\min\{ f_t(0, e_t)(1 + \alpha_D - \alpha_Y), 1 \}$ & $f_t(1, e_t)(1 - \alpha_D - \alpha_Y)$ \\
        (\romannum{10}) & 1 & 0 & 0 & 1 & $\min\{ f_t(0, e_t)(1 + \alpha_D - \alpha_Y), 1 \}$ & $f_t(1, e_t)(1 - \alpha_D + \alpha_Y)$ \\
        (\romannum{11}) & 1 & 0 & 1 & 0 & $\min\{ f_t(0, e_t)(1 + \alpha_D + \alpha_Y), 1 \}$ & $f_t(1, e_t)(1 - \alpha_D - \alpha_Y)$ \\
        (\romannum{12}) & 1 & 0 & 1 & 1 & $\min\{ f_t(0, e_t)(1 + \alpha_D + \alpha_Y), 1 \}$ & $f_t(1, e_t)(1 - \alpha_D + \alpha_Y)$ \\
        (\romannum{13}) & 1 & 1 & 0 & 0 & $\min\{ f_t(0, e_t)(1 + \alpha_D - \alpha_Y), 1 \}$ & $\min\{ f_t(1, e_t)(1 + \alpha_D - \alpha_Y), 1 \}$ \\
        (\romannum{14}) & 1 & 1 & 0 & 1 & $\min\{ f_t(0, e_t)(1 + \alpha_D - \alpha_Y), 1 \}$ & $\min\{ f_t(1, e_t)(1 + \alpha_D + \alpha_Y), 1 \}$ \\
        (\romannum{15}) & 1 & 1 & 1 & 0 & $\min\{ f_t(0, e_t)(1 + \alpha_D + \alpha_Y), 1 \}$ & $\min\{ f_t(1, e_t)(1 + \alpha_D - \alpha_Y), 1 \}$ \\
        (\romannum{16}) & 1 & 1 & 1 & 1 & $\min\{ f_t(0, e_t)(1 + \alpha_D + \alpha_Y), 1 \}$ & $\min\{ f_t(1, e_t)(1 + \alpha_D + \alpha_Y), 1 \}$ \\
        \bottomrule
    \end{tabular}
\end{table}\end{landscape}
\clearpage
\renewcommand{\arraystretch}{1}
}

\afterpage{
\renewcommand{\arraystretch}{1.5}
\begin{landscape}\begin{table}[htbp]
    \caption{
        When $\alpha_D < \alpha_Y$, compare cases of $H_{t + 1}$ with different values of $A_t$.
    }
    \label{table:H_next_Ymajor}
    \centering
    \begin{tabular}{l|cc|cc|ll}
        \toprule
        \multirow{2}{*}{Case} & \multicolumn{2}{c|}{$D_t$} & \multicolumn{2}{c|}{$\Yori_t$} & \multicolumn{2}{c}{$H_{t + 1}$} \\
         \cmidrule{2-7}
         & if $A_t = 0$
        & if $A_t = 1$
        & if $A_t = 0$
        & if $A_t = 1$
        & if $A_t = 0$
        & if $A_t = 1$ \\
        \midrule
        (\romannum{1}) & 0 & 0 & 0 & 0 & $f_t(0, e_t)(1 - \alpha_D - \alpha_Y)$ & $f_t(1, e_t)(1 - \alpha_D - \alpha_Y)$ \\
        (\romannum{2}) & 0 & 0 & 0 & 1 & $f_t(0, e_t)(1 - \alpha_D - \alpha_Y)$ & $\min\{ f_t(1, e_t)(1 - \alpha_D + \alpha_Y), 1 \}$ \\
        (\romannum{3}) & 0 & 0 & 1 & 0 & $\min\{ f_t(0, e_t)(1 - \alpha_D + \alpha_Y), 1 \}$ & $f_t(1, e_t)(1 - \alpha_D - \alpha_Y)$ \\
        (\romannum{4}) & 0 & 0 & 1 & 1 & $\min\{ f_t(0, e_t)(1 - \alpha_D + \alpha_Y), 1 \}$ & $\min\{ f_t(1, e_t)(1 - \alpha_D + \alpha_Y), 1 \}$ \\
        (\romannum{5}) & 0 & 1 & 0 & 0 & $f_t(0, e_t)(1 - \alpha_D - \alpha_Y)$ & $f_t(1, e_t)(1 + \alpha_D - \alpha_Y)$ \\
        (\romannum{6}) & 0 & 1 & 0 & 1 & $f_t(0, e_t)(1 - \alpha_D - \alpha_Y)$ & $\min\{ f_t(1, e_t)(1 + \alpha_D + \alpha_Y), 1 \}$ \\
        (\romannum{7}) & 0 & 1 & 1 & 0 & $\min\{ f_t(0, e_t)(1 - \alpha_D + \alpha_Y), 1 \}$ & $f_t(1, e_t)(1 + \alpha_D - \alpha_Y)$ \\
        (\romannum{8}) & 0 & 1 & 1 & 1 & $\min\{ f_t(0, e_t)(1 - \alpha_D + \alpha_Y), 1 \}$ & $\min\{ f_t(1, e_t)(1 + \alpha_D + \alpha_Y), 1 \}$ \\
        (\romannum{9}) & 1 & 0 & 0 & 0 & $f_t(0, e_t)(1 + \alpha_D - \alpha_Y)$ & $f_t(1, e_t)(1 - \alpha_D - \alpha_Y)$ \\
        (\romannum{10}) & 1 & 0 & 0 & 1 & $f_t(0, e_t)(1 + \alpha_D - \alpha_Y)$ & $\min\{ f_t(1, e_t)(1 - \alpha_D + \alpha_Y), 1 \}$ \\
        (\romannum{11}) & 1 & 0 & 1 & 0 & $\min\{ f_t(0, e_t)(1 + \alpha_D + \alpha_Y), 1 \}$ & $f_t(1, e_t)(1 - \alpha_D - \alpha_Y)$ \\
        (\romannum{12}) & 1 & 0 & 1 & 1 & $\min\{ f_t(0, e_t)(1 + \alpha_D + \alpha_Y), 1 \}$ & $\min\{ f_t(1, e_t)(1 - \alpha_D + \alpha_Y), 1 \}$ \\
        (\romannum{13}) & 1 & 1 & 0 & 0 & $f_t(0, e_t)(1 + \alpha_D - \alpha_Y)$ & $f_t(1, e_t)(1 + \alpha_D - \alpha_Y)$ \\
        (\romannum{14}) & 1 & 1 & 0 & 1 & $f_t(0, e_t)(1 + \alpha_D - \alpha_Y)$ & $\min\{ f_t(1, e_t)(1 + \alpha_D + \alpha_Y), 1 \}$ \\
        (\romannum{15}) & 1 & 1 & 1 & 0 & $\min\{ f_t(0, e_t)(1 + \alpha_D + \alpha_Y), 1 \}$ & $f_t(1, e_t)(1 + \alpha_D - \alpha_Y)$ \\
        (\romannum{16}) & 1 & 1 & 1 & 1 & $\min\{ f_t(0, e_t)(1 + \alpha_D + \alpha_Y), 1 \}$ & $\min\{ f_t(1, e_t)(1 + \alpha_D + \alpha_Y), 1 \}$ \\
        \bottomrule
    \end{tabular}
\end{table}\end{landscape}
\clearpage
\renewcommand{\arraystretch}{1}
}

\afterpage{
\renewcommand{\arraystretch}{1.5}
\begin{landscape}\begin{table}[htbp]
    \caption{
        When $\alpha_D = \alpha_Y = \alpha$, compare cases of $H_{t + 1}$ with different values of $A_t$.
    }
    \label{table:H_next_nomajor}
    \centering
    \begin{tabular}{l|cc|cc|ll}
        \toprule
        \multirow{2}{*}{Case} & \multicolumn{2}{c|}{$D_t$} & \multicolumn{2}{c|}{$\Yori_t$} & \multicolumn{2}{c}{$H_{t + 1}$} \\
         \cmidrule{2-7}
         & if $A_t = 0$
        & if $A_t = 1$
        & if $A_t = 0$
        & if $A_t = 1$
        & if $A_t = 0$
        & if $A_t = 1$ \\
        \midrule
        (\romannum{1}) & 0 & 0 & 0 & 0 & $f_t(0, e_t)(1 - 2\alpha)$ & $f_t(1, e_t)(1 - 2\alpha)$ \\
        (\romannum{2}) & 0 & 0 & 0 & 1 & $f_t(0, e_t)(1 - 2\alpha)$ & $f_t(1, e_t)$ \\
        (\romannum{3}) & 0 & 0 & 1 & 0 & $f_t(0, e_t)$ & $f_t(1, e_t)(1 - 2\alpha)$ \\
        (\romannum{4}) & 0 & 0 & 1 & 1 & $f_t(0, e_t)$ & $f_t(1, e_t)$ \\
        (\romannum{5}) & 0 & 1 & 0 & 0 & $f_t(0, e_t)(1 - 2\alpha)$ & $f_t(1, e_t)$ \\
        (\romannum{6}) & 0 & 1 & 0 & 1 & $f_t(0, e_t)(1 - 2\alpha)$ & $\min\{ f_t(1, e_t)(1 + 2\alpha), 1 \}$ \\
        (\romannum{7}) & 0 & 1 & 1 & 0 & $f_t(0, e_t)$ & $f_t(1, e_t)$ \\
        (\romannum{8}) & 0 & 1 & 1 & 1 & $f_t(0, e_t)$ & $\min\{ f_t(1, e_t)(1 + 2\alpha), 1 \}$ \\
        (\romannum{9}) & 1 & 0 & 0 & 0 & $f_t(0, e_t)$ & $f_t(1, e_t)(1 - 2\alpha)$ \\
        (\romannum{10}) & 1 & 0 & 0 & 1 & $f_t(0, e_t)$ & $f_t(1, e_t)$ \\
        (\romannum{11}) & 1 & 0 & 1 & 0 & $\min\{ f_t(0, e_t)(1 + 2\alpha), 1 \}$ & $f_t(1, e_t)(1 - 2\alpha)$ \\
        (\romannum{12}) & 1 & 0 & 1 & 1 & $\min\{ f_t(0, e_t)(1 + 2\alpha), 1 \}$ & $f_t(1, e_t)$ \\
        (\romannum{13}) & 1 & 1 & 0 & 0 & $f_t(0, e_t)$ & $f_t(1, e_t)$ \\
        (\romannum{14}) & 1 & 1 & 0 & 1 & $f_t(0, e_t)$ & $\min\{ f_t(1, e_t)(1 + 2\alpha), 1 \}$ \\
        (\romannum{15}) & 1 & 1 & 1 & 0 & $\min\{ f_t(0, e_t)(1 + 2\alpha), 1 \}$ & $f_t(1, e_t)$ \\
        (\romannum{16}) & 1 & 1 & 1 & 1 & $\min\{ f_t(0, e_t)(1 + 2\alpha), 1 \}$ & $\min\{ f_t(1, e_t)(1 + 2\alpha), 1 \}$ \\
        \bottomrule
    \end{tabular}
\end{table}\end{landscape}
\clearpage
\renewcommand{\arraystretch}{1}
}

    \afterpage{
\renewcommand{\arraystretch}{1.5}
\begin{landscape}\begin{longtable}{l|cc|cc|ll}
    \caption{
        When $\alpha_D > \alpha_Y$, list possible instantiations of $\varphi_{t + 1}(e_{t + 1})$.
    }
    \label{table:delta_next_Dmajor}
    \\
    \hline
    \multirow{2}{*}{Case}
    & \multicolumn{2}{c|}{$D_t$}
    & \multicolumn{2}{c|}{$\Yori_t$}
    &
    & \multirow{2}{*}{$\varphi_{t + 1}(e_{t + 1}) = f_{t + 1}(0, e_{t + 1}) - f_{t + 1}(1, e_{t + 1})$} \\ \cmidrule{2-5}
     & if $A_t = 0$
    & if $A_t = 1$
    & if $A_t = 0$
    & if $A_t = 1$
    &  \\
    \hline
    \endfirsthead

    \multicolumn{7}{c}%
    {{\tablename\ \thetable{} (continued from the previous page)}} \\
    \hline
    \multirow{2}{*}{Case}
    & \multicolumn{2}{c|}{$D_t$}
    & \multicolumn{2}{c|}{$\Yori_t$}
    &
    & \multirow{2}{*}{$\varphi_{t + 1}(e_{t + 1}) = f_{t + 1}(0, e_{t + 1}) - f_{t + 1}(1, e_{t + 1})$} \\ \cmidrule{2-5}
     & if $A_t = 0$
    & if $A_t = 1$
    & if $A_t = 0$
    & if $A_t = 1$
    &  \\
    \hline
    \endhead

    \hline
    \endfoot

    \hline
    \endlastfoot

    (\romannum{1}) & 0 & 0 & 0 & 0 & & $\varphi_t(e_t)(1 - \alpha_D - \alpha_Y)$ \\
    (\romannum{2}) & 0 & 0 & 0 & 1 & & $\varphi_t(e_t)(1 - \alpha_D) - \alpha_Y \eta_t(e_t)$ \\
    (\romannum{3}) & 0 & 0 & 1 & 0 & & $\varphi_t(e_t)(1 - \alpha_D) + \alpha_Y \eta_t(e_t)$ \\
    (\romannum{4}) & 0 & 0 & 1 & 1 & & $\varphi_t(e_t)(1 - \alpha_D + \alpha_Y)$ \\
    (\romannum{5})
    & 0
    & 1
    & 0
    & 0
    & (\romannum{5}.1)
    & $\varphi_t(e_t)(1 - \alpha_Y) - \alpha_D \eta_t(e_t)$, if $f_t(1, e_t) \in (0, \frac{1}{1 + \alpha_D - \alpha_Y})$ \\
     & - & - & - & - & (\romannum{5}.2)
    & $f_t(0, e_t)(1 - \alpha_D - \alpha_Y) - 1$, otherwise \\
    (\romannum{6})
    & 0
    & 1
    & 0
    & 1
    & (\romannum{6}.1)
    & $\varphi_t(e_t) - (\alpha_D + \alpha_Y) \eta_t(e_t)$, if $f_t(1, e_t) \in (0, \frac{1}{1 + \alpha_D + \alpha_Y})$ \\
     & - & - & - & - & (\romannum{6}.2)
    & $f_t(0, e_t)(1 - \alpha_D - \alpha_Y) - 1$, otherwise \\
    (\romannum{7})
    & 0
    & 1
    & 1
    & 0
    & (\romannum{7}.1)
    & $\varphi_t(e_t) - (\alpha_D - \alpha_Y) \eta_t(e_t)$, if $f_t(1, e_t) \in (0, \frac{1}{1 + \alpha_D - \alpha_Y})$ \\
     & - & - & - & - & (\romannum{7}.2)
    & $f_t(0, e_t)(1 - \alpha_D + \alpha_Y) - 1$, otherwise \\
    (\romannum{8})
    & 0
    & 1
    & 1
    & 1
    & (\romannum{8}.1)
    & $\varphi_t(e_t)(1 + \alpha_Y) - \alpha_D \eta_t(e_t)$, if $f_t(1, e_t) \in (0, \frac{1}{1 + \alpha_D + \alpha_Y})$ \\
     & - & - & - & - & (\romannum{8}.2)
    & $f_t(0, e_t)(1 - \alpha_D + \alpha_Y) - 1$, otherwise \\
    (\romannum{9})
    & 1
    & 0
    & 0
    & 0
    & (\romannum{9}.1)
    & $\varphi_t(e_t)(1 - \alpha_Y) + \alpha_D \eta_t(e_t)$, if $f_t(0, e_t) \in (0, \frac{1}{1 + \alpha_D - \alpha_Y})$ \\
     & - & - & - & - & (\romannum{9}.2)
    & $1 - f_t(1, e_t)(1 - \alpha_D - \alpha_Y)$, otherwise \\
    (\romannum{10})
    & 1
    & 0
    & 0
    & 1
    & (\romannum{10}.1)
    & $\varphi_t(e_t) + (\alpha_D - \alpha_Y) \eta_t(e_t)$, if $f_t(0, e_t) \in (0, \frac{1}{1 + \alpha_D - \alpha_Y})$ \\
     & - & - & - & - & (\romannum{10}.2)
    & $1 - f_t(1, e_t)(1 - \alpha_D + \alpha_Y)$, otherwise \\
    (\romannum{11})
    & 1
    & 0
    & 1
    & 0
    & (\romannum{11}.1)
    & $\varphi_t(e_t) + (\alpha_D + \alpha_Y) \eta_t(e_t)$, if $f_t(0, e_t) \in (0, \frac{1}{1 + \alpha_D + \alpha_Y})$ \\
     & - & - & - & - & (\romannum{11}.2)
    & $1 - f_t(1, e_t)(1 - \alpha_D - \alpha_Y)$, otherwise \\
    (\romannum{12})
    & 1
    & 0
    & 1
    & 1
    & (\romannum{12}.1)
    & $\varphi_t(e_t)(1 + \alpha_Y) + \alpha_D \eta_t(e_t)$, if $f_t(0, e_t) \in (0, \frac{1}{1 + \alpha_D + \alpha_Y})$ \\
     & & & & & (\romannum{12}.2)
    & $1 - f_t(1, e_t)(1 - \alpha_D + \alpha_Y)$, otherwise \\
    (\romannum{13})
    & 1
    & 1
    & 0
    & 0
    & (\romannum{13}.1)
    & $\varphi_t(e_t)(1 + \alpha_D - \alpha_Y)$, if $f_t(0, e_t) \in (0, \frac{1}{1 + \alpha_D - \alpha_Y})$ and $f_t(1, e_t) \in (0, \frac{1}{1 + \alpha_D - \alpha_Y})$ \\
     & - & - & - & - & (\romannum{13}.2)
    & $f_t(0, e_t)(1 + \alpha_D - \alpha_Y) - 1$, if $f_t(0, e_t) \in (0, \frac{1}{1 + \alpha_D - \alpha_Y})$ and $f_t(1, e_t) \in [\frac{1}{1 + \alpha_D - \alpha_Y}, 1]$  \\
     & - & - & - & - & (\romannum{13}.3)
    & $1 - f_t(1, e_t)(1 + \alpha_D - \alpha_Y)$, if $f_t(0, e_t) \in [\frac{1}{1 + \alpha_D - \alpha_Y}, 1]$ and $f_t(1, e_t) \in (0, \frac{1}{1 + \alpha_D - \alpha_Y})$ \\
     & - & - & - & - & (\romannum{13}.4)
    & $0$, otherwise \\
    (\romannum{14})
    & 1
    & 1
    & 0
    & 1
    & (\romannum{14}.1)
    & $\varphi_t(e_t)(1 + \alpha_D) - \alpha_Y \eta_t(e_t)$, if $f_t(0, e_t) \in (0, \frac{1}{1 + \alpha_D - \alpha_Y})$ and $f_t(1, e_t) \in (0, \frac{1}{1 + \alpha_D + \alpha_Y})$ \\
     & - & - & - & - & (\romannum{14}.2)
    & $f_t(0, e_t)(1 + \alpha_D - \alpha_Y) - 1$, if $f_t(0, e_t) \in (0, \frac{1}{1 + \alpha_D - \alpha_Y})$ and $f_t(1, e_t) \in [\frac{1}{1 + \alpha_D + \alpha_Y}, 1]$  \\
     & - & - & - & - & (\romannum{14}.3)
    & $1 - f_t(1, e_t)(1 + \alpha_D + \alpha_Y)$, if $f_t(0, e_t) \in [\frac{1}{1 + \alpha_D - \alpha_Y}, 1]$ and $f_t(1, e_t) \in (0, \frac{1}{1 + \alpha_D + \alpha_Y})$ \\
     & - & - & - & - & (\romannum{14}.4)
    & $0$, otherwise \\
    (\romannum{15})
    & 1
    & 1
    & 1
    & 0
    & (\romannum{15}.1)
    & $\varphi_t(e_t)(1 + \alpha_D) + \alpha_Y \eta_t(e_t)$, if $f_t(0, e_t) \in (0, \frac{1}{1 + \alpha_D + \alpha_Y})$ and $f_t(1, e_t) \in (0, \frac{1}{1 + \alpha_D - \alpha_Y})$ \\
     & - & - & - & - & (\romannum{15}.2)
    & $f_t(0, e_t)(1 + \alpha_D + \alpha_Y) - 1$, if $f_t(0, e_t) \in (0, \frac{1}{1 + \alpha_D + \alpha_Y})$ and $f_t(1, e_t) \in [\frac{1}{1 + \alpha_D - \alpha_Y}, 1]$  \\
     & - & - & - & - & (\romannum{15}.3)
    & $1 - f_t(1, e_t)(1 + \alpha_D - \alpha_Y)$, if $f_t(0, e_t) \in [\frac{1}{1 + \alpha_D + \alpha_Y}, 1]$ and $f_t(1, e_t) \in (0, \frac{1}{1 + \alpha_D - \alpha_Y})$ \\
     & - & - & - & - & (\romannum{15}.4)
    & $0$, otherwise \\
    (\romannum{16})
    & 1
    & 1
    & 1
    & 1
    & (\romannum{16}.1)
    & $\varphi_t(e_t)(1 + \alpha_D + \alpha_Y)$, if $f_t(0, e_t) \in (0, \frac{1}{1 + \alpha_D + \alpha_Y})$ and $f_t(1, e_t) \in (0, \frac{1}{1 + \alpha_D + \alpha_Y})$ \\
     & - & - & - & - & (\romannum{16}.2)
    & $f_t(0, e_t)(1 + \alpha_D + \alpha_Y) - 1$, if $f_t(0, e_t) \in (0, \frac{1}{1 + \alpha_D + \alpha_Y})$ and $f_t(1, e_t) \in [\frac{1}{1 + \alpha_D + \alpha_Y}, 1]$  \\
     & - & - & - & - & (\romannum{16}.3)
    & $1 - f_t(1, e_t)(1 + \alpha_D + \alpha_Y)$, if $f_t(0, e_t) \in [\frac{1}{1 + \alpha_D + \alpha_Y}, 1]$ and $f_t(1, e_t) \in (0, \frac{1}{1 + \alpha_D + \alpha_Y})$ \\
     & - & - & - & - & (\romannum{16}.4)
    & $0$, otherwise \\
    \hline
\end{longtable}\end{landscape}
\clearpage
\renewcommand{\arraystretch}{1}
}

\afterpage{
\renewcommand{\arraystretch}{1.5}
\begin{landscape}\begin{longtable}{l|cc|cc|ll}
    \caption{
        When $\alpha_D < \alpha_Y$, list possible instantiations of $\varphi_{t + 1}(e_{t + 1})$.
    }
    \label{table:delta_next_Ymajor}
    \\
    \hline
    \multirow{2}{*}{Case}
    & \multicolumn{2}{c|}{$D_t$}
    & \multicolumn{2}{c|}{$\Yori_t$}
    &
    & \multirow{2}{*}{$\varphi_{t + 1}(e_{t + 1}) = f_{t + 1}(0, e_{t + 1}) - f_{t + 1}(1, e_{t + 1})$} \\ \cmidrule{2-5}
     & if $A_t = 0$
    & if $A_t = 1$
    & if $A_t = 0$
    & if $A_t = 1$
    &  \\
    \hline
    \endfirsthead

    \multicolumn{7}{c}%
    {{\tablename\ \thetable{} (continued from the previous page)}} \\
    \hline
    \multirow{2}{*}{Case}
    & \multicolumn{2}{c|}{$D_t$}
    & \multicolumn{2}{c|}{$\Yori_t$}
    &
    & \multirow{2}{*}{$\varphi_{t + 1}(e_{t  + 1}) = f_{t + 1}(0, e_{t + 1}) - f_{t + 1}(1, e_{t + 1})$} \\ \cmidrule{2-5}
     & if $A_t = 0$
    & if $A_t = 1$
    & if $A_t = 0$
    & if $A_t = 1$
    &  \\
    \hline
    \endhead

    \hline
    \endfoot

    \hline
    \endlastfoot

    (\romannum{1}) & 0 & 0 & 0 & 0 & & $\varphi_t(e_t)(1 - \alpha_D - \alpha_Y)$ \\
    (\romannum{2})
    & 0
    & 0
    & 0
    & 1
    & (\romannum{2}.1)
    & $\varphi_t(e_t)(1 - \alpha_D) - \alpha_Y \eta_t(e_t)$, if $f_t(1, e_t) \in (0, \frac{1}{1 - \alpha_D + \alpha_Y})$ \\
     & - & - & - & - & (\romannum{2}.2)
    & $f_t(0, e_t)(1 - \alpha_D - \alpha_Y) - 1$, otherwise \\
    (\romannum{3})
    & 0
    & 0
    & 1
    & 0
    & (\romannum{3}.1)
    & $\varphi_t(e_t)(1 - \alpha_D) + \alpha_Y \eta_t(e_t)$, if $f_t(0, e_t) \in (0, \frac{1}{1 - \alpha_D + \alpha_Y})$ \\
     & - & - & - & - & (\romannum{3}.2)
    & $1 - f_t(1, e_t)(1 - \alpha_D - \alpha_Y)$, otherwise \\
    (\romannum{4})
    & 0
    & 0
    & 1
    & 1
    & (\romannum{4}.1)
    & $\varphi_t(e_t)(1 - \alpha_D + \alpha_Y)$, if $f_t(0, e_t) \in (0, \frac{1}{1 - \alpha_D + \alpha_Y})$ and $f_t(1, e_t) \in (0, \frac{1}{1 - \alpha_D + \alpha_Y})$ \\
     & - & - & - & - & (\romannum{4}.2)
    & $f_t(0, e_t)(1 - \alpha_D + \alpha_Y) - 1$, if $f_t(0, e_t) \in (0, \frac{1}{1 - \alpha_D + \alpha_Y})$ and $f_t(1, e_t) \in [\frac{1}{1 - \alpha_D + \alpha_Y}, 1]$  \\
     & - & - & - & - & (\romannum{4}.3)
    & $1 - f_t(1, e_t)(1 - \alpha_D + \alpha_Y)$, if $f_t(0, e_t) \in [\frac{1}{1 - \alpha_D + \alpha_Y}, 1]$ and $f_t(1, e_t) \in (0, \frac{1}{1 - \alpha_D + \alpha_Y})$ \\
     & - & - & - & - & (\romannum{4}.4)
    & $0$, otherwise \\
    (\romannum{5}) & 0 & 1 & 0 & 0 & & $\varphi_t(e_t)(1 - \alpha_Y) - \alpha_D \eta_t(e_t)$ \\
    (\romannum{6})
    & 0
    & 1
    & 0
    & 1
    & (\romannum{6}.1)
    & $\varphi_t(e_t) - (\alpha_D + \alpha_Y) \eta_t(e_t)$, if $f_t(1, e_t) \in (0, \frac{1}{1 + \alpha_D + \alpha_Y})$ \\
     & - & - & - & - & (\romannum{6}.2)
    & $f_t(0, e_t)(1 - \alpha_D - \alpha_Y) - 1$, otherwise \\
    (\romannum{7})
    & 0
    & 1
    & 1
    & 0
    & (\romannum{7}.1)
    & $\varphi_t(e_t) - (\alpha_D - \alpha_Y) \eta_t(e_t)$, if $f_t(0, e_t) \in (0, \frac{1}{1 - \alpha_D + \alpha_Y})$ \\
     & - & - & - & - & (\romannum{7}.2)
    & $1 - f_t(1, e_t)(1 + \alpha_D - \alpha_Y)$, otherwise \\
    (\romannum{8})
    & 0
    & 1
    & 1
    & 1
    & (\romannum{8}.1)
    & $\varphi_t(e_t)(1 + \alpha_Y) - \alpha_D \eta_t(e_t)$, if $f_t(0, e_t) \in (0, \frac{1}{1 - \alpha_D + \alpha_Y})$ and $f_t(1, e_t) \in (0, \frac{1}{1 + \alpha_D + \alpha_Y})$ \\
     & - & - & - & - & (\romannum{8}.2)
    & $f_t(0, e_t)(1 - \alpha_D + \alpha_Y) - 1$, if $f_t(0, e_t) \in (0, \frac{1}{1 - \alpha_D + \alpha_Y})$ and $f_t(1, e_t) \in [\frac{1}{1 + \alpha_D + \alpha_Y}, 1]$  \\
     & - & - & - & - & (\romannum{8}.3)
    & $1 - f_t(1, e_t)(1 + \alpha_D + \alpha_Y)$, if $f_t(0, e_t) \in [\frac{1}{1 - \alpha_D + \alpha_Y}, 1]$ and $f_t(1, e_t) \in (0, \frac{1}{1 + \alpha_D + \alpha_Y})$ \\
     & - & - & - & - & (\romannum{8}.4)
    & $0$, otherwise \\
    (\romannum{9}) & 1 & 0 & 0 & 0 & & $\varphi_t(e_t)(1 - \alpha_Y) + \alpha_D \eta_t(e_t)$ \\
    (\romannum{10})
    & 1
    & 0
    & 0
    & 1
    & (\romannum{10}.1)
    & $\varphi_t(e_t) + (\alpha_D - \alpha_Y) \eta_t(e_t)$, if $f_t(1, e_t) \in (0, \frac{1}{1 - \alpha_D + \alpha_Y})$ \\
     & - & - & - & - & (\romannum{10}.2)
    & $f_t(0, e_t)(1 + \alpha_D - \alpha_Y) - 1$, otherwise \\
    (\romannum{11})
    & 1
    & 0
    & 1
    & 0
    & (\romannum{11}.1)
    & $\varphi_t(e_t) + (\alpha_D + \alpha_Y) \eta_t(e_t)$, if $f_t(0, e_t) \in (0, \frac{1}{1 + \alpha_D + \alpha_Y})$ \\
     & - & - & - & - & (\romannum{11}.2)
    & $1 - f_t(1, e_t)(1 - \alpha_D - \alpha_Y)$, otherwise \\
    (\romannum{12})
    & 1
    & 0
    & 1
    & 1
    & (\romannum{12}.1)
    & $\varphi_t(e_t)(1 + \alpha_Y) + \alpha_D \eta_t(e_t)$, if $f_t(0, e_t) \in (0, \frac{1}{1 + \alpha_D + \alpha_Y})$ and $f_t(1, e_t) \in (0, \frac{1}{1 - \alpha_D + \alpha_Y})$ \\
     & - & - & - & - & (\romannum{12}.2)
    & $f_t(0, e_t)(1 + \alpha_D + \alpha_Y) - 1$, if $f_t(0, e_t) \in (0, \frac{1}{1 + \alpha_D + \alpha_Y})$ and $f_t(1, e_t) \in [\frac{1}{1 - \alpha_D + \alpha_Y}, 1]$  \\
     & - & - & - & - & (\romannum{12}.3)
    & $1 - f_t(1, e_t)(1 - \alpha_D + \alpha_Y)$, if $f_t(0, e_t) \in [\frac{1}{1 + \alpha_D + \alpha_Y}, 1]$ and $f_t(1, e_t) \in (0, \frac{1}{1 - \alpha_D + \alpha_Y})$ \\
     & - & - & - & - & (\romannum{12}.4)
    & $0$, otherwise \\
    (\romannum{13}) & 1 & 1 & 0 & 0 & & $\varphi_t(e_t)(1 + \alpha_D - \alpha_Y)$ \\
    (\romannum{14})
    & 1
    & 1
    & 0
    & 1
    & (\romannum{14}.1)
    & $\varphi_t(e_t)(1 + \alpha_D) - \alpha_Y \eta_t(e_t)$, if $f_t(1, e_t) \in (0, \frac{1}{1 + \alpha_D + \alpha_Y})$ \\
     & - & - & - & - & (\romannum{14}.2)
    & $f_t(0, e_t)(1 + \alpha_D - \alpha_Y) - 1$, otherwise \\
    (\romannum{15})
    & 1
    & 1
    & 1
    & 0
    & (\romannum{15}.1)
    & $\varphi_t(e_t)(1 + \alpha_D) + \alpha_Y \eta_t(e_t)$, if $f_t(0, e_t) \in (0, \frac{1}{1 + \alpha_D + \alpha_Y})$ \\
     & - & - & - & - & (\romannum{15}.2)
    & $1 - f_t(1, e_t)(1 + \alpha_D - \alpha_Y)$, otherwise \\
    (\romannum{16})
    & 1
    & 1
    & 1
    & 1
    & (\romannum{16}.1)
    & $\varphi_t(e_t)(1 + \alpha_D + \alpha_Y)$, if $f_t(0, e_t) \in (0, \frac{1}{1 + \alpha_D + \alpha_Y})$ and $f_t(1, e_t) \in (0, \frac{1}{1 + \alpha_D + \alpha_Y})$ \\
     & - & - & - & - & (\romannum{16}.2)
    & $f_t(0, e_t)(1 + \alpha_D + \alpha_Y) - 1$, if $f_t(0, e_t) \in (0, \frac{1}{1 + \alpha_D + \alpha_Y})$ and $f_t(1, e_t) \in [\frac{1}{1 + \alpha_D + \alpha_Y}, 1]$  \\
     & - & - & - & - & (\romannum{16}.3)
    & $1 - f_t(1, e_t)(1 + \alpha_D + \alpha_Y)$, if $f_t(0, e_t) \in [\frac{1}{1 + \alpha_D + \alpha_Y}, 1]$ and $f_t(1, e_t) \in (0, \frac{1}{1 + \alpha_D + \alpha_Y})$ \\
     & - & - & - & - & (\romannum{16}.4)
    & $0$, otherwise \\
    \hline
\end{longtable}\end{landscape}
\clearpage
\renewcommand{\arraystretch}{1}
}

\afterpage{
\renewcommand{\arraystretch}{1.2}
\begin{landscape}\begin{longtable}{l|cc|cc|ll}
    \caption{
        When $\alpha_D = \alpha_Y = \alpha$, list possible instantiations of $\varphi_{t + 1}(e_{t + 1})$.
    }
    \label{table:delta_next_nomajor}
    \\
    \hline
    \multirow{2}{*}{Case}
    & \multicolumn{2}{c|}{$D_t$}
    & \multicolumn{2}{c|}{$\Yori_t$}
    &
    & \multirow{2}{*}{$\varphi_{t + 1}(e_{t  + 1}) = f_{t + 1}(0, e_{t + 1}) - f_{t + 1}(1, e_{t + 1})$} \\ \cmidrule{2-5}
     & if $A_t = 0$
    & if $A_t = 1$
    & if $A_t = 0$
    & if $A_t = 1$
    &  \\
    \hline
    \endfirsthead

    \multicolumn{7}{c}%
    {{\tablename\ \thetable{} (continued from the previous page)}} \\
    \hline
    \multirow{2}{*}{Case}
    & \multicolumn{2}{c|}{$D_t$}
    & \multicolumn{2}{c|}{$\Yori_t$}
    &
    & \multirow{2}{*}{$\varphi_{t + 1}(e_{t  + 1}) = f_{t + 1}(0, e_{t + 1}) - f_{t + 1}(1, e_{t + 1})$} \\ \cmidrule{2-5}
     & if $A_t = 0$
    & if $A_t = 1$
    & if $A_t = 0$
    & if $A_t = 1$
    &  \\
    \hline
    \endhead

    \hline
    \endfoot

    \hline
    \endlastfoot

    (\romannum{1}) & 0 & 0 & 0 & 0 & & $\varphi_t(e_t)(1 - 2\alpha)$ \\
    (\romannum{2}) & 0 & 0 & 0 & 1 & & $\varphi_t(e_t)(1 - \alpha) - \alpha \eta_t(e_t)$ \\
    (\romannum{3}) & 0 & 0 & 1 & 0 & & $\varphi_t(e_t)(1 - \alpha) + \alpha \eta_t(e_t)$ \\
    (\romannum{4}) & 0 & 0 & 1 & 1 & & $\varphi_t(e_t)$ \\
    (\romannum{5}) & 0 & 1 & 0 & 0 & & $\varphi_t(e_t)(1 - \alpha) - \alpha \eta_t(e_t)$ \\
    (\romannum{6})
    & 0
    & 1
    & 0
    & 1
    & (\romannum{6}.1)
    & $\varphi_t(e_t) - 2 \alpha \eta_t(e_t)$, if $f_t(1, e_t) \in (0, \frac{1}{1 + 2 \alpha})$ \\
     & - & - & - & - & (\romannum{6}.2)
    & $f_t(0, e_t)(1 - 2 \alpha) - 1$, otherwise \\
    (\romannum{7}) & 0 & 1 & 1 & 0 & & $\varphi_t(e_t)$ \\
    (\romannum{8})
    & 0
    & 1
    & 1
    & 1
    & (\romannum{8}.1)
    & $\varphi_t(e_t)(1 + \alpha) - \alpha \eta_t(e_t)$, if $f_t(1, e_t) \in (0, \frac{1}{1 + 2 \alpha})$ \\
     & - & - & - & - & (\romannum{8}.2)
    & $f_t(0, e_t) - 1$, otherwise \\
    (\romannum{9}) & 1 & 0 & 0 & 0 & & $\varphi_t(e_t)(1 - \alpha) + \alpha \eta_t(e_t)$ \\
    (\romannum{10}) & 1 & 0 & 0 & 1 & & $\varphi_t(e_t)$ \\
    (\romannum{11})
    & 1
    & 0
    & 1
    & 0
    & (\romannum{11}.1)
    & $\varphi_t(e_t) + 2 \alpha \eta_t(e_t)$, if $f_t(0, e_t) \in (0, \frac{1}{1 + 2 \alpha})$ \\
     & - & - & - & - & (\romannum{11}.2)
    & $1 - f_t(1, e_t)(1 - 2 \alpha)$, otherwise \\
    (\romannum{12})
    & 1
    & 0
    & 1
    & 1
    & (\romannum{12}.1)
    & $\varphi_t(e_t)(1 + \alpha) + \alpha \eta_t(e_t)$, if $f_t(0, e_t) \in (0, \frac{1}{1 + 2 \alpha})$ \\
     & - & - & - & - & (\romannum{12}.2)
    & $1 - f_t(1, e_t)$, otherwise \\
    (\romannum{13}) & 1 & 1 & 0 & 0 & & $\varphi_t(e_t)$ \\
    (\romannum{14})
    & 1
    & 1
    & 0
    & 1
    & (\romannum{14}.1)
    & $\varphi_t(e_t)(1 + \alpha) - \alpha \eta_t(e_t)$, if $f_t(1, e_t) \in (0, \frac{1}{1 + 2 \alpha})$ \\
     & - & - & - & - & (\romannum{14}.2)
    & $f_t(0, e_t) - 1$, otherwise \\
    (\romannum{15})
    & 1
    & 1
    & 1
    & 0
    & (\romannum{15}.1)
    & $\varphi_t(e_t)(1 + \alpha) + \alpha \eta_t(e_t)$, if $f_t(0, e_t) \in (0, \frac{1}{1 + 2 \alpha})$ \\
     & - & - & - & - & (\romannum{15}.2)
    & $1 - f_t(1, e_t)$, otherwise \\
    (\romannum{16})
    & 1
    & 1
    & 1
    & 1
    & (\romannum{16}.1)
    & $\varphi_t(e_t)(1 + 2\alpha)$, if $f_t(0, e_t) \in (0, \frac{1}{1 + 2\alpha})$ and $f_t(1, e_t) \in (0, \frac{1}{1 + 2\alpha})$ \\
     & - & - & - & - & (\romannum{16}.2)
    & $f_t(0, e_t)(1 + 2\alpha) - 1$, if $f_t(0, e_t) \in (0, \frac{1}{1 + 2\alpha})$ and $f_t(1, e_t) \in [\frac{1}{1 + 2\alpha}, 1]$  \\
     & - & - & - & - & (\romannum{16}.3)
    & $1 - f_t(1, e_t)(1 + 2\alpha)$, if $f_t(0, e_t) \in [\frac{1}{1 + 2\alpha}, 1]$ and $f_t(1, e_t) \in (0, \frac{1}{1 + 2\alpha})$ \\
     & - & - & - & - & (\romannum{16}.4)
    & $0$, otherwise \\
    \hline
\end{longtable}\end{landscape}
\clearpage
\renewcommand{\arraystretch}{1}
}

    \stopcontents[supplement]


\end{bibunit}

\end{document}